\documentclass[11pt]{article}
\usepackage[utf8]{inputenc}
\usepackage[margin=1in]{geometry}

\usepackage{microtype}
\usepackage{graphicx}
\usepackage{subfigure}
\usepackage{booktabs}
\usepackage{natbib}
\setcitestyle{open={(},close={)}}

\usepackage[colorlinks=true,citecolor=blue]{hyperref}

\usepackage{amsmath,amsthm,amsfonts,amssymb,mathdots,array,mathrsfs,bm,bbm,stmaryrd,graphicx,subfigure,xcolor}
\usepackage{algorithm,algorithmic}
\usepackage{breakcites}

\usepackage[T1]{fontenc}
\usepackage{enumerate}
\usepackage{inputenc}

\usepackage{graphicx} 
\usepackage{subfigure}

\usepackage{booktabs,balance}
\usepackage{rotating}
\usepackage{boldline}
\usepackage{makecell}
\usepackage{multirow}
\usepackage{balance}

\usepackage{tikz}

\newtheorem{theorem}{Theorem}[section]

\newtheorem{lemma}[theorem]{Lemma}
\newtheorem{proposition}[theorem]{Proposition}

\newtheorem{remark}{Remark}[section]

\usepackage{xcolor}         
\usepackage{amssymb, amsmath}
\usepackage{algorithm}
\usepackage{algorithmic}
\usepackage{graphicx}
\usepackage{subfigure}

\newcommand{\bma}{\begin{matrix*}[r]}
	\newcommand{\ema}{\end{matrix*}}

\begin{document}

\title{\bf Empirical Risk Minimization for Losses without Variance}
\author{\vspace{0.5in}\\\textbf{Guanhua Fang} \\
School of Management\\
Fudan University\\
670 Guoshun Road, Shanghai 200433,  China\\
  \texttt{fanggh@fudan.edu.cn}
\vspace{0.5in}\\
\textbf{Ping Li} \\
Cognitive Computing Lab\\
Baidu Research\\
10900 NE 8th St. Bellevue, WA 98004, USA\\
  \texttt{pingli98@gmail.com}\vspace{0.5in} \\
  \textbf{Gennady Samorodnitsky} \\
School of Operations Research and Information Engineering\\
  Cornell University\\
  616 Thurston Ave. Ithaca, NY 14853, USA\\
  \texttt{gs18@cornell.edu}
}

\date{}
\maketitle

\begin{abstract}
\noindent\footnote{The work of Guanghua Fang was conducted as a postdoctoral researcher at Baidu Research -- Bellevue, WA 98004, USA. The work of Gennady Samorodnitsky was conducted as a consulting researcher at Baidu Research -- Bellevue.}This paper considers an empirical risk minimization problem under heavy-tailed settings,
where data does not have finite variance, but only has $p$-th moment with
$p \in (1,2)$.
Instead of using estimation procedure based on truncated observed data,
we choose the optimizer by minimizing the risk value.
Those risk values can be robustly estimated via using the remarkable Catoni's method~\citep{Cat12}.
Thanks to the structure of Catoni-type influence functions, we are able to establish excess risk upper bounds via using generalized generic chaining methods.
Moreover, we take computational issues into consideration.
We especially theoretically investigate two types of optimization methods, robust gradient descent algorithm and empirical risk-based methods.
With an extensive numerical study, we find that the optimizer based on empirical risks via Catoni-style estimation indeed shows better performance than other baselines.
It indicates that estimation directly based on truncated data may lead to unsatisfactory results.
\end{abstract}
\newpage

\section{Introduction}

Modern data usually exhibit heavy-tail phenomena.
For example, in financial market~\citep{bradley2003financial, ahn2012new}, the returns are usually not normally distributed. In telecommunications~\citep{crovella1999estimating, glaz2001scan}, the data sources may sometimes experience a burst of extreme events.
In network analysis, the distributions of indegrees, outdegrees and sizes of connected components might be heavy-tailed~\citep{meusel2014graph}.
In past decades, theoretical analysis of heavy-tailed data has increasingly become a hot topic in the machine learning field, including multi-armed bandits~\citep{BCL13}, reinforcement learning~\citep{zhuang2021no}, mean-estimation problems~\citep{minsker2018sub, LM19}, etc.
Among those, empirical risk minimization theory for heavy-tailed data has not been fully explored yet, especially in the situation when variance does not exist. In this paper, we provide a corresponding theoretical framework standing on the remarkable estimator introduced by~\citet{Cat12}.

Empirical risk minimization (a small sample of the important works include \cite{vapnik1991principles, van2000empirical, bartlett2005local, zhang2017mixup}) is one of the basic and fundamental principles in statistical learning problems.
The general setting can be described as follows.
Let $X$ be a random variable taking values in a measurable space $\mathcal X$ and let $\mathcal F$ be a set of functions defined on $\mathcal X$.
For each fixed function $f \in \mathcal F$, we define the risk
$m_f = \mathbb E [f(X)]$ and let $m^{\ast} = \inf_{f \in \mathcal F} m_f$ be the optimal risk.
Given a set of samples of $n$ random variables $X_1, \ldots, X_n$ which are identically and independently distributed as $X$, one aims at finding a function which leads to the smallest risk.
That is, our goal is to find $\hat f := \arg\min_{f \in \mathcal F} \mathbb E [f(X)|X_1, \ldots, X_n]$ which is the best we can do.
To this end, one can introduce an empirical risk minimizer,
\begin{eqnarray}\label{eq:def:mini}
f_{ERM} = \arg\min_{f \in \mathcal F} \frac{1}{n} \sum_{i=1}^n f(X_i).
\end{eqnarray}
The performance of empirical risk minimization is measured by the risk of $f_{ERM}$, that is,
\[m_{ERM} = \mathbb E[f_{ERM}(X)|X_1, \ldots, X_n],\]
where expectation is taken with respect to a new sample $X$ and is conditioning on all observed data $X_1, \ldots, X_n$.

\vspace{2mm}

\noindent \textbf{Generic Example}.
To be more concrete, we can consider a general prediction problem.
The training data is $(Z_1, Y_1), \ldots, (Z_n, Y_n)$ which are independently and identically distributed. One wishes to predict the value of response $Y$ given a new observation $Z$.
A predictor is a function $g$ whose quality is evaluated with a pre-determined loss function $\ell$. Then risk of predictor $g$ is defined as
$\mathbb E \ell(g(Z), Y)$. Given a class $\mathcal G$ of functions $g$, empirical risk minimization procedure returns a function that minimizes
the risk $\frac{1}{n}\sum_{i=1}^n \ell(g(Z_i),Y_i)$ over $\mathcal G$. Adopting the notion in this paper, we can treat
$(Z_i, Y_i)$ as $X_i$, the function $f(\cdot)$ represents $\ell(g(\cdot),\cdot)$ and $m_f$ substitutes $\mathbb E [\ell(g(Z),Y)]$.

In this work, we specifically consider the heavy-tailed setting in the infinite variance scenario, that is,
$\mathbb E[(f(X))^2]$ may not exist for some $f \in \mathcal F$. Instead, we assume a weak moment condition,
\[\mathbb E[|f(X)|^p] \leq v \]
for any $f \in \mathcal F$, where $p \in (1,2)$ and $v$ is a fixed positive constant.
(Note that $p = 2$ refers to the finite variance case.)
Under the current setting, $f_{ERM}$ defined in \eqref{eq:def:mini} is not reliable and is sensitive to outliers~\citep{lerasle2011robust, diakonikolas2020outlier}. Moreover, it has been shown that the $f_{ERM}$ is no longer rate-optimal~\citep{Cat12} in the presence of weak moment condition.
Therefore, we need to look for the robust estimation method and develop the corresponding learning theory.

\noindent \textbf{Literature Review}.
In classical ERM literature, it is commonly assumed that loss function $f$'s are bounded~\citep{bartlett2006empirical, yi2020non} or noise terms (input values) are uniformly bounded~\citep{liu2014robustness, zhu2021taming}.
Later, an ERM theory was extended to unbounded cases within sub-Gaussian classes,~\citep{lecue2013learning, lecue2016performance}.
Moreover, a further progress is made to address ERM theory under heavy-tailed settings where the tails of noise decay much slower than sub-Gaussian rate~\citep{brownlees2015empirical, hsu2016loss, roy2021empirical}.

In the line of recent literature in robust empirical risk minimization theory, the main approaches can be divided into two categories.
The first category of approach is based on using truncated loss~\citep{zhang2018ell, xu2020learning, chen2021generalized, xu2022non}. That is, it introduces a (non-linear) truncation function $\phi$ and aims to find the following optimizer,
\[\hat f = \arg\min_{f \in \mathcal F} \frac{1}{n} \sum_{i=1}^n \phi(f(X_i)).\]
For example, $\phi(x)$ can be taken as
\[
\phi(x) =
\begin{cases}
	\frac{1}{M} x^2 &  |x| \leq M, \\
	|x| &  |x| > M
\end{cases}
\]
in~\citet{huber2011robust};
\[\phi(x) =
\begin{cases}
\log(1 + x + x^2/2) &  x \geq 0, \\
- \log(1 - x + x^2/2)  &  x < 0
\end{cases}
\]
in~\citet{Cat12};
\[\phi(x) =
\begin{cases}
\log(1 + \sum_{k=1}^m |x|^k/k!) &  x \geq 0, \\
- \log(1 + \sum_{k=1}^m |x|^k/k!)  &  x < 0
\end{cases}
\]
in~\citet{xu2020learning}.
This type of approach is intuitive and is also computationally friendly, since it only requires an extra truncation procedure in addition to classical empirical risk minimization. However, the drawback is that the final estimator $\hat f$ is not necessarily the minimizer in terms of the risk values.

The second category of technique is based on robust estimation of loss~\citep{brownlees2015empirical}. That is,
\[\hat f = \arg\min_{f \in \mathcal F} \hat \mu_f, ~~ \text{where}~ \hat \mu_f ~\text{is a robust estimation of loss}~m_f. \]
The advantage of such method is that it indeed finds the best function which minimizes the excess risk based on training samples.
However, computation of $\hat f$ could be extremely prohibitive when functional space $\mathcal F$ is large.
In this work, we develop new theories under weak moment condition $1 < p < 2$ via adopting the second type of approach. Additionally, we propose several feasible computational schemes to overcome the optimization obstacles in statistical learning or deep learning (DL) problems.
A short summary of our \textbf{two types of contributions} are given in Table~\ref{tab:summary} and Table~\ref{tab:alg}.

\begin{table}[h!]
	\centering
	\small
	\begin{tabular}{ c|c|c}
		\hline
		\hline
		\multicolumn{3}{c}{Robust ERM Theory} \\
		\hline
		& $p >= 2$ & $1 < p < 2$ \\
		\hline
		Truncated Loss  &~\citet{zhang2018ell, xu2020learning}  & ~\citet{chen2021generalized, xu2022non}  \\
		\hline
		Robust Loss Estimation  &~\citet{brownlees2015empirical} &  Ours \\
		\hline
		\hline
	\end{tabular}
\caption{A brief summary of robust empirical risk minimization theory.}\label{tab:summary}
\end{table}

\begin{table}[ht!]
    \centering
    \begin{tabular}{c|c|c}
        \hline
		\hline
		\multicolumn{3}{c}{Computational Approach} \\
		\hline
        Robust Gradient Descent &  Algorithm~\ref{alg:GD} & {\color{blue} slow} as parameter dimension goes large \\
        \hline
        \multirow{2}{*}{Empirical Risk-based Methods}  & Algorithm~\ref{alg:ERM} & parameter dimension-free\\ \cline{2-3}
        & Algorithm~\ref{alg:DW} & easy-implementable with DL framework \\
        \hline
        \hline
    \end{tabular}
    \caption{Advantages and disadvantages of three algorithms discussed in the appendix.}
    \label{tab:alg}
\end{table}

\newpage

\noindent \textbf{Technical Overview}.
The theory of bounding excess risk, $m_{\hat f} - m^{\ast}$, relies on the following points, (1)
constructing a good estimator ($\hat \mu_f$) of risk $m_f$ and (2) obtaining sharper uniform deviation bounds of $\sup_{f \in \mathcal F} |m_f - \hat \mu_f|$.

For the first point, we borrow the idea from the remarkable Catoni's estimator~\citep{Cat12}.
For any $f \in \mathcal F$, we consider the following estimator $\hat \mu_f$ to approximate $m_f = \mathbb E [f(X)]$ such that
$\hat \mu_f$ is the root of non-linear equation
\begin{eqnarray}\label{root}
0 = \hat r_f(\mu) = \frac{1}{n \alpha} \sum_{i=1}^n \phi(\alpha(f(X_i) - \mu)).
\end{eqnarray}
The influence function $\phi$ is non-linear and is assumed to satisfy \eqref{eq:C_HT2}
\begin{equation} \label{eq:C_HT2}
-\log(1 - x + C_{p} |x|^{p}) \leq \phi(x) \leq \log(1+x+C_{p} |x|^{p}),
\end{equation}
($C_p$ is a constant depending on moment $p$) and $\alpha$ is a tuning parameter which can be dependent on the sample size $n$.
Via using \eqref{root}, we define the new estimator as
\begin{eqnarray}\label{def:muf}
\hat f = \arg\min_{f \in \mathcal F} \hat \mu_f.
\end{eqnarray}
To understand the excess risk $m_{\hat f} - m^{\ast}$,
it is required to deal with the second point. This is due to the following fact
\[m_{\hat f} - m^{\ast} = (m_{\hat f} - \hat \mu_{\hat f}) +
(\hat \mu_{\hat f} - m^{\ast}) \leq 2 \sup_{f \in \mathcal F}
|m_f - \hat \mu_{f}|,\]
where we use the fact that both $\hat f$ and $f^{\ast} (:= \arg\min_{f \in \mathcal F} m_f)$ belong to space $\mathcal F$.

For this point, we follow the idea of \cite{brownlees2015empirical}.
The right hand side of above inequality depends on developing new theory of suprema of some empirical process.
We summarize the high level idea here. We first establish a sharp concentration bound of $|m_f - \hat \mu_f|$ when $f$ equals $f^{\ast}$. The calculation of such bound is more involved and sophisticated than that in case $p = 2$.
Next, thanks to the optimality of $\hat f$ and observations in~\citet{brownlees2015empirical}, we find that it only suffices to study term
$\frac{1}{\alpha} \mathbb E[\phi(\alpha(\hat f(X) - \mu))]$ with the choice of $\mu = \mu_0 := m_{f^{\ast}} + \epsilon$ where $\epsilon$ is a small constant number.
Finally, we find the fluctuation bounds of $\sup_{f} |X_f(\mu_0) - X_{f^{\ast}}(\mu_0)|$, where
$$X_f(\mu) := \frac{1}{n} \sum_{i=1}^n [\frac{1}{\alpha}\phi(\alpha(f(X_i) - \mu)) - \frac{1}{\alpha} \mathbb E[\phi(\alpha(f(X_i) - \mu))]],$$
via using generalized Talagrand's chaining method \citep{talagrand1996majorizing} to conclude the proof.
We here would like to highlight the main technical difference from \cite{brownlees2015empirical} that we \textbf{cannot directly apply} the Hoeffeding inequality / Bernstein inequality or use the classical Talagrand's chaining result since the data does not have the finite variance in our settings.
Thanks to the special construction of the influence function, we find that $\phi$ satisfies the Holder condition (see later explanation in \eqref{ass:holder}). It helps to simplify the upper bounds of $|X_f(\mu) - X_{f^{\ast}}(\mu)|$ which further allows us to establish Bernstein-type concentration inequality.
By introducing the generalized $\gamma_{\beta,p}$ functional for $p < 2$, we are able to establish the new chaining results.

On the other hand, for computational feasibility, we propose empirical risk-based methods, via introducing a novel way for calculating the robust gradient based on nice properties of Catoni-style influence function.
Specifically, we treat
\[\nu_i := \frac{\phi'(\alpha(f(X_i) - \hat \mu_{f}))}{\sum_i \phi'(\alpha(f(X_i) - \hat \mu_{f}))}\]
as the weight of $i$-th sample so that sample $i$ will get lower weight as $f(X_i)$ deviates from current risk estimate $\hat \mu_f$.
Then weighted average
(i.e., $\sum \nu_i \partial f(X_i)$ where $\partial f$ represents the derivative of $f$ with respect to model parameters) is adopted as the new gradient direction.
It is then shown that the iterates will converge to the local maximum of $\hat \mu_f$.
In addition, we further accelerate the algorithm by finding approximate value of $\hat \mu_f$ instead of solving Catoni-style non-linear equation.
All above steps can be easily embedded into any popular deep learning optimizers for practical use.
Reader can find more details in Section~\ref{sec:algorithm}.
In comparison with the robust gradient descent methods~\citep{holland2019better, holland2021learning}, we do not require to solve non-linear equations coordinate-wisely and it largely increases the computational efficiency.
Therefore, our method is parameter dimension-free in the sense that the relative computational time (with respect to vanilla gradient descent) is independent of number of model parameters.
By comparing with other truncated loss based methods~\citep{zhang2018ell, chen2021generalized}, our algorithm returns an estimator which is close to a (local) maximum point, while existing algorithms provides no theoretical results of gaps or relationships between their proposed estimator and the excess risk minimizer.

\vspace{0.1in}
\noindent \textbf{Notations}.
For moment order $p \in (1,2)$, we also occasionally write $p = 1 + \varepsilon$ ($0 < \varepsilon < 1$) for notational simplicity. We use $\phi(x)$ to represent the Catoni-style influence function.
We use $n$ to denote the sample size, use $\alpha$ as a tuning parameter to let $\delta \in (0,1)$ be a fixed confidence level.
Symbol $f$ is some generic loss function,
$\nabla f$ represents the first-order derivative of function $f$ and
$(\nabla f)[j]$ denotes its $j$-th coordinate.
$\mathbb E$ and $\mathbb P$ are used to denote the generic expectation and probability, respectively. $Y_n = O(X_n)$ and $ Y_n = o(X_n)$ represents that $Y_n \leq C X_n$ for some constant $C$ and  $Y_n / X_n \rightarrow 0$. For a random variable $X$ and $r\geq 1$ we use the notation $\|X\|_{\psi_r}$ for the Orlicz norm
$$
\|X\|_{\psi_r}=\inf\bigl\{c>0:\, E\exp\bigl\{ |X|^r/c^r\bigr\}\leq 2.
$$

\vspace{0.1in}
\noindent \textbf{Organization}.
The rest of paper is organized as follows.
In Section~\ref{sec:catoni}, we provide background of Catoni's estimator and corresponding technical prerequisites for its generalizations.
In Section~\ref{sec:mf}, we describe the main steps of how to derive a tight upper bound of excess risk, $m_{\hat f} - m^{\ast}$.
In Section~\ref{sec:Q}, we present our main theories via using empirical process theory to finalize the bound.
Several examples are given in Section~\ref{sec:example} to help readers to understand our results.
In Section~\ref{sec:algorithm}, we propose a new empirical risk-based gradient descent algorithm and discuss its advantages over existing methods.
Multiple experimental results, shown in Section~\ref{sec:sim}, corroborate our new theory and validate the effectiveness of proposed computational approaches.
The concluding remarks are given in Section~\ref{sec:conclusion}.

\newpage

\section{Catoni's Estimator with $p \in (1,2)$}\label{sec:catoni}

In this section, we present a generalized Catoni's estimator under weak moment condition.
To start with, we first recap the classical Catoni's estimator.
Considering a sequence of i.i.d random variables~$\{X_i\}_{i=1}^n$ be such that~$\mathbb{E}(X_1)=\mu$ and~$\mathbb{E}|X_1-\mu|^{2} \leq v$, Catoni~\citep{Cat12} introduced a robust mean estimator, $\hat \mu$, which is the solution to the following non-linear equation.
\begin{eqnarray}
\sum_{i=1}^n \phi(\alpha(X_i - \mu)) = 0\nonumber .
\end{eqnarray}
Here $\phi$ is an influence function which is non-decreasing and satisfies
$- \log(1 - x + x^2/2) \leq \phi(x) \leq \log(1 + x + x^2/2)$.
For general $p = 1 + \varepsilon \in (1,2)$, using the same idea, we similarly consider~$\phi: \mathbb{R} \rightarrow \mathbb{R}$ to be a non-decreasing influence function such that for all~$x \in \mathbb{R}$
\begin{equation} \label{eq:C_HT}
-\log(1 - x + C_{p} |x|^{p}) \leq \phi(x) \leq \log(1+x+C_{p} |x|^{p}),
\end{equation}
where $C_p$ is a constant depending on the moment order $p$.

\begin{lemma} \label{lem:nec.suf}
	A function~$\phi$ satisfying~\eqref{eq:C_HT} exists if and only if
	$
	C_{p} \geq \Big( \frac{\varepsilon}{1+\varepsilon} \Big)^{\frac{1+\varepsilon}{2}} \Big( \frac{1 - \varepsilon}{\varepsilon} \Big)^{\frac{1-\varepsilon}{2}}$.
\end{lemma}

\begin{remark}
	From now on we choose~$C_{p} = \Big( \frac{\varepsilon}{1+\varepsilon} \Big)^{\frac{1+\varepsilon}{2}} \Big( \frac{1 - \varepsilon}{\varepsilon} \Big)^{\frac{1-\varepsilon}{2}}$. When~$\varepsilon = 1$, we recover the coefficient in~\citet{Cat12}, namely $C_2 = 1/2$. (Here the standard convention~$0^0:=1$ applies.)
\end{remark}

Similarly, we define the generalized Catoni's M-estimator~$\widetilde{\mu}_{c}$ as a solution to the equation
\begin{equation} \label{eq:E_Ct_est.n}
\sum_{i=1}^n \phi\Big(\alpha (X_i - \mu) \Big) = 0
\end{equation}
using an influence function~$\phi$ satisfying~\eqref{eq:C_HT}. If the solution is not unique, choose~$\widetilde{\mu}_{c}$ to be the median solution.
We here make some requirements on $n, \alpha$. That is, in the sequel, we always assume
\begin{eqnarray}
C_p \alpha^{\varepsilon}(1 - h)^{-\varepsilon} &<& 1/2; \label{require:n:alpha1} \\
h^{-\varepsilon} \alpha^{1+\varepsilon} C_p v + \frac{\log(2/\delta)}{n}
&\leq& \frac{\varepsilon}{1 + \varepsilon}(1 - h)(\frac{1}{(1 + \varepsilon) C_p})^{1/\varepsilon}; \label{require:n:alpha2} \\
h^{-\varepsilon} C_p \alpha^{\varepsilon} v + C_p \alpha^{\varepsilon}(1 - h)^{-\varepsilon} + \frac{\log(2/\delta)}{\alpha n} & <& 1 \label{require:n:alpha3}
\end{eqnarray}
hold. Note that inequalities \eqref{require:n:alpha1}-\eqref{require:n:alpha3} are ($h, \delta$)-dependent. We therefore call \eqref{require:n:alpha1}-\eqref{require:n:alpha3} as
($h, \delta$)-condition.
In fact, the condition is very mild since that
(7) - (9) are easy to be satisfied when $n$ is large and $\alpha$ is small with any fixed $h$ and $\delta$.

\begin{remark}\label{rmk:h-delta}
In \eqref{require:n:alpha1}-\eqref{require:n:alpha3}, $\delta$ is confidence parameter. In most cases, $\delta$ can be simply taken as $0.05$.
On the other hand, $h$ is a tuning parameter in $(0,1)$ and it appears since ``$a + b x + c|x|^p = 0$"-type equation does not admit a closed form solution and we need to find an approximation to it.
In practice, we can simply treat $h = \frac{1}{2}$ for simplicity.
\end{remark}

\begin{theorem} \label{thm:Nprob_ineq}
	Let~$\{X_i\}_{i=1}^n$ be i.i.d random variables with mean~$\mu$ and $\mathbb{E}|X_1-\mu|^{1+\varepsilon} \leq v$. Let~$\delta \in (0,1),~\varepsilon \in (0,1)$ and $h \in (0,1)$.
	Assume that ($h, \delta$)-condition holds, then we have the Catoni's M-estimator~$\widetilde{\mu}_{c}$ satisfies
	\begin{eqnarray}
	|\widetilde{\mu}_{c} - \mu| \leq 2(h^{-\varepsilon} C_{p} \alpha^{\varepsilon} v + \frac{\log(2/\delta)}{\alpha n}).
	\end{eqnarray}
	with probability $1 - \delta$. Especially, we take
	\begin{equation} \label{eq:main_alpha}
	\alpha = \Big( \frac{\log(2/\delta)}{n C_{p} v} \Big)^{\frac{1}{1+\varepsilon}} h^{\frac{\varepsilon}{1+\varepsilon}},
	\end{equation}
	Then it holds
	\begin{equation} \label{eq:Main.n.c}
	|\widetilde{\mu}_{c} - \mu| \leq
	4 (C_{p} v)^{\frac{1}{1+\varepsilon}} h^{\frac{-\varepsilon}{1+\varepsilon}} \Big( \frac{\log(2/\delta)}{n} \Big)^{\frac{\varepsilon}{1+\varepsilon}}.
	\end{equation}
\end{theorem}

\begin{remark}
By treating $f(X_i)$ as $X_i$, $m_f$ as $\mu$ and $\hat \mu_{f}$ as $\tilde \mu_c$ in Theorem~\ref{thm:Nprob_ineq}, we then get the upper bound of $|\hat \mu_f - m_f|$ for any fixed $f \in \mathcal F$.
\end{remark}

\begin{remark}
    Here we would like to point out that generalizing Catoni's estimator to the case of $1 < p < 2$ is not a new idea. Estimation bound of $|\tilde \mu_c - \mu|$ is also considered in the existing literature. Our Theorem \ref{thm:Nprob_ineq} is an extended version of Theorem 2.6 in \cite{chen2021generalized} and is a parallel version of Theorem 3.2 in \cite{bhatt2022nearly}.
\end{remark}

\section{Bounding  $m_{\hat f} - m^{\ast}$}\label{sec:mf}

Given the results in Section~\ref{sec:catoni}, we provide main procedures of how to get the upper bound of $m_{\hat f} - m^{\ast}$ in this section.

\subsection{Finite  $\mathcal F$}
To start with, we first consider $\mathcal F$ to be a discrete family as a warm up.
We let $|\mathcal F|$ be the cardinality of $\mathcal F$, i.e., number of functions $f \in \mathcal F$.
First of all, note that
\begin{eqnarray}
m_{\hat f} - m^{\ast} &=& (m_{\hat f} - \hat \mu_{\hat f}) + (\hat \mu_{\hat f} - m^{\ast}) \nonumber \\
&\leq& (m_{\hat f} - \hat \mu_{\hat f}) + (\hat \mu_{f^{\ast}} - m^{\ast}) \label{def:hatf} \\
&\leq& 2 \sup_{f \in \mathcal F} |m_f - \hat \mu_f|. \label{ass:fast}
\end{eqnarray}
In above, \eqref{def:hatf} uses the definition of $\hat f$ is the minimizer of $\hat \mu_f$ and \eqref{ass:fast} uses the fact that $\hat f, f^{\ast}$ belongs to $\mathcal F$.
Furthermore, recalling that $\hat \mu_f$ is the solution to
$0 = \frac{1}{n \alpha} \sum_{i=1}^n \phi(\alpha(f(X_i) - \mu))$, we can apply Theorem~\ref{thm:Nprob_ineq} to $|\hat \mu_f - m_f|$ for each $f \in \mathcal F$.
By union bound, we will have the following result.
\begin{theorem}\label{thm:discrete}
	Let $\{X_{i}\}_{i=1}^n$ be a set of i.i.d. random variables.
	Assume $(h, \delta/|\mathcal F|)$-condition holds, $\sup_{f \in \mathcal F} \mathbb E|f(X_1)| \leq v$ and $\delta \in (0,1)$ and $h \in (0,1)$.
	We will have
	\begin{eqnarray}\label{bd:discrete}
	m_{\hat f} - m^{\ast} \leq 4 (h^{-\varepsilon} C_p \alpha^{\varepsilon} v +
	\frac{\log(2 |\mathcal F| / \delta)}{\alpha n})
	\end{eqnarray}
	with probability  at least $1 - \delta$.
\end{theorem}

Remark: we can optimize the bound in the right hand side of \eqref{bd:discrete} by choosing
$\alpha = \big( \frac{\log(2|\mathcal F|/\delta) h^{\varepsilon}}{n v C_p}\big)^{\frac{1}{1 + \varepsilon}}$.
Then we have
\[m_{\hat f} - m^{\ast} \leq 8 \bigg(\frac{\log(\frac{2|\mathcal F|}{\delta})}{n}\bigg)^{\frac{\varepsilon}{1 + \varepsilon}}
(\frac{h^{\varepsilon}}{v C_p})^{-\frac{1}{1 + \varepsilon}}.\]

\subsection{General $\mathcal F$}

However, in the general case, $\mathcal \mathcal F$ may not be finite. In other words, $|\mathcal F|$ is not well-defined. In this section, we seek an alternative approach to bounding $m_{\hat f} - m^{\ast}$.

A few additional useful quantities are given as follows. For convenience, we define
$$\hat r_f(\mu) := \frac{1}{n \alpha} \sum_{i=1}^n \phi(\alpha(f(X_i) - \mu))~ \text{and} ~ \bar r_f(\mu) := \frac{1}{\alpha} \mathbb E[\phi(\alpha(f(X) - \mu))].$$
That is, $\bar r_f(\mu)$ is the population version of $\hat r_f(\mu)$.
We further define
$X_f(\mu) := \hat r_f(\mu) - \bar r_f(\mu)$, let $\bar \mu_f$ be the solution to $\bar r_f(\mu) = 0$
and set
\begin{equation} \label{e:A}
A_{\alpha}(\delta) = 2 (h^{-\epsilon} C_p \alpha^{\epsilon} v + \frac{\log(2/\delta)}{\alpha n}).
\end{equation}
We define two approximate functions,
\begin{eqnarray}
B_f^{+}(\mu, \eta) &=& (m_f - \mu) +  C_p \alpha^{\epsilon} (1 - h)^{-\epsilon} |m_f - \mu|^p + h^{-\epsilon} C_p \alpha^{\epsilon} v + \eta, \nonumber \\
B_f^{-}(\mu, \eta) &=& (m_f - \mu) -  C_p \alpha^{\epsilon} (1 - h)^{-\epsilon} |m_f - \mu|^p - h^{-\epsilon} C_p \alpha^{\epsilon} v - \eta \nonumber ,
\end{eqnarray}
and let
\[\mu_{f}^{+}(\eta) = m_f + 2 h^{-\epsilon} C_p \alpha^{\epsilon} v + 2 \eta, ~~ \mu_{f}^{-}(\eta) = m_f - 2 h^{-\epsilon} C_p \alpha^{\epsilon} v - 2 \eta.\]
We additionally require that
\begin{eqnarray}\label{eq:cond:alpha:eta}
C_p \alpha^{\epsilon} (1 - h)^{-\epsilon} 2^p (h^{-\epsilon} C_p \alpha^{\epsilon} v + 2 \eta)^{p-1} < 1,
\end{eqnarray}
which is called \textit{$\eta$-condition}.
Here $\eta$ is a positive number which may be specified from place to place. We typically use $\eta$
specified in \eqref{eq:eta} below.
Under \eqref{eq:cond:alpha:eta}, it is easy to check that both $B^{+}_f(\mu, \eta) = 0$ and $B^{-}_f(\mu, \eta) = 0$ have at least one solution.
Furthermore,  it can be seen that
$\mu_f^{+}(\eta)$ is the upper bound of the smallest root of
$B_f^{+}(\mu, \eta)$ and
$\mu_f^{-}(\eta)$ is the lower bound of the largest roof of
$B_f^{-}(\mu, \eta)$.

\begin{remark}\label{rmk:eta}
Here $\eta$-condition is mild.
It is not hard to see that the left hand side of \eqref{eq:cond:alpha:eta} is an increasing function of $\alpha$.
Moreover, $\alpha$ is usually taken as $(1/n)^{1/(1 +\epsilon)}$ in applications.
Hence, $\eta$-condition holds once sample size $n$ is sufficiently large.
\end{remark}

By a closer look of structure of $\bar r_f(\mu)$, we claim the following lemma.

\begin{lemma}\label{lem:key:obs1}
    Let $\mu_0 := m_{f^{\ast}} + A_{\alpha}(\delta)$. Then it holds
    \begin{eqnarray}
    \bar r_{\hat f}(\mu_0) \leq 2 L_p A_{\alpha}(\delta) + Q(\mu_0, \delta) =: \eta, \label{eq:eta}
    \end{eqnarray}
    where $Q(\mu, \delta)$ is the $1 - \delta$ quantile of $\sup_{f \in \mathcal F}|X_f(\mu) - X_{f^{\ast}}(\mu)|$ by $Q(\mu, \delta)$, i.e., the minimum possible $q$ satisfying that
\[\mathbb P(\sup_{f \in \mathcal F} |X_f(\mu) - X_{f^{\ast}}(\mu)| \leq q ) \geq 1 - \delta.\]
\end{lemma}

By
$\eta =  2 L_p A_{\alpha}(\delta) + Q(\mu_0, \delta)$ in \eqref{eq:eta}
and re-arrangement, we then have
\begin{eqnarray}
m_{\hat f} - m_{f^{\ast}} &\leq& 2 h^{-\epsilon} C_p \alpha^{\epsilon} v+ A_{\alpha}(\delta)+
4 L_p A_{\alpha}(\delta) + 2 Q(\mu_0, \delta) \nonumber \\
&\leq& 6 L_p(h^{-\epsilon} C_p \alpha^{\epsilon} v + \frac{\log(2/\delta)}{\alpha n}) + 2 Q(\mu_0, \delta). \label{e:gen.cat}
\end{eqnarray}
In the next section, we focus on working with $Q(\mu_0, \delta)$.
The final bounds are given as in \eqref{eta1} and \eqref{eta2}.
By the careful choice of $\alpha$'s order as $n^{-1/(1 + \epsilon)}$,
the excess risk is
$O(\big(\frac{1}{n}\big)^{\epsilon/(1 + \epsilon)})$.

\section{Working with $Q(\mu_0, \delta)$}\label{sec:Q}

In this section, our goal is to find the upper bound of $Q(\mu_0, \delta)$, which requires to compute tail probability $\mathbb P(\sup_{f \in \mathcal F} |X_f(\mu) - X_{f^{\ast}}(\mu)| > t)$.
It relies on the generic chaining technique \citep{talagrand1996majorizing}.
To introduce the main results,
we need to describe the geometric structure of class $\mathcal F$ under different distances. In particular, the $L_p$ distance is defined as
$d_p(f, f') = (\mathbb E[|f(X) - f'(X)|^p])^{1/p}$,
the $L_{\infty}$ distance is given as
$D(f,f') = \text{ess-sup}_{x \in \mathcal X} |f(x) - f'(x)|$,
and the sample $L_p$-distance is
$d_{X,p}(f,f') = (\frac{1}{n} \sum_{i=1}^n |f(X_i) - f'(X_i)|^p)^{1/p}$.
For any metric space $T$ equipped with a distance $d$,
we define $\textrm{diam}_d(T)$ to be the diameter of space $T$ under metric $d$.
We additionally define the quantity $\gamma_{\beta, p}(T,d)$ ($\beta = 1, 2$) as
\begin{eqnarray}\label{eq:gamma}
\gamma_{\beta, p}(T,d) = \inf_{\mathcal A_n} \sup_{t \in T} \sum_{n \geq 0} 2^{n/\beta} (\Delta_d(A_n(t)))^{p/2},
\end{eqnarray}
where $(\mathcal A_n)$ is an increasing sequence of partitions of $T$ and is \textit{admissible} if for all $n \geq 0$, $|\mathcal A_n| \leq 2^{2^n}$. For any $t \in T$, $A_n(t)$ is the unique element of $\mathcal A_n$ that contains $t$. $\Delta_d(A)$ denotes the diameter of the set $A \subset T$ given metric $d$. Sometimes, we may also write $\Delta(A)$ instead of $\Delta_{d}(A)$ when there is no ambiguity.
Here functional $\gamma_{\beta,p}(T,d)$ is the generalized version of $\gamma_{\beta}(T,d)$ (when $p = 2$) which appears in the literature of generic chaining methods~\citep{talagrand1996majorizing}.
Although $\gamma_{\beta,p}(T,d)$ is hard to compute based on its definition, an upper bound on $\gamma_{\beta,p}(T,d)$ can be obtained in the following lemma.
\begin{lemma}\label{lem:chain}
	For any $1 < p \leq 2$ and $\beta$, there exists a constant $C_{\beta, p}$ such that
	\begin{eqnarray}
	\gamma_{\beta, p}(T,d) \leq C_{\beta, p} \int_0^{\infty} \epsilon^{p/2 - 1} (\log N(T,d,\epsilon/2))^{1/\beta} d \epsilon, \label{eq:bound:gamma}
	\end{eqnarray}
	where $N(T,d,\epsilon)$ is the $\epsilon$-covering of space $(T,d)$.
\end{lemma}

Lemma~\ref{lem:chain} gives us a way to upper bound $\gamma_{\beta,p}(T,d)$, whose value is finite once the covering number is integrable in the sense of \eqref{eq:bound:gamma}.

\subsection{Bound 1 on $Q(\mu, \delta)$}

In the rest of paper, we further consider some special type of Catoni influence function which satisfies a \textit{H\"older condition} that
\begin{eqnarray}\label{ass:holder}
|\phi(x_1) - \phi(x_2)| \leq C_{3p} |x_1 - x_2|^{p/2} ~~ \text{for any}~ x_1, x_2 \in \mathbb R.
\end{eqnarray}
The above H\"older requirement is very mild that we can easily construct influence functions to satisfy \eqref{ass:holder}. For example, by Lemma \ref{lem:lip}, the following two special functions have Holder continuity.
\begin{itemize}
	\item[1.] (unbounded case):
	\begin{eqnarray}\label{eq:example1}
	\phi_1(x) =
	\begin{cases}
	\log(1 + x + C_p|x|^p) &~ x \geq 0\\
	-\log(1 - x + C_p|x|^p) &~ x < 0.\\
	\end{cases}
	\end{eqnarray}
	\item[2.] (bounded case):
	\begin{eqnarray}\label{eq:example2}
	\phi_2(x) =
	\begin{cases}
	- \log(1 + A_2 + C_p A_2^p) & \text{if}~ x \leq - A_2 \\
	-\log(1 - x + C_p|x|^p) &  \text{if} ~ - A_2 \leq x \leq 0, \\
	\log(1 + x + C_p|x|^p) &  \text{if} ~ 0 < x \leq A_1, \\
	\log(1 + A_1 + C_p A_1^p) & \text{if}~ x \geq A_1.
	\end{cases}
	\end{eqnarray}
	with both $|A_1|, |A_2| > (p C_p)^{-\frac{1}{p-1}}$
\end{itemize}
By taking
$A^\prime_{\alpha}(\delta) = 2 (h^{-\epsilon} C_p \alpha^{\epsilon} 2^pv + \frac{\log(2/\delta)}{\alpha n})$,
we have the following result.

\begin{theorem}\label{thm:bound1}
	Let $\{X_{i}\}_{i=1}^n$ be a set of i.i.d. random variables with $\sup_{f \in \mathcal F} \mathbb E|f(X_1)|^p \leq v$.
	Let $\delta \in (0,1)$ and $h \in (0,1)$ and suppose that the  $(h, \delta)$-condition and the $\eta$-condition  with
$\eta = 2L_pA^\prime_\alpha(\delta) +
 Q_1(\delta)$
 and $Q_1(\delta)$ defined in \eqref{def:Q1} below,
 hold, and  assume that the influence function satisfies the H\"older condition \eqref{ass:holder} .
 Then we have
	\begin{eqnarray}
	m_{\hat f} - m^{\ast} \leq 6 L_p\left(h^{-\epsilon} C_p \alpha^{\epsilon} 2^p v + \frac{\log(2/\delta)}{\alpha n}\right) + 2 Q_1(\delta)    \label{eta1}
	\end{eqnarray}
	with probability  at least $1 -  2\delta$. Here
	\begin{eqnarray}\label{def:Q1}
	Q_1(\delta) = 384 C_{3p}\log 2 \,   \log (2 /\delta) \left(\frac{2 \alpha^{p/2-1}}{3 n} \gamma_{1,p}(\mathcal F, D) + \sqrt{\frac{\alpha^{p-2}}{n}} \gamma_{2,p}(\mathcal F, d_p)\right) .    
	\end{eqnarray}
\end{theorem}

\begin{remark}
In particular, if we take $\alpha = (\frac{\log(2/\delta)}{n v})^{\frac{1}{1 + \varepsilon}}$, then
\[m_{\hat f} - m^{\ast} = O_p\bigg( (\frac{1}{n})^{\frac{\varepsilon}{1 + \varepsilon}} \cdot \big( (\log(2/\delta))^{\frac{\varepsilon}{1 + \varepsilon}} v^{1/(1 + \varepsilon)} + (\frac{\log(2/\delta)}{v})^{\frac{\varepsilon - 1}{2(1 + \varepsilon)}} \gamma_{2,p}(\mathcal F, d_p) \big)\bigg).\]
Here the term  $\gamma_{1,p}(\mathcal F, D)$ disappears since it contains an $n$-dependent term of a smaller order.
\end{remark}

\subsection{Bound 2 on $Q(\mu, \delta)$}

Theorem \ref{thm:bound1} can be used when the sup norm $D$ is finite.
In this section, we establish a second upper bound on $Q(\mu, \delta)$, which does not rely on $D$.
Instead, we use the sample $L_p$-distance and define
\[\Gamma_{\delta} := \min \{ c : \mathbb P(\gamma_{2,p}(\mathcal F, d_{X,p})  > c) \leq \frac{\delta}{8}\}\]
to measure the span of space $(\mathcal F, d_{X,p})$.

\begin{theorem}\label{thm:bound2}
	Let $\{X_{i}\}_{i=1}^n$ be a set of i.i.d. random variables with $\sup_{f \in \mathcal F} \mathbb E|f(X_1)|^p \leq v$.
Let $\delta \in (0,1)$ and $h \in (0,1)$.
	Suppose that the $(h, \delta)$-condition and the $\eta$-condition with
$\eta = 2L_pA^\prime_\alpha(\delta) +
 Q_2(\delta)$
 and $Q_2(\delta)$ defined in \eqref{def:Q2} below,  hold and assume that the influence function satisfies the H\"older condition \eqref{ass:holder} .
 Then, with probability at least $1 - 2 \delta$, we have
	\begin{eqnarray}
	& & m_{\hat f} - m^{\ast} \leq 6 L_p\left(h^{-\epsilon} C_p \alpha^{\epsilon} 2^pv + \frac{\log(2/\delta)}{\alpha n}\right) + 2 Q_2(\delta)    \label{eta2} \\
	\text{with} ~~  & & Q_2(\delta) =  K\max\{\Gamma_{\delta}, (\textrm{diam}_{d_p}(\mathcal F))^{p/2}\} \sqrt{\frac{\log(8/\delta)}{n \alpha^{2-p}}}  \label{def:Q2}
	\end{eqnarray}
 for a universal constant $K$.
\end{theorem}

\newpage

\section{Illustrative Examples}\label{sec:example}

In this section, we provide several examples to help readers to understand our theoretical results.

\noindent
\textbf{$L_1$ regression.}
In this setting, we let $\mathcal F = \{  f_g(z,y) = |g(z) - y|: g \in \mathcal G\}$ and assume that $E\bigl| g(Z)-Y|^p\leq v$ for every $g\in \mathcal G$.
For the maximum distance, since
\[D(f_g,f_{g'}) = \sup_{z,y}| |g(z) - y| - |g'(z) - y| | \leq  d_\infty(g,g'),\]
the covering number of $\mathcal F$ under the distance $D$ is bounded by the covering number of $\mathcal G$ under the sup norm.
Similarly, for the norm $d_p$, we have
\[d_p(f_g,f_{g'}) = (\mathbb E|f_g(X) - f_{g'}(X)|^p)^{1/p}
\leq (\mathbb E|g(z) - g'(z)|^p)^{1/p} = d_p(g, g').\]
Hence the covering number of $\mathcal F$ under distance the $d_p$ is bounded by the covering number of $\mathcal G$ under the same distance.
Applying Theorem \ref{thm:bound1}, we obtain the following result.
\begin{proposition}
	In the $L_1$-regression problem,
 \begin{eqnarray}
	m_{\hat f} - m^{\ast}
	\leq 6 L_p \left(h^{-\varepsilon} C_p \alpha^{\varepsilon} 2^pv + \frac{\log(2/\delta)}{\alpha n}\right)
	+ C' \log(2/\delta)
	\Big( \frac{2 \alpha^{p/2-1}}{3n} \gamma_{1,p}(\mathcal G, d_{\infty}) + \sqrt{\frac{\alpha^{p-2}}{n}} \gamma_{2,p}(\mathcal G, d_{p}) \Big)
	\end{eqnarray}
	with probability $1 - \delta$ for any $n$ that satisfies the $(h, \delta)$ condition and the $\eta$-condition  for
 $$
 \eta = 2L_pA^\prime_\alpha(\delta)
 +  C' \log(2/\delta)
	\Big( \frac{2 \alpha^{p/2-1}}{3n} \gamma_{1,p}(\mathcal G, d_{\infty}) + \sqrt{\frac{\alpha^{p-2}}{n}} \gamma_{2,p}(\mathcal G, d_{p}) \Big) ~\text{with}~  C' = 384 C_{3p}\log 2.
 $$
\end{proposition}

\noindent
\textbf{$L_2$ regression.}
In this setting,
we let $\mathcal F = \{f_g(z,y) = (g(z) - y)^2: g \in \mathcal G\}$ with $d_{\infty}(g,g')$ being bounded and apply Theorem \ref{thm:bound2} with some straightforward calculations to get the next result.

\begin{proposition}\label{prop:l2}
	In the L2-regression problem, it holds that with $\mathbb E[|f_g|^p] \leq \infty$ for any $f_g \in \mathcal F$.
	\begin{eqnarray}\label{L2:slow}
	&& m_{\hat f} - m^{\ast}  \\
	&\leq& 6 L_p (h^{-\varepsilon} C_p \alpha^{\varepsilon} v + \frac{\log(2/\delta)}{\alpha n})
	+ L_2 C_{3p} \sqrt{\frac{\log(8/\delta)}{n \alpha^{2-p}}}
	(\Delta^p + \mathbb E[|Y|^p] + \sqrt{8v / n\delta})^{1/p} \gamma_{2,p}(\mathcal G, d_{\infty}) \nonumber
	\end{eqnarray}
	with probability $1 - 2 \delta$ for any $n \geq N_0$. ($N_0$ is still a positive constant satisfying $(h, \delta)$ condition and $\eta$-condition and $\Delta$ is a positive constant larger than $\text{diam}_{d_p}(\mathcal F)^{p/2}$.)
\end{proposition}

\begin{remark}
    The above result applies to the special linear model $Y = \beta^T X + \epsilon$ with $\mathbb E[|\epsilon|^{2p}] < \infty$. Compared to the state of art result \citep{hsu2016loss}, our result is established under even weaker moment condition, that is, fourth moment of error term $\epsilon$ does not exist.
\end{remark}

\noindent \textbf{Kernel Learning}.
Consider the following optimaztion problem,
\begin{eqnarray}
\hat f = \arg\min_{f = L \circ h \in \mathcal F} \bigg\{ \hat \mu_f + \lambda_n \|h\|_{\mathcal H}^2 \bigg\}, \label{obj:kernel}
\end{eqnarray}
where $\mathcal F = L \circ \mathcal H$, where $L$ is a deterministic loss function and $\mathcal H$ is a reproducing kernel Hilbert space (RKHS) associated with kernel $K(x,y)$.
In this section, we assume
$L \circ \mathcal H $ is $L_p$-integrable and takes the form that $L(Y - h(X))$,
loss function $L$ satisfies that
$|L(Y - h_1(X)) - L(Y - h_2(X))| \leq C(Y) |h_1(X) - h_2(X)|$  for any $h_1, h_2 \in \mathcal H$ where $C(Y)$ is an square integrable function.
Kernel $K$ is assumed to be a mercer kernel.
Moreover, without loss of generality, we can always assume the true underlying $h^{\ast}$ has bounded norm and particularly we assume
$\| h^{\ast}\|_{\mathcal H}^2 \leq 1$.

\begin{proposition}\label{prop:kernel}
	In the kernel regression problem, it holds that
	\begin{eqnarray}
	m_{\hat f} - m^{\ast}
	&\leq& 6 L_p (h^{-\varepsilon} C_p \alpha^{\varepsilon} v + \frac{\log(2/\delta)}{\alpha n} + \lambda_n) \nonumber \\
	& & + L_1 C_{3p} \log (2 /\delta) (\frac{2 \alpha^{p/2-1}}{3 n} \gamma_{1,p}(L \circ \mathcal H, D) + \sqrt{\frac{\alpha^{p-2}}{n}} \gamma_{2,p}(L \circ \mathcal H, d_p))
	\end{eqnarray}
	with probability $1 - 2 \delta$ for any $n \geq N_0$. ($N_0$ is a large constant satisfying $(h, \delta)$ and $\eta$-condition.)
\end{proposition}

\begin{remark}\label{rmk:ridge}
	By taking $\mathcal H = \{h(Z)~|~\beta^T Z, , \beta \in \mathbb R^d\}$ with kernel $K(f_1, f_2) = \beta_1 \cdot \beta_2$, $\mathcal F = \{f(X) | (Y - h(Z))^2; h \in \mathcal H \}$ and influence function $\phi(x) = x$, then \eqref{obj:kernel} is reduced to the standard ridge regression
	\[\arg\min_{\beta} \frac{1}{n} \sum_{i=1}^n (Y_i - \beta^T Z_i)^2 + \lambda_n \|\beta\|_2^2. \]
\end{remark}

\begin{remark}\label{rmk:DL}
In deep learning, the RKHS can be taken as the space spanned by ReLU functions.
\end{remark}

\section{Computations}\label{sec:algorithm}

In this section, we discuss the computational issues of finding the optimizer $\hat f \in \mathcal F$. In the sequel, to facilitate our life, we only consider the case that $\mathcal F$ is a parametric family. That is, $\mathcal F = \{f_{w}: w \in \mathbb R^d\}$, $d$ is dimension of parameter vector $w$. Since there are infinitely many candidates for $w$, it is not desired to directly solve
$\hat \mu_{f_w}$ from \eqref{def:muf} for all $f_{w} \in  \mathcal F$ to get $f_{\hat w}$.
Therefore, we want to design an algorithm which returns $f_{\tilde w}$ such that $m_{f_{\tilde w}}$ shares similar properties as $m_{f_{\hat w}}$.

We know that $w^{\ast} = \arg\min_w m_{f_w} = \arg\min_w \mathbb E[f_w(X)]$.
To solve $w^{\ast}$, it is popular to use gradient descent methods~\citep{chong2004introduction, lemarechal2012cauchy, ruder2016overview} once we know the explicit formula of $\mathbb E[f_w(X)]$.
The gradient of $\mathbb E[f_w(X)]$ is $\nabla \mathbb E[f_w(X)]$ which is equal to $\mathbb E[ \nabla f_w(X)]$ provided that $|\nabla f_w(X)|$ is integrable.
In the rest, we discuss on two types of tractable methods for computing $\hat f$.

\subsection{On Robust Gradient Descent Method}\label{sec:robustGD}

We first present a computational method via using robust gradient.
The core idea of such method is to estimating
$\mathbb E[\nabla f_{w}(X)]$ via using robust techniques.
This type of approach is interesting since robust gradient can be embedded to any machine learning optimizer. Similar type of algorithm is also considered in the literature~\citep{holland2019better, holland2021learning} when moment order $p \geq 2$. The procedure is described in Algorithm~\ref{alg:GD}.

\begin{algorithm}[ht]
 \hrulefill
	\caption{Robust Gradient Descent Method}
	\label{alg:GD}
	\begin{algorithmic}[1]
		\STATE {\bfseries Input:}
		Observations: $\{X_i, i \in \{1, \ldots, n\}\}$. A bounded Catoni influence function $\phi$.
		\STATE {\bfseries Output:}
		Estimated parameter: $\tilde w$
		\STATE {\bfseries Initialization:} Randomly choose $w^{(0)}$ from $\mathbb R^d$ and set time index $t = 0$.
		\WHILE{not converged}
		\STATE [\textbf{Robust gradient}]~ Compute robust gradient $g^{(t)}$ by solving
		\begin{eqnarray}
		\sum_{i=1}^n \phi(\alpha(\nabla f_{w^{(t)}}(X_i)[j] - g^{(t)}[j])) = 0
		\end{eqnarray}
		\textbf{coordinate-wisely} for $j \in [d]$.
		\STATE Update parameter by $w^{(t+1)} = w^{(t)} - \gamma_t g^{(t)}$.
		\STATE Increase time index $t = t+1$.
		\ENDWHILE
		\STATE Return parameter estimate $\tilde w = w^{(t)}$.
	\end{algorithmic}
  \hrulefill
\end{algorithm}

To be self-contained, we also provide the complete theoretical analysis of Algorithm~\ref{alg:GD}. The required assumptions on function $f \in \mathcal F$ are given as below.

\begin{itemize}
    \item[]
    \textbf{A0}.  It holds that $\mathbb E[|\nabla f_w(X)[j]|^p] \leq v$ for any $f_w \in \mathcal F$ and $j \in [d]$.
    \item[]
    \textbf{A1}.  For any $\eta \in (0,1)$, there exists a constant $R_{\eta}$ such that it holds $|\nabla f_{w_1}(X) - \nabla f_{w_2}(X)| \leq R_{\eta} |w_1 - w_2|$ with $1 - \eta$ probability.
    \item[]
    \noindent \textbf{A2}.  Let $F(w)$ be $\mathbb E[f_w(X)]$. It is assumed that $F(w)$ is $L$-lipschitz continuous.
\end{itemize}

Assumption A0 ensures that each coordinate of gradient satisfies the weaker moment condition.
Assumption A1 here is mild and is weaker than the bounded Lipschitz condition, that is, $|\nabla f_{w_1}(X) - \nabla f_{w_2}(X)| \leq R |w_1 - w_2|$ for any $X$.
Assumption A2 is a smoothness condition on loss function in population version.

We further define term $\bar A_{\alpha}(\delta) =A_{\alpha}(\frac{\delta}{(2 D_w (R_{\eta} +L_f) \alpha /(32 A \eta))^d})$, where $\eta$ is any constant satisfies
$64 \eta/\alpha < A_{\alpha}(\delta)$.
Here $\bar A_{\alpha}(\delta)$ can be viewed as the uniform error for controlling the difference between $\nabla f(w^{(t)})$ and its expectation.
In particular, we can take $\alpha = n^{-1/p}$ and $\eta = O(1/n)$, then $R_{\eta} = \Theta(n^{-1/p})$ and $\bar A_{\alpha}(\delta)$ becomes $O(n^{-\frac{p-1}{p}} (d \log n + \log(1/\delta))^{\frac{p-1}{p}})$.
We then establish the following convergence results.

\begin{theorem}\label{thm:convergence:alg}
	Under Assumptions A0-A1 and let $\delta \geq \exp\{-n/4\}$ and $\gamma_t \equiv \gamma \leq \frac{4}{9 L}$, then Algorithm~\ref{alg:GD} will return an estimator $\hat w$ ,with $1 - 2\delta$ probability, it holds that
	 \[T_{stop} \leq \frac{18(m_{f_{w^{(0)}}} - m_{f_{w^{\ast}}})}{5 \gamma d (\bar A_{\alpha}(\delta))^2},\]
	 where	$T_{stop} := \min \{t: \|\nabla F(w^{(t)})\| \leq \sqrt{d} \bar A_{\alpha}(\delta) \}$.
	
	If $f$ is additionally assumed to be $\kappa$-strictly convex, then it holds with probability $1 - 2 \delta$ that
	\[\|w^{(t+1)} - w^{\ast}\| \leq (1 - \sqrt{\frac{2 \gamma \kappa L}{\kappa + L}})^{t+1} \|w^{(0)} - w^{\ast}\| + \sqrt{d}\bar A_{\alpha}(\delta) \frac{\gamma}{1 - \sqrt{\frac{2 \gamma \kappa L}{\kappa + L}}}\]
 for any $0 \leq t < T_{stop}$.
\end{theorem}

\begin{remark}
    The proof technique of Theorem~\ref{thm:convergence:alg} is similar to that of Theorem 5 in~\citet{holland2019better}.
    However, our case $1 < p < 2$ requires more involved computation than case $p \geq 2$ does.
    In addition, results in~\citet{holland2019better} are established under stronger conditions, i.e., $\sup_{X}|\nabla f_{w_1}(X) - \nabla f_{w_2}(X)| \leq R |w_1 - w_2|$, which is not generally true when the support of $X$ is unbounded.
\end{remark}

\newpage

\subsection{Acceleration: Finding $\hat f_w$ via risk minimization}

Unfortunately, computing an robust estimate of $\mathbb E[\nabla f_{w}(X)]$ is generally inefficient when dimension $d$ goes extremely large.
We need to go around with this issue by seeking other types of approaches.

By previous section, our ultimate goal is to find out the minimizer, $\arg\min_{w} m_{f_w}$.
It is necessary to find the solution of $\nabla_w m_{f_w} = \mathbf 0$.
However, the explicit formula of $m_{f_w}$ is unknown to us. We then look for the solution of $\nabla_{w} \hat \mu_{f_w} = \mathbf 0$ instead.
Recall the following identity,
\begin{eqnarray}
0 = \frac{1}{n \alpha} \sum_i^n \phi(\alpha(f_{w}(X_i) - \hat \mu_{f_w})).
\end{eqnarray}
Taking derivative with respect to $w$ on both sides, we have
\begin{eqnarray}
\mathbf 0 &=& \frac{1}{n \alpha} \frac{\partial \{\sum_i^n \phi(\alpha(f_{w}(X_i) - \mu_{f_w}))\}}{\partial w} \nonumber \\
&=& \frac{1}{n} \sum_i \phi'(\alpha(f_w(X_i) - \hat \mu_{f_w})) \big( \frac{\partial f_w(X_i)}{\partial w} - \nabla_w \hat \mu_{f_w} \big).
\end{eqnarray}
With re-arrangement, we arrive at
\begin{eqnarray}
\nabla_w \hat \mu_{f_w} = \frac{\sum_i \phi'(\alpha(f_w(X_i) - \hat \mu_{f_w})) \frac{\partial f_w(X_i)}{\partial w} }{\sum_i \phi'(\alpha(f_w(X_i) - \hat \mu_{f_w}))}.
\end{eqnarray}
By gradient descent, $w^{(t+1)} = w^{(t)} - \gamma_t \nabla_w \hat \mu_{f_w^{(t)}}$, we can find the stationary point $\tilde w$ such that $\nabla_w \hat \mu_{f_w}|_{w = \tilde w} = 0$.
Therefore, we have the following proposition.

\vspace{0.1in}

\begin{proposition}
If $f_{\hat w}$ is the optimizer as defined in \eqref{def:muf}, then it holds that
	$\nabla_{w} \hat \mu_{f_w}|_{w = \hat w} = \mathbf 0$.
\end{proposition}

\vspace{0.1in}

By above reasons, we naturally have the following empirical risk-based gradient descent algorithm, i.e., Algorithm~\ref{alg:ERM}.
Some remarks are explained here.
Since $\phi$ is non-decreasing, then $\phi'(\cdot)$ is always non-negative.
We can view $\frac{\phi'(\alpha(f_{w^{(t)}}(X_i) - \hat \mu_{f_w}))}{\sum_i \phi'(\alpha(f_{w^{(t)}}(X_i) - \hat \mu_{f_{w^{(t)}}}))}$ as the weight of $i$-th sample at iteration $t$.
Therefore, larger $f_{w^{(t)}}(X_i)$ has smaller weight, thanks to the construction of $\phi$.
Compared with robust gradient descent method (see Section~\ref{sec:robustGD}),
we only need to solve non-linear equation for \textbf{one time} in each iteration instead of computing robust gradient coordinate by coordinate.
Hence Algorithm~\ref{alg:ERM} is more computationally friendly than Algorithm~\ref{alg:GD}.
In the field of robust statistics, a similar type of Algorithm is considered in \cite{mathieu2021excess}.

\begin{algorithm}[ht]
  \hrulefill
	\caption{Empirical Risk Gradient Descent}
	\label{alg:ERM}
	\begin{algorithmic}[1]
		\STATE {\bfseries Input:}
		Observations: $\{X_i, i \in \{1, \ldots, n\}\}$.
		\STATE {\bfseries Output:}
		Estimated parameter: $\tilde w$
		\STATE {\bfseries Initialization:} Randomly choose $w^{(0)}$ from $\mathbb R^d$ and set time index $t = 0$.
		\WHILE{not converged}
		\STATE Find $\hat \mu_{f_{w^{(t)}}}$ by solving
        \begin{eqnarray}
        0 = \frac{1}{n \alpha} \sum_{i=1}^n \phi(\alpha(f_{w^{(t)}}(X_i) - \mu)). \label{eq:sol:alg}
        \end{eqnarray}
		\STATE Compute gradient $g^{(t)}$ by
		\begin{eqnarray}
		g^{(t)} = \nabla_w \hat \mu_{f_w^{(t)}} = \frac{\sum_i \phi'(\alpha(f_{w^{(t)}}(X_i) - \hat \mu_{f_w^{(t)}})) \frac{\partial f_{w^{(t)}}(X_i)}{\partial w} }{\sum_i \phi'(\alpha(f_{w^{(t)}}(X_i) - \hat \mu_{f_w^{(t)}}))}. \label{eq:update_novel}
		\end{eqnarray}
		\STATE Update parameter by $w^{(t+1)} = w^{(t)} - \gamma_t g^{(t)}$.
		\STATE Increase time index $t = t+1$.
		\ENDWHILE
		\STATE Return parameter estimate $\tilde w = w^{(t)}$.
	\end{algorithmic}
   \hrulefill
\end{algorithm}

We can see that Algorithm~\ref{alg:ERM} requires to solve the non-linear equation \eqref{eq:sol:alg}.
To further accelerate the whole estimation procedure, we compute an approximate of $\hat \mu_{f_{w^{(t)}}}$
instead of computing it exactly.
At $(t+1)$-th step, such approximation is obtained via the following recursive formula
\[\hat \mu^{(t+1)} = \hat \mu^{(t)} + \sum_i \nu_i^{(t)} (f_{w^{(t+1)}}(X_i) - f_{w^{(t)}}(X_i)),\]
where $\hat \mu^{(t)}$ is viewed as the proxy of $\hat \mu_{f_{w^{(t)}}}$
and $\nu_i^{(t)} := \frac{\phi'(\alpha(f_{w^{(t)}}(X_i) - \hat \mu_{f_w}))}{\sum_i \phi'(\alpha(f_{w^{(t)}}(X_i) - \hat \mu_{f_{w^{(t)}}}))}$ is the weight of sample~$i$.

\newpage

The full procedure is summarized in Algorithm~\ref{alg:DW}.
Note that weights $\{\nu_i^{(t)}\}$'s have been used \emph{twice} for computing gradient $g^{(t)}$ and approximated risk $\hat \mu^{(t+1)}$.
Therefore, we call this method as \textbf{double-weighted} gradient descent algorithm.

\begin{algorithm}[ht]
  \hrulefill
	\caption{Double-weighted Gradient Descent}
	\label{alg:DW}
	\begin{algorithmic}[1]
		\STATE {\bfseries Input:}
		Observations: $\{X_i, i \in \{1, \ldots, n\}\}$; A catoni-influence function $\phi$; Stopping threshold $\tilde \varrho$.
		\STATE {\bfseries Output:}
		Estimated parameter: $\tilde w$
		\STATE {\bfseries Initialization:} Set time index $t = 0$.  Randomly choose $w^{(0)}$ from $\mathbb R^d$, choose an $\mu^{(0)} \in \mathbb R^{+}$ such that $\hat \mu^{(0)}$ is an $\alpha$-approximate solution to
		\begin{eqnarray}
		0 = \frac{1}{n} \sum_{i=1}^n \phi(\alpha(f_{w^{(0)}}(X_i) - \mu))
		\end{eqnarray}
		Choose weights to be
		\[\nu_i^{(0)} = \frac{\phi'(\alpha(f_{w^{(0)}}(X_i) - \hat \mu^{(0)})) }{\sum_i \phi'(\alpha(f_{w^{(0)}}(X_i) - \hat \mu^{(0)}))}.\]
		\WHILE{not converged}
		\STATE Compute gradient $g^{(t)}$ by
		\begin{eqnarray}
		g^{(t)} = \sum \nu_i^{(t)} \frac{\partial f_{w^{(t)}}(X_i)}{\partial w}. \label{eq:update_novel2}
		\end{eqnarray}
		\STATE Update parameter by $w^{(t+1)} = w^{(t)} - \gamma_t g^{(t)}$.
		\STATE Find $\hat \mu^{(t+1)}$ by computing
        \begin{eqnarray}
        \hat \mu^{(t+1)} = \hat \mu^{(t)} + \sum_i \nu_i^{(t)} (f_{w^{(t+1)}}(X_i) - f_{w^{(t)}}(X_i)).
        \label{eq:double:weight}
        \end{eqnarray}
		\STATE Compute weights
		\[\nu_i^{(t+1)} = \frac{\phi'(\alpha(f_{w^{(t+1)}}(X_i) - \hat \mu^{(t+1)})) }{\sum_i \phi'(\alpha(f_{w^{(t+1)}}(X_i) - \hat \mu^{(t+1)}))}.\]
		\STATE Increase time index $t = t+1$.
		\ENDWHILE
		\STATE Return parameter estimate $\tilde w = w^{(t)}$.
	\end{algorithmic}
   \hrulefill
\end{algorithm}

We further prove the convergent property of Algorithm~\ref{alg:DW}.
In particular, for any $\varrho \approx (\sqrt{d} \alpha^p)^{1/3}$, we can show that it takes no longer than $O(1/\varrho^2)$ steps to return an estimator whose gradient norm is not larger than $\varrho$.

\begin{theorem}\label{thm:alg3}
Suppose $\alpha \in (0,1)$ and step size $\gamma_t \equiv \gamma \leq 1/(5L)$ in Algorithm \ref{alg:DW}, where $L$ is the Lipschitz constant of $F(w) := \mathbb E[f_w(X)]$.
For any $\varrho$ satisfying $\varrho \geq \tilde \varrho := \Big(36 C_{\phi''} \alpha^{p} \sqrt{d} (\log n) \frac{\hat \mu_{f_{w^{(0)}}}}{\gamma (p-1)}\Big)^{1/3}$, we define $T_{end}(\varrho) := \min \{t: \|\nabla_{w} \hat \mu_{f_{w^{(t)}}}\| \leq \varrho \}$. Then it holds
\[T_{end}(\rho) \leq \frac{\hat \mu_{f_{w^{(0)}}}}{\gamma \varrho^2} \]
with probability going to 1, where $C_{\phi''}$ is a constant depending on the second derivative of influence function $\phi(x)$.
\end{theorem}

\begin{remark}
    Especially, we can take $\alpha = \Theta(n^{-1/p})$.
    Then $\tilde \varrho =
     \Big(36 C_{\phi''} \alpha^{p} \sqrt{d} (\log n) \frac{\hat \mu_{f_{w^{(0)}}}}{\gamma (p-1)}\Big)^{1/3} = O((\sqrt{d} \log n / n)^{1/3}) = o(1)$ once $d = o(n/\log n)^2$.
     Hence the algorithm returns a solution close to a local stationary point with high probability.
\end{remark}

\newpage\clearpage

\section{Numerical Experiments}\label{sec:sim}

In this section, we provide multiple simulation results to show the usefulness and superiority of our proposed empirical risk-based method.
Specifically, we compare the following algorithms and we first define
\[ \phi_{wide}(x) =
\begin{cases}
\log(1 + x + C_p|x|^p) &~ x \geq 0\\
-\log(1 - x + C_p|x|^p) &~ x < 0.\\
\end{cases}
\]

\vspace{0.05in}

\noindent and

\vspace{0.05in}

\[
\phi_{narrow} (x) =
\begin{cases}
\log \big(1 - \frac{p-1}{p}(p C_p)^{-\frac{1}{p-1}} \big)& \text{if}~ x \leq - (p C_p)^{-\frac{1}{p-1}}, \\
\log(1 + x + C_p|x|^p) &  \text{if} ~ - (p C_p)^{-\frac{1}{p-1}} \leq x \leq 0, \\
- \log(1 - x + C_p|x|^p) &  \text{if} ~ 0 < x \leq (p C_p)^{-\frac{1}{p-1}}, \\
- \log \big(1 - \frac{p-1}{p}(p C_p)^{-\frac{1}{p-1}} \big) & \text{if}~ x \geq (p C_p)^{-\frac{1}{p-1}}.
\end{cases}
\]
$\phi_{wide}(x)$ and $\phi_{narrow}(x)$ are widest and narrowest influence functions satisfying \eqref{eq:C_HT}. Please refer to \cite{bhatt2022minimax} for more detailed explanations of ``widest'' and ``narrowest''.

\vspace{0.1in}

\begin{itemize}
	\item[] \textbf{ERM-wide}. Algorithm~\ref{alg:ERM} with choice of $\phi(x) = \phi_{wide}(x)$
	\item[] \textbf{ERM-narrow}.
	Algorithm~\ref{alg:ERM} with choice of $\phi(x) = \phi_{narrow}(x)$.
	\item[] \textbf{Grad-wide}. Algorithm~\ref{alg:GD} with choice of $\phi(x) = \phi_{wide}(x)$
	\item[] \textbf{Grad-narrow}. Algorithm~\ref{alg:GD} with choice of $\phi(x) = \phi_{narrow}(x)$.
	\item[] \textbf{Mean}.  Algorithm~\ref{alg:GD} via replacing step 5 (Robust gradient) by computing $$g^{(t)} = \frac{1}{n} \sum_{i=1}^n \nabla f_{w^{(t)}}(X_i).$$
	\item[] \textbf{Trim}.  Algorithm~\ref{alg:GD} via replacing step 5 (Robust gradient) by computing $$g^{(t)} = \frac{1}{n} \sum_{i=1}^n \text{Trunc}(\nabla f_{w^{(t)}}(X_i), B),$$ where $\text{Trunc}(X, B) = X \mathbf 1\{|X| \leq B\}$.
	(Remark: This "Trim" algorithm is theoretically equivalent to the truncated loss method.)
\end{itemize}

 \vspace{0.2in}

The choices of tuning parameter in Sections~\ref{sec:sim:reg} -~\ref{sec:sim:median} are given:
$\alpha = \alpha_{erm} := v^{-1/p} (p-1)^{-1/p} h^{(p-1)/p} C_p^{-1/p}$
$(\frac{\log(2T_{max}/\delta)}{n})^{1/p}$ for ERM-wide and ERM-narrow methods;
$\alpha = \alpha_{grad} := v^{-1/p} (p-1)^{-1/p} h^{(p-1)/p} C_p^{-1/p} (\frac{\log(2dT_{max}/\delta)}{n})^{1/p}$ for Grad-wide and Grad-narrow methods;
$B = v^{1/p}(\frac{n}{\log(2dT_{max}/ \delta)})^{1/p}$.
$T_{max}$ is the maximum iteration number.
In experiments, we further fix $h = 0.5$, $v = 1$ and $T_{max} = 1000$.

\subsection{Regression}\label{sec:sim:reg}

We first consider a regression problem, where we in particular assume that
$Y_i = X_i^T w_{\ast} + \xi_{i}$, where $\xi_i$ are symmetrized Pareto random variables.
That is, $\xi_i = (2 u_i - 1) \tilde \xi_i$ with
$$\tilde \xi_i \sim_{i.i.d.} F_{pareto}(x) ~ \text{and}~ u_i = \text{Bernoulli}(0.5),$$
$F_{pareto}(x) = 1 - \frac{1}{x^a}$ and $a$ is the shape (tail) parameter.
We further choose dimension $d \in \{2, 4, 8, 16, 32\}$ and set $a \in \{0.5, 1, 2\}$ when $p = 2$ or $a \in \{0.5, 1, 1.5\}$ when $p = 1.5$. For each setting, we replicate for 50 times. The averages of estimation errors ($\|\hat w - w_{\ast}\|_2$'s) are reported in Figure~\ref{fig:regression}.

\begin{figure}[ht!]

\mbox{\hspace{-0.25in}
	\includegraphics[width=2.4in]{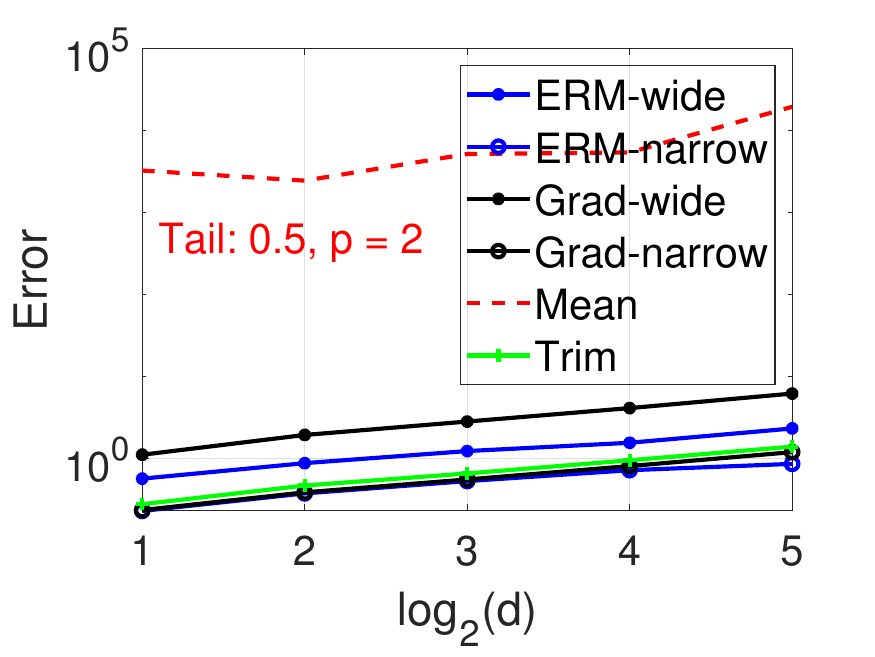}\hspace{-0.15in}
	\includegraphics[width=2.4in]{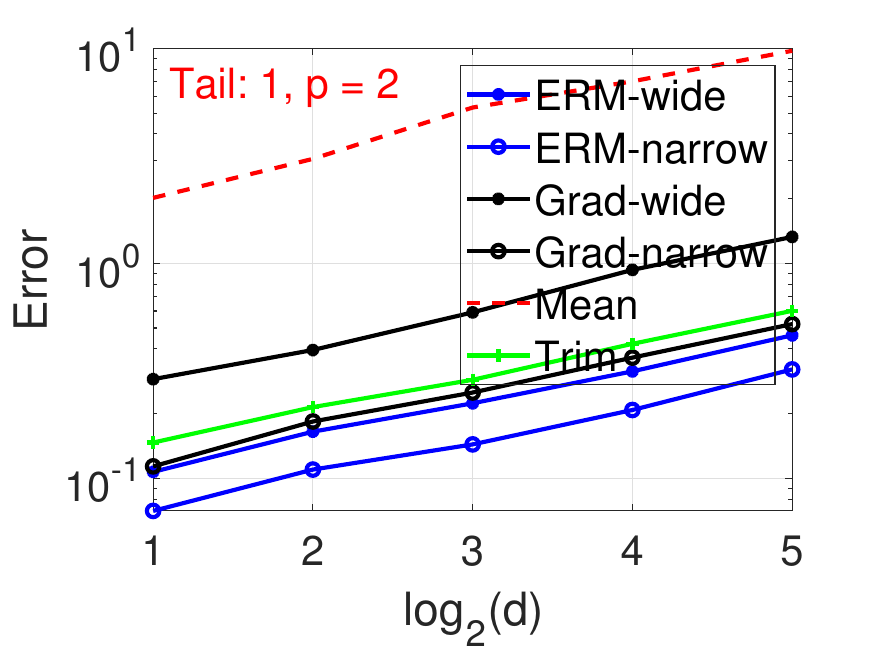}\hspace{-0.15in}
	\includegraphics[width=2.4in]{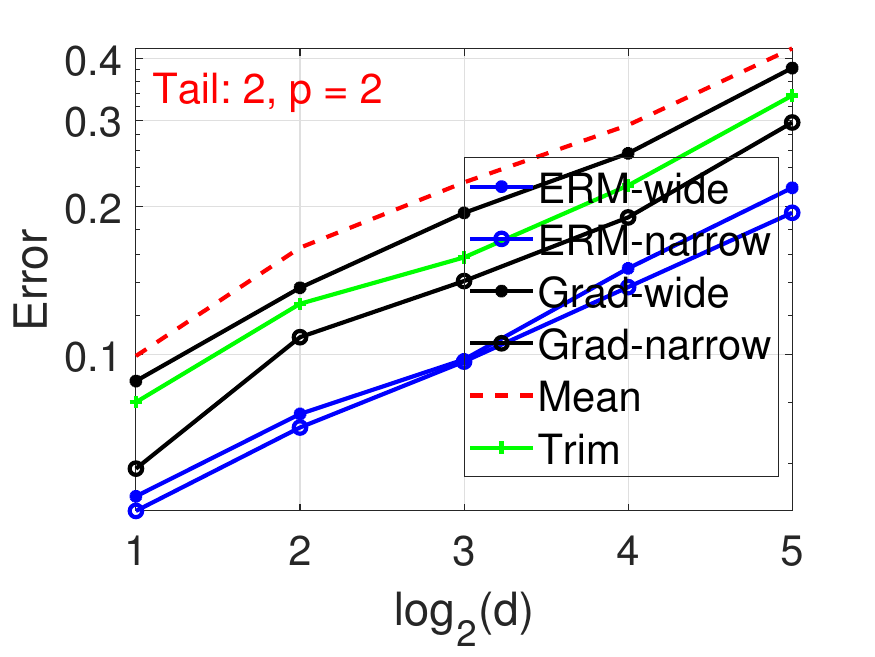}
}

\mbox{\hspace{-0.25in}
	\includegraphics[width=2.4in]{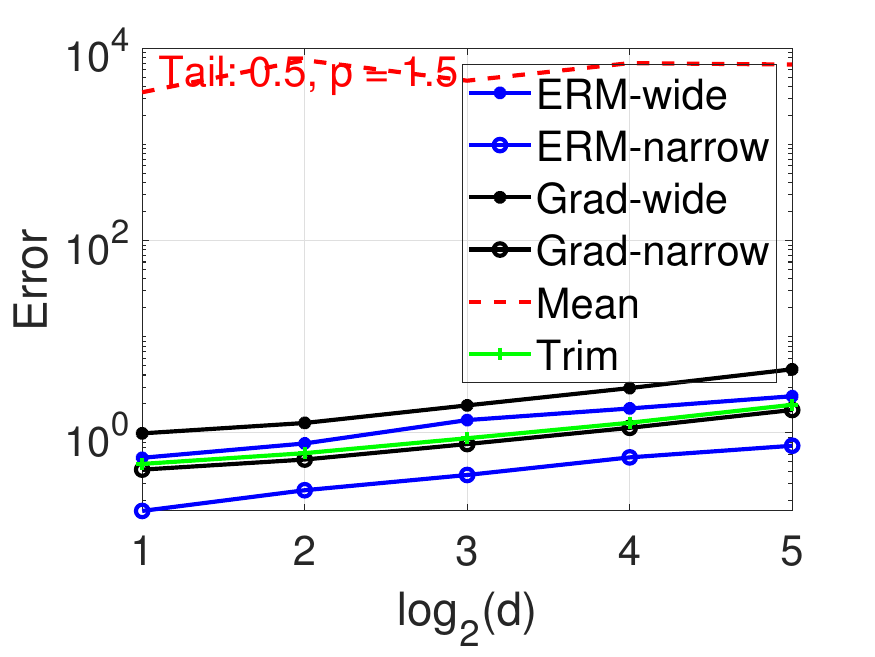}\hspace{-0.15in}
	\includegraphics[width=2.4in]{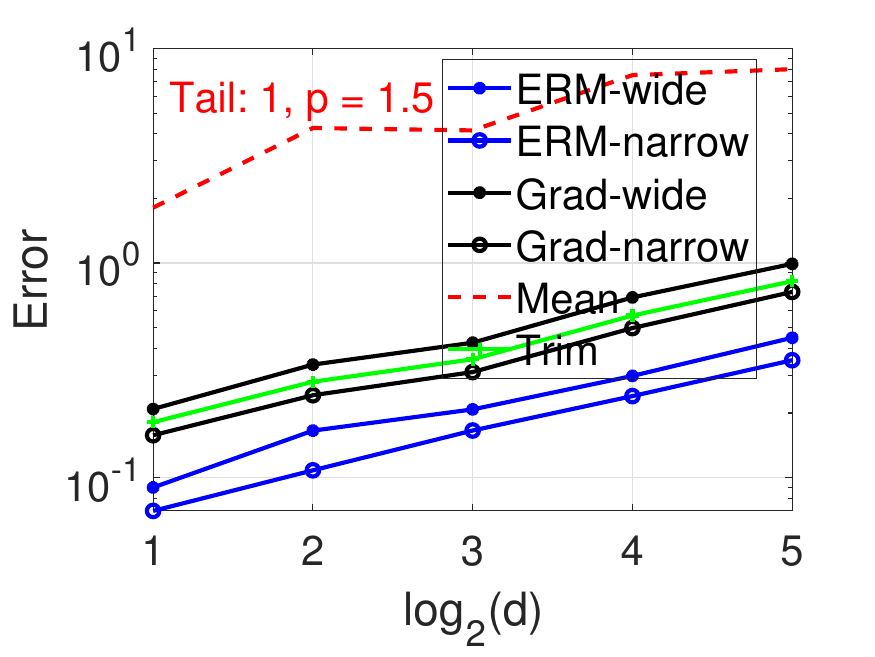}\hspace{-0.15in}
	\includegraphics[width=2.4in]{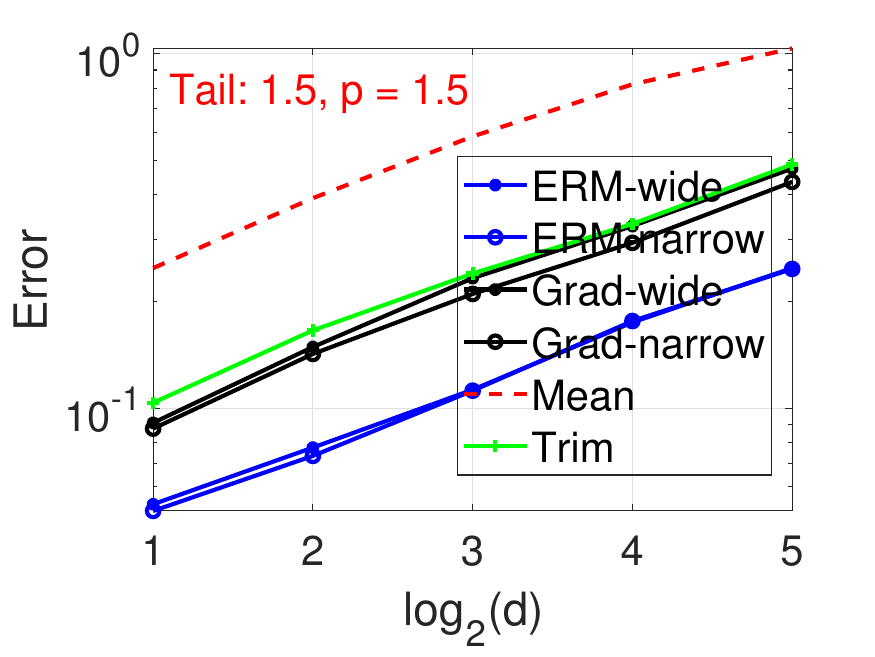}
}

	\caption{Comparison between six methods in regression problems under different dimension values and shape parameters.}\label{fig:regression}	
\end{figure}

\subsection{Regression with Contamination}

We next consider a regression problem with contamination, where we in particular assume that the clean data follows
$Y_i = X_i^T w_{\ast} + \xi_{i}$, where $\xi_i$ are standard normal random variables.
The data are contaminated in the follow fashion.
$\tilde Y_i = Y_i$ with probability $1 - \eta$ and
$\tilde Y_i = (2 u_i - 1) \tilde \xi_i$ with probability $\eta$.
Here $\eta \in (0,1)$ is the contamination rate and $\tilde \xi_i$, $u_i$ are the same as in the previous setting.
In this scenario, we fix $d = 8$ and choose contamination probability $\eta \in \{5\%, 10\%, 20\%, 30\%, 40\% \}$.
The choice of tail parameter $a$ remains the same as in the previous section.
The results of estimation errors are shown in Figure~\ref{fig:contamination}.

\begin{figure}[ht!]

\mbox{\hspace{-0.25in}
	\includegraphics[width=2.4in]{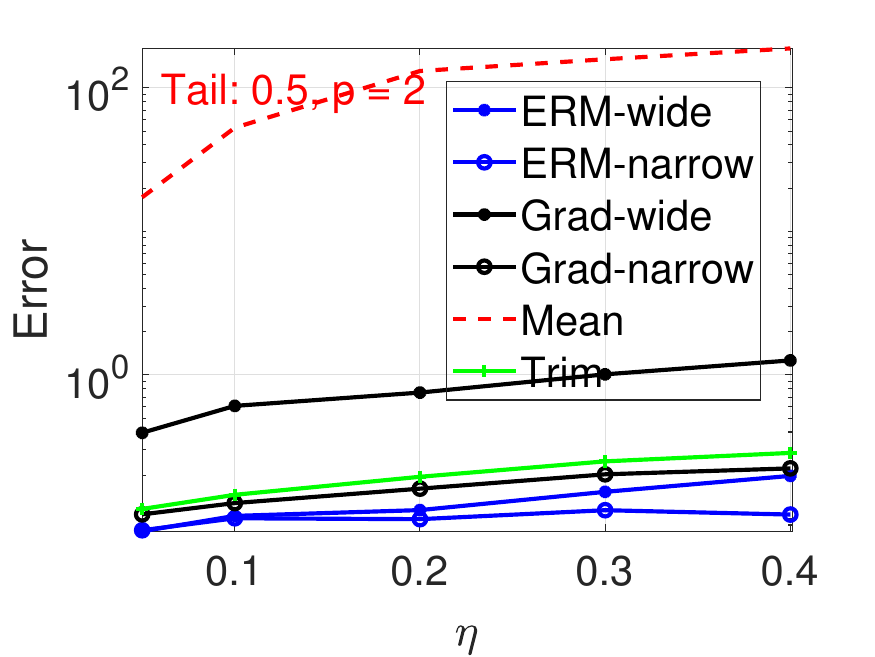}\hspace{-0.15in}
	\includegraphics[width=2.4in]{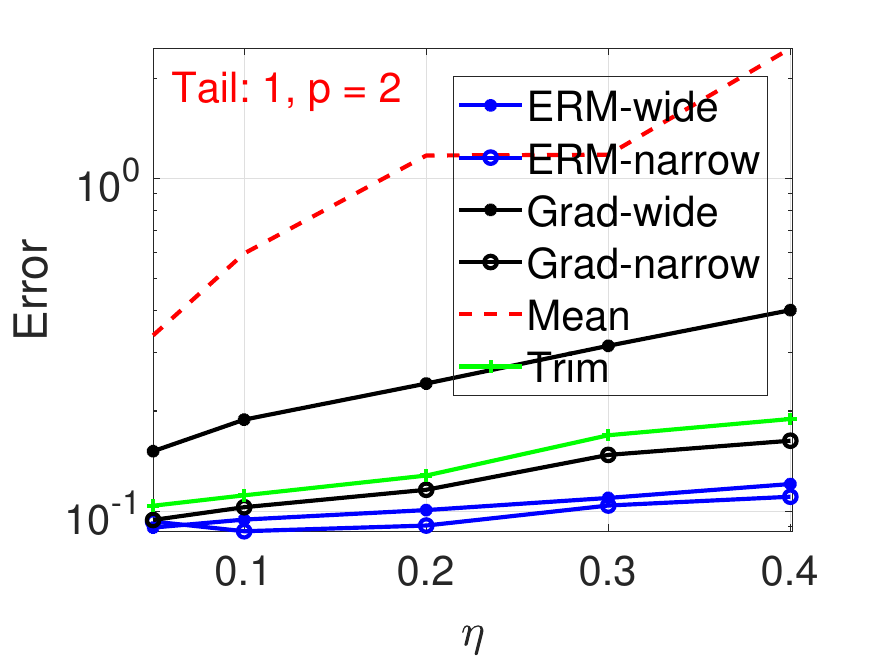}\hspace{-0.15in}
	\includegraphics[width=2.4in]{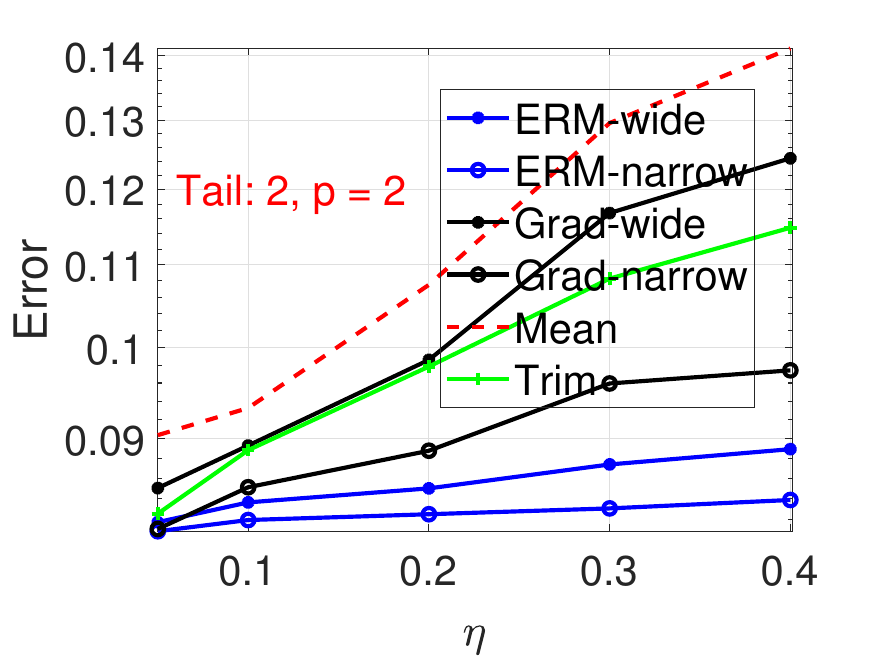}
}

\mbox{\hspace{-0.25in}
	\includegraphics[width=2.4in]{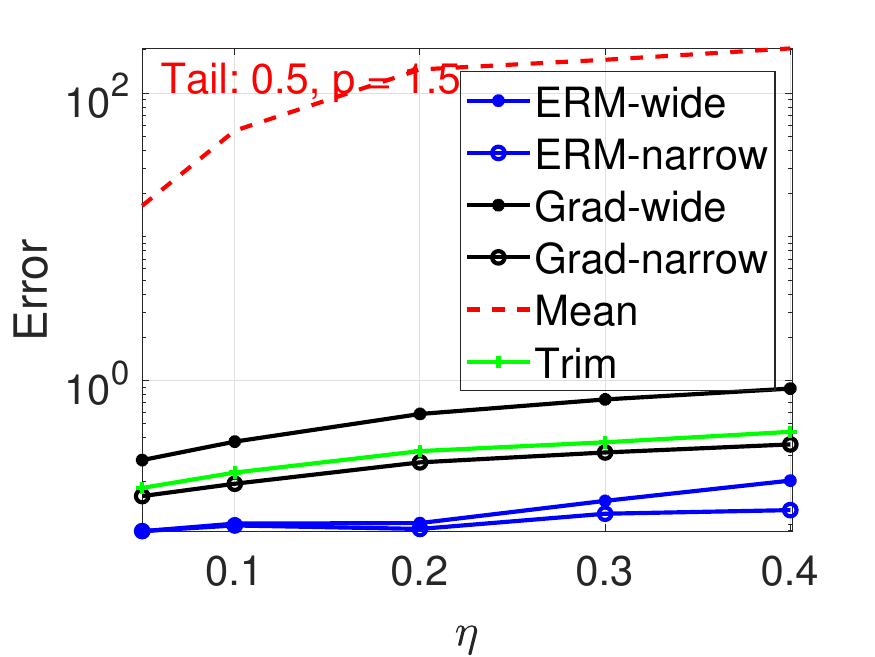}\hspace{-0.15in}
	\includegraphics[width=2.4in]{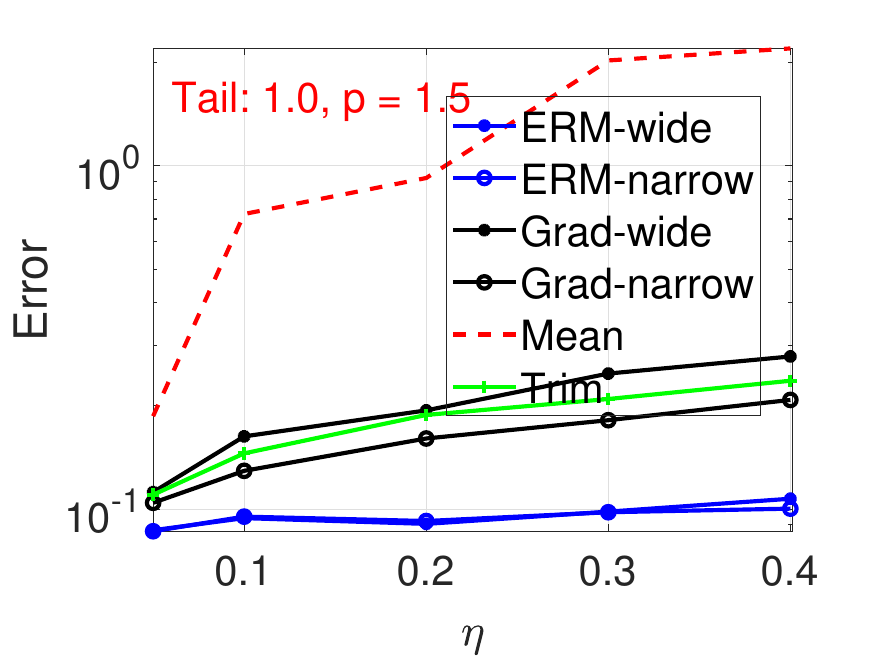}\hspace{-0.15in}
	\includegraphics[width=2.4in]{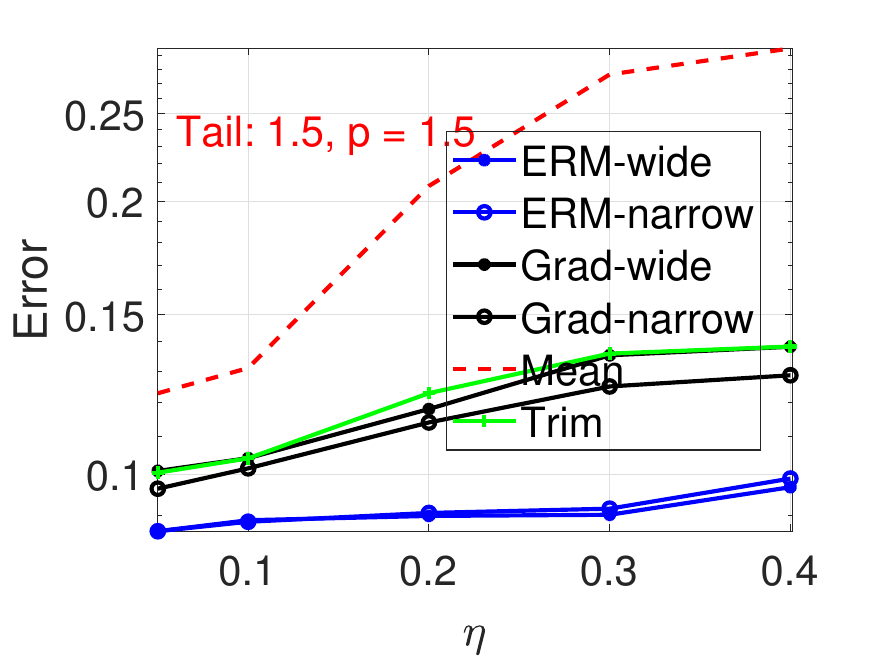}
}

	\caption{Comparison between six methods in regression problems under different contamination rates and shape parameters.}\label{fig:contamination}	
\end{figure}



\newpage

\subsection{$K$-means Clustering}

We next consider a $K$-means clustering problem in this section.
The data generation is described as follows,
\[Y_i = W_{c_i}^{\ast} + \boldsymbol{\xi}_i ~ \text{with}~ c_i \sim \text{Multinom}(1,1,\boldsymbol{\pi}),\]
where $\boldsymbol{\pi} = (\pi_1, \ldots, \pi_K)$ such that $0 < \pi_k < 1$ and $\sum_{k=1}^K \pi_k = 1$. $W^{\ast} = (W_k^{\ast})$ is a $d$ by $k$ matrix.
$\boldsymbol{\xi}_i \in \mathbb R^d$ and each of its coordinate follows a symmetrized Pareto random variables as described before.

For optimization, we postulate the following formulation.
\begin{eqnarray}\label{eq:Kmean}
\min_{W} \sum_i \min_{k \in [K]} l(Y_i, W_k),
\end{eqnarray}
where $W = (W_k)$ is a $d$ by $k$ matrix with $W_k$ being its $k$-th column. We then perform the following estimation scheme for each of six algorithms until the convergence.
\begin{itemize}
	\item For each cluster $k$, update $W_k^{(t+1)} = W_k^{(t)} - \gamma_t g_k^{(t)}$ where $g_k^{(t)}$ is obtained via using ERM-wide (ERM-narrow, Grad-wide, Grad-narrow, Mean or Grad-trim) algorithm.
	\item For each $i$, we assign class label $c_i := \arg\min_k l(Y_i, W_k^{(t+1)})$.
\end{itemize}

\newpage

In this clustering task, we consider three different settings.
(i) We fix $d = 2$, $a = 1$ and let $K \in \{2,3,4,5,6\}$.
(ii) We fix $d = 4$, $K = 2$ and let $a \in \{0.5, 1, 1.5, 2.5, 3.5\}$.
(iii) We fix $d = 2$, $K = 3$ and let $\eta \in \{5\%, 10\%, 20\%, 30\%, 40\%\}$.
The average of estimation errors ($\|\hat W - W^{\ast}\|_2$) are provided in Figure~\ref{fig:Kmean}.

\begin{figure}[ht!]
	\centering
	\includegraphics[width=2.4in]{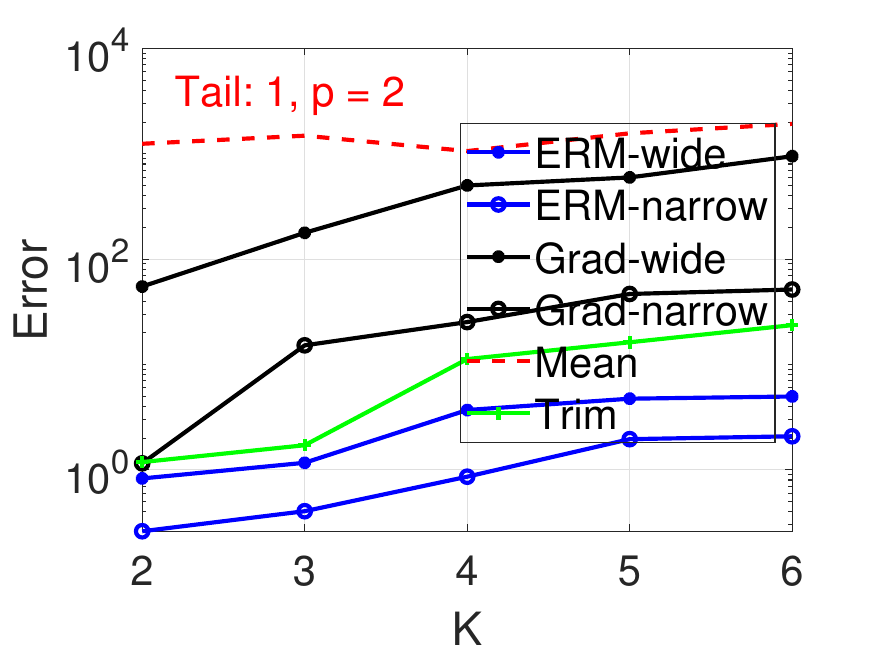}
	\includegraphics[width=2.4in]{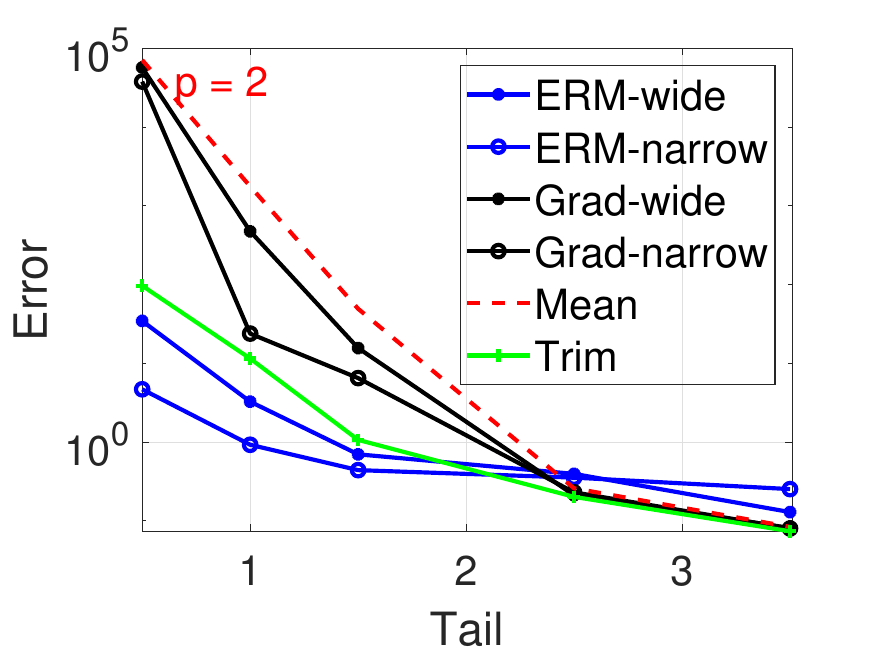}
	\includegraphics[width=2.4in]{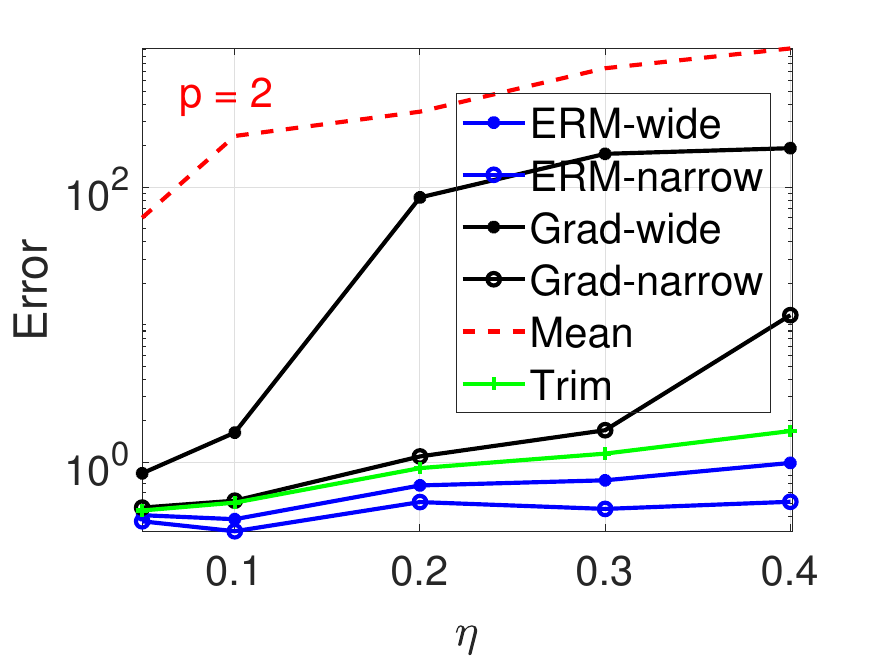}
	\includegraphics[width=2.4in]{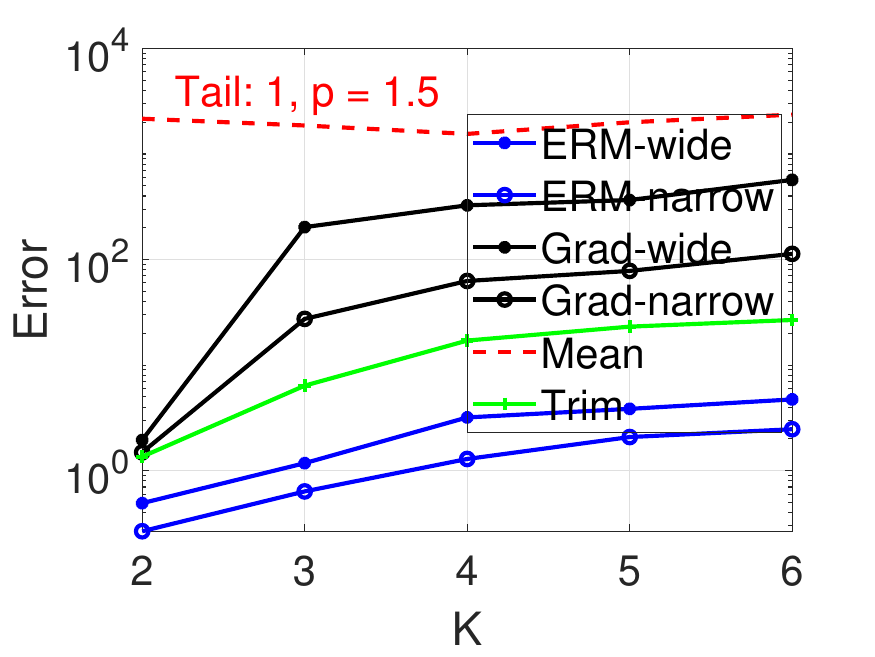}
	\includegraphics[width=2.4in]{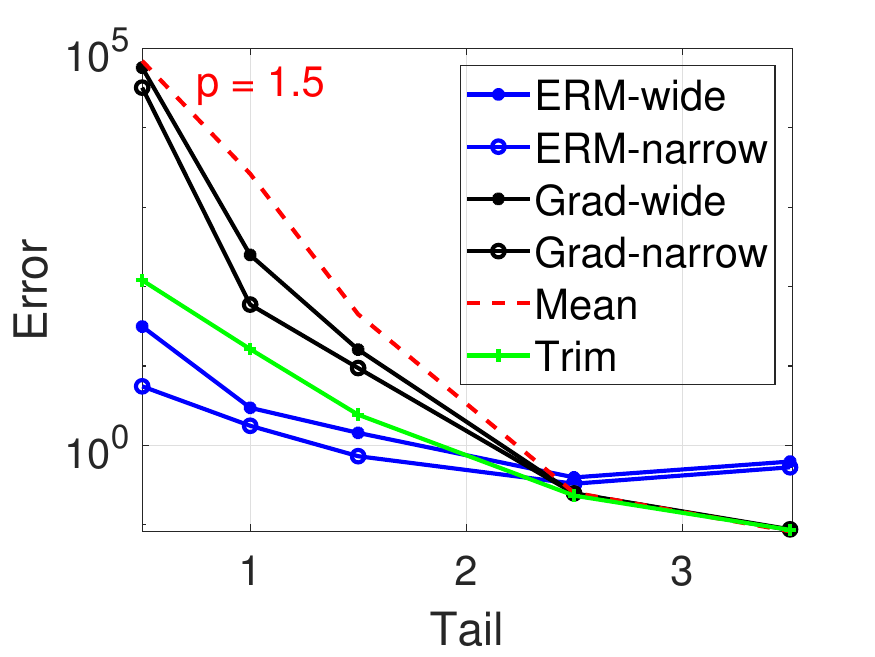}
	\includegraphics[width=2.4in]{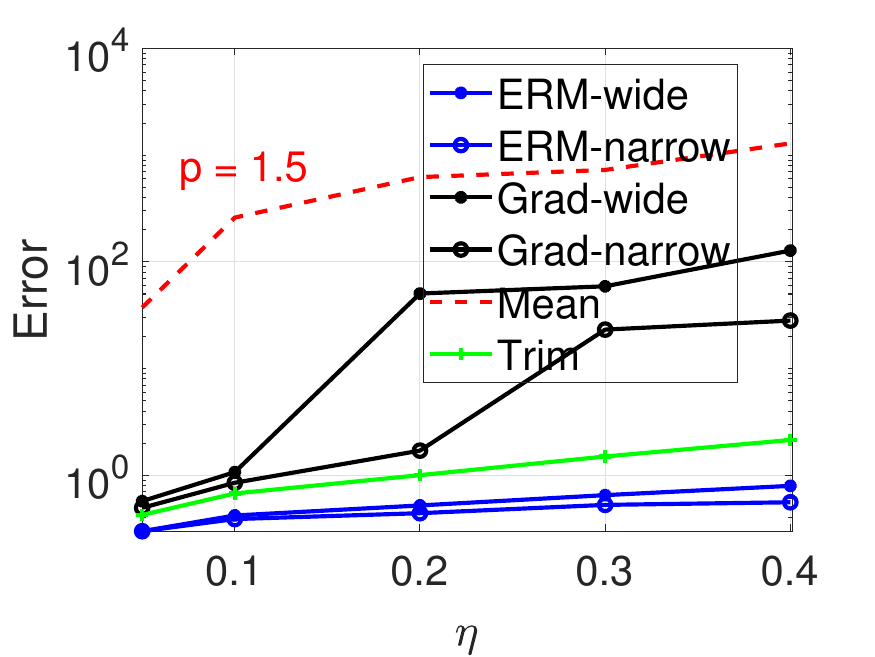}
	\caption{Comparison between six methods in $k$-means clustering problems.}\label{fig:Kmean}	
\end{figure}

\subsection{Comparison with geometric Median}\label{sec:sim:median}

Geometric median~\citep{minsker2015geometric, hsu2016loss} is another type of popular method for estimation problem in heavy-tailed setting.
This approach is the generalization of ``Median of mean'' estimator~\citep{BCL13}.
For a regression problem, the main framework of geometric median can be described as follows.

\newpage

\begin{enumerate}
	\item We divide data to $M$ subsets.
	\item Compute estimators $w_1, \ldots, w_M$ from $M$ subsets.
	\item For each $i \in [M]$, we compute the distance between $d_{ij} = \|w_i - w_j\|$  ($j \neq i, j \in [M]$).
	Compute $r_i$ be the median of $d_{ij}$'s.
	\item Compute $i_{\ast} := \arg\min_i r_i$ and
	return $\hat w = w_{i_{\ast}}$.
\end{enumerate}
We then compare our proposed ERM gradient method (ERM-narrow) with such geometric median estimator under the choice of dimension $d \in \{5, 10, 20\}$
and $a \in \{0.5, 0.6, 0.7, 0.8, 0.9, 1, 1.2, 1.6, 2\}$.
The results are given in Figure~\ref{fig:median}.

\begin{figure}[ht!]

\mbox{\hspace{-0.25in}
	\includegraphics[width=2.4in]{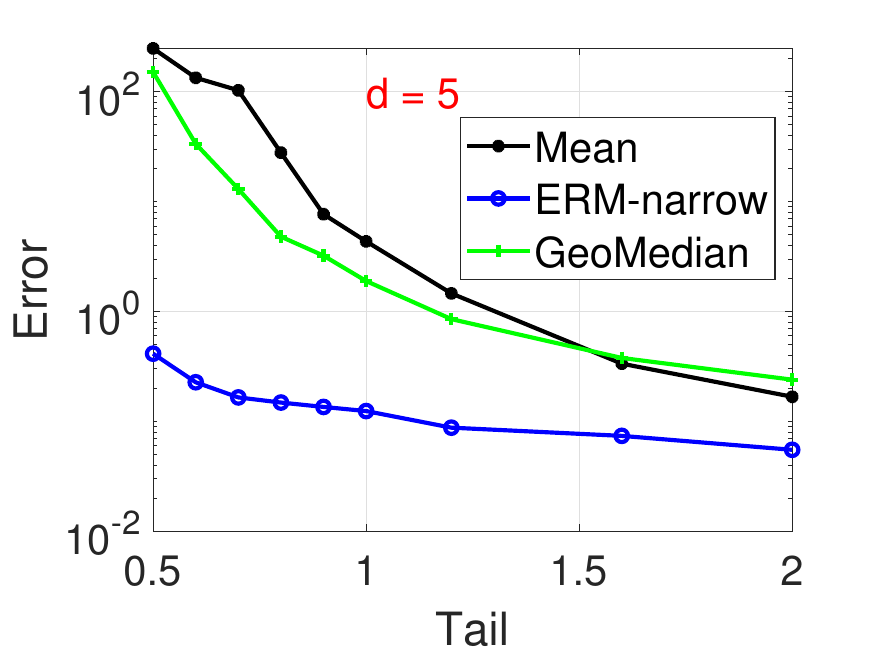}\hspace{-0.15in}
	\includegraphics[width=2.4in]{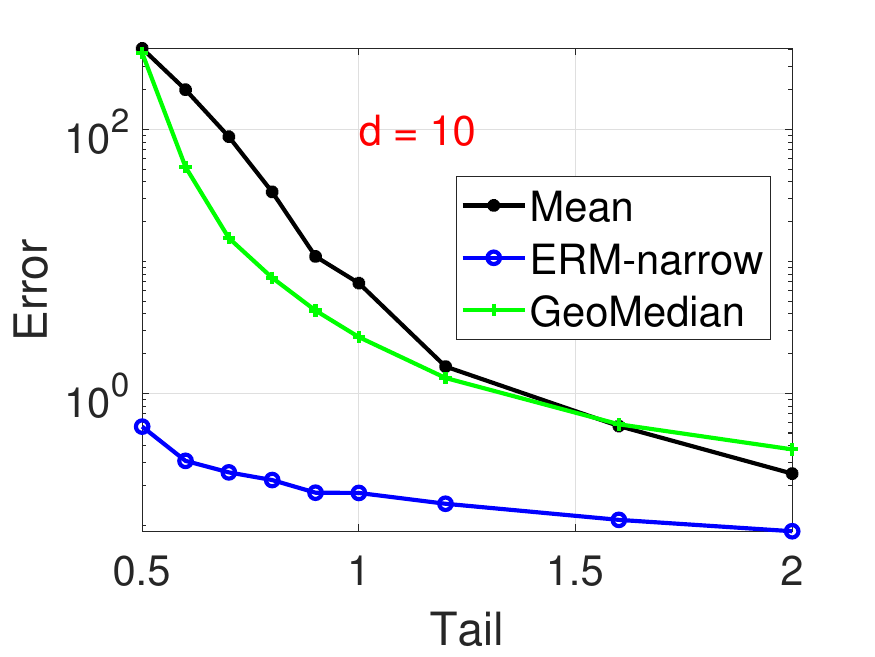}\hspace{-0.15in}
	\includegraphics[width=2.4in]{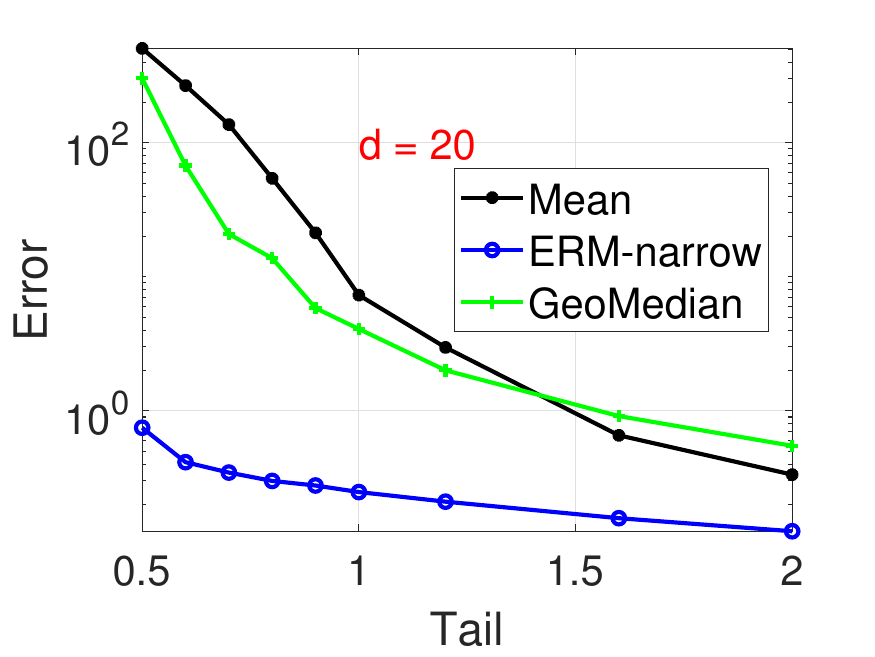}
}
	\caption{Comparison between geometric median methods and proposed ERM gradient method.}\label{fig:median}	
\end{figure}

\subsection{Combining with Deep Learning Framework}

We implement our proposed double-weighted gradient descent (Algorithm~\ref{alg:DW}) in the Pytorch platform and compare that with vanilla gradient descent,
clipped-norm gradient descent
and trimmed gradient descent methods.

\begin{remark}
    In clipped-norm gradient descent method, the update formula is
    $w^{(t+1)} = w^{(t)} - \gamma_t g^{(t)} / \|g^{(t)}\|$,
    where $g^{(t)} = \frac{1}{n} \sum_i \frac{\partial f_{w^{(t)}}(X_i)}{\partial w}$.
    In trimmed gradient descent method,  we compute the trimmed loss of sample $i$, i.e., $f_{trim}(X_i) = \min\{f(X_i), B\}$,
    ($B$ is a tuning parameter to truncate those large loss).
    Then update formula is
    $w^{(t+1)} = w^{(t)} - \gamma_t g_{trim}^{(t)}$,
    where $g_{trim}^{(t)} = \frac{1}{n} \sum_i \frac{\partial f_{trim, w^{(t)}}(X_i)}{\partial w}$.
\end{remark}

The data $y_i = f(x_i) + \epsilon_i$, where $\epsilon_i$'s are symmetrized Pareto random variable with shape parameter taken in $\{1.2, 1.4, 1.6, 1.8, 2.0\}$.
The underlying generating function $f$ is chosen as
$f(x) = \sin(x)$ [called as "Sin"],
$f(x) = \sin(x) \exp\{x\}$ [called as "Sin\_exp"],
and $f(x) = x^2 \cos(x)$ [called as "Cos\_x2"], respectively.
The considered neural network $\hat f$ is chosen to be two-layer ReLU network with 128 hidden units.

In the training stage, we choose 2000 points $x_i$'s from $[-\pi, \pi]$.
In the testing stage, we randomly sample 100 points $x_j$'s from $\text{Normal}(0, 1)$.
The average prediction error is reported, i.e., $\frac{1}{100} \sum_{j=1}^{100}|\hat f(x_j) - f(x_j)|$.
For the choice of $\alpha$, we set it to be the inverse of 95 \% quantile of $f_{w^{(0)}}(X_i)$.
For $B$ in trimmed gradient descent, we set it to $B = 1 / \alpha$.
Each case is replicated for 50 times and the corresponding result is given in Figure \ref{fig:torch}.

\begin{figure}[ht!]
\mbox{\hspace{-0.25in}
	\includegraphics[width=2.4in]{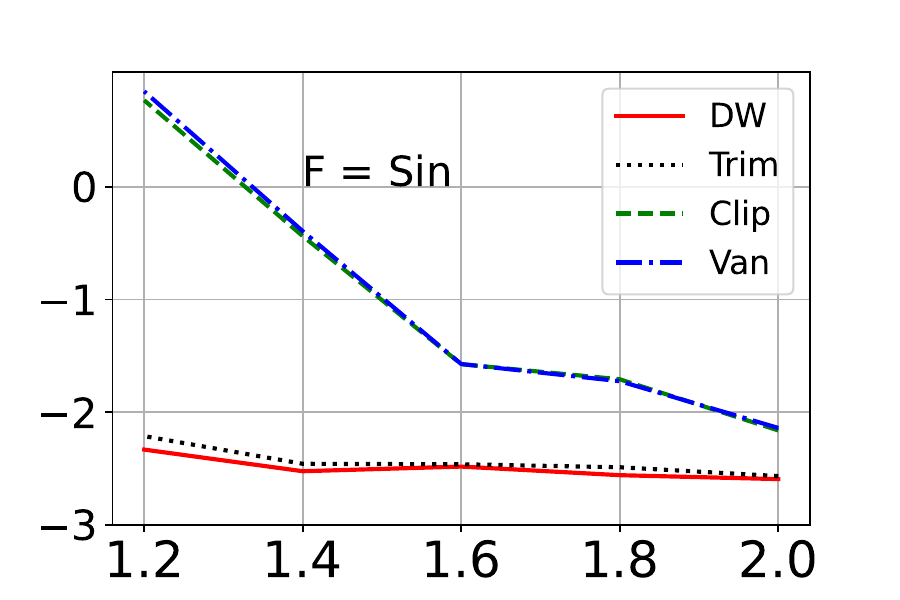}\hspace{-0.15in}
	\includegraphics[width=2.4in]{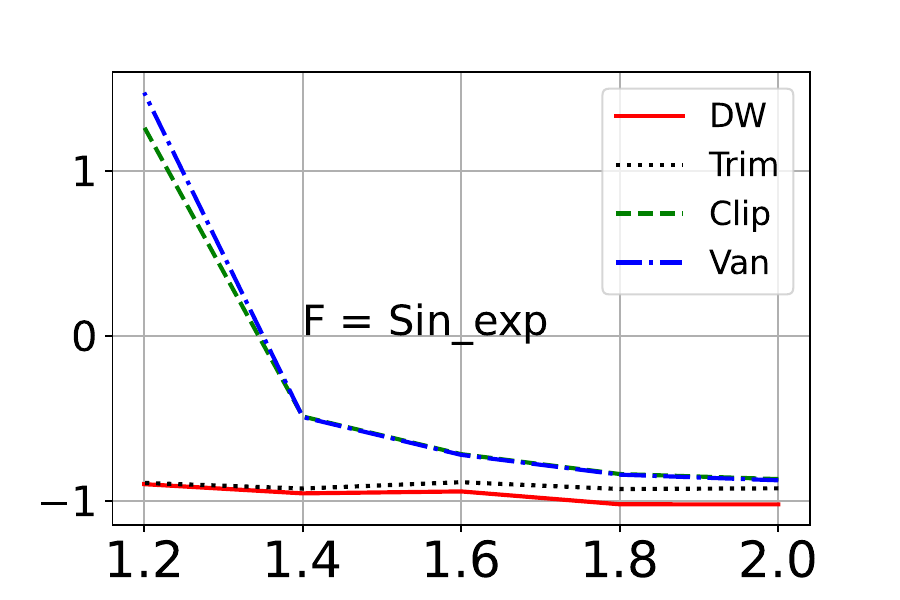}\hspace{-0.15in}
	\includegraphics[width=2.4in]{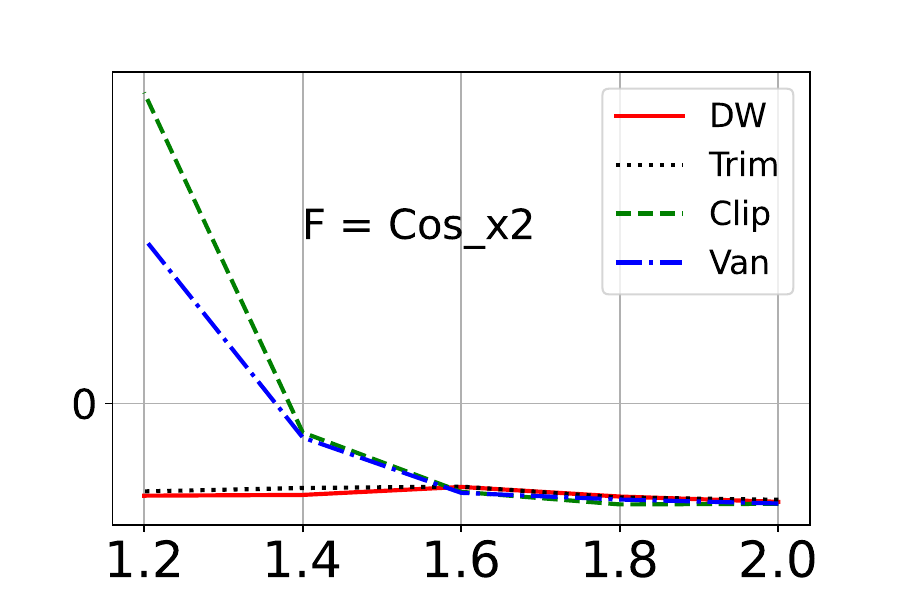}
}
 \caption{Prediction Error under Different Underlying Functions. The prediction error is reported under log-scale. (Lower is better.)}\label{fig:torch}	
\end{figure}

\newpage

Lastly, we report the relative computational time (i.e., $\frac{t_{dw}}{t_{van}}$, where $t_{dw}$ is the time for our double-weighted method training on single data set
and $t_{van}$ is similarly defined for vanilla gradient descent) in Table \ref{tab:time}.
\begin{table}[ht!]
    \centering
    \begin{tabular}{c|c|c|c}
    \hline
    \hline
         & "Sin" & "Sin\_exp" & "Cos\_x2" \\
        \hline
     mean    & 1.343 & 1.307 & 1.281\\
     \hline
     std    & 0.1528 & 0.1769 & 0.1048 \\
     \hline
     \hline
    \end{tabular}
    \caption{Table for relative computational time between proposed double-weighted algorithm and vanilla gradient descent method.}
    \label{tab:time}
\end{table}

\subsection{Findings}
From Figures~\ref{fig:regression} -~\ref{fig:Kmean}, we can see that the proposed ERM gradient outperforms other baselines when tails of data distributions are heavy.
This suggests that finding the optimizers by minimizing the risk values indeed improve the performances.
Truncation loss (trim) method has similar performance as robust gradient descent method. It indicates that estimating robust gradient coordinate-wisely is equivalent to truncating data in practice.
The choice of influence function $\phi$ matters the final outcomes. Choosing narrow function $\phi_{narrow}$ will lead to a more robust estimator.
Based on Figure~\ref{fig:median}, we can find that geometric median method is sub-optimal compared with the proposed method.
Although median-type method also has reasonable theoretical guarantees, it is not a satisfactory algorithm in practice.
Lastly, from Figure~\ref{fig:torch}, we can see that our proposed double weighted gradient descent method can be well embedded into Pytorch or other deep learning framework.
It can achieve the lowest prediction error via using multi-layer ReLU network approximation.
Additionally, from Table \ref{tab:time}, it reveals that the computational time of the proposed new algorithm is also comparable to the classical gradient descent method.

\newpage

\section{Conclusions}\label{sec:conclusion}

In this paper, we consider empirical risk minimization problem with heavy-tailed data, which is an important area in learning theory.
We assume a weaker moment condition that data does not have finite variance, but only has $p$-th moment with $1 < p < 2$.
In contrast to using truncation-based method, we directly work on estimating the excess risk values, where we adopt Catoni's method for robust estimation. The final optimizer is returned via minimizing the estimated excess risk.
For such optimizer, we establish the excess risk bounds by studying the properties of Catoni's influence functions and using the generic chaining techniques.
On the other hand, we propose an empirical risk-type gradient algorithms to address the computational challenges.
The proposed algorithm gives a computationally friendly way to compute the robust gradient and also leads to better performance in terms of estimation errors.
Compared with other competing baselines, our method is shown to be more robust under different types of outliers and data contamination.
Our findings unveil that estimator based on minimizing risk values can be practically better than common truncation methods.
Our method is also shown to be easy-to-implement in Pytorch platform.
Therefore it might be interesting and promising to study the proposed methodology in different large deep learning models as the future work.

\bibliographystyle{plainnat}
\bibliography{reference}


\clearpage

\appendix

\section{Proof of Results in Section \ref{sec:catoni}} \label{sec:proof2}

\begin{proof}[Proof of Lemma \ref{lem:nec.suf}]
	A necessary and sufficient condition for the existence of a function satisfying~\eqref{eq:C_HT} is given by
	\[
	(1 - x + C_{p} x^{1+\varepsilon}) (1+x+C_{p} x^{1+ \varepsilon}) \geq 1,~\forall~x\geq 0.
	\]
	Rearranging, this reduces to
	\[
	2 C_{p} x^{1+\varepsilon} + C^2_{p} x^{2(1+\varepsilon)} \geq x^2,~\forall~x\geq 0,
	\]
	which is equivalent to the condition
	\begin{equation} \label{eq:const}
	C^2_{p} x^{2\varepsilon} + 2 C_{p} x^{\varepsilon-1} > 1,~\forall~x>0.
	\end{equation}
	The minimum of the expression in the left hand side over~$x>0$ is achieved at
	\[
	x_{*} = \Big( \frac{1 - \varepsilon}{C_{p} \varepsilon} \Big)^{\frac{1}{1+\varepsilon}},
	\]
	and substituting this value in~\eqref{eq:const} and solving for~$C_{\varepsilon}$ produces the desired result.
\end{proof}

\begin{proof}[Proof of Theorem \ref{thm:Nprob_ineq}]
	As in~\citet{Cat12}, define
	\[
	r_n(\theta) = \sum_{i=1}^n \psi\Big( \alpha (X_i - \theta) \Big),
	\]
	and note that~$r_n(\theta)$ is non-increasing in~$\theta \in \mathbb{R}$. Using the upper bound on the influence function in~\eqref{eq:C_HT},
	\begin{align*}
	\mathbb{E}\Big[ \exp(r_n(\theta))  \Big] &= \Big( \mathbb{E} \Big[ \exp\Big(\psi\Big( \alpha (X_1 - \theta) \Big) \Big] \Big)^n\\
	& \leq \Big( \mathbb{E} \Big[ 1 + \alpha (X_1 -\theta) + C_{p} \alpha^{1+\varepsilon} |X_1 - \theta|^{1+\varepsilon} \Big] \Big)^n \\
	&= \Big(1 + \alpha (\mu -\theta) + C_{p} \alpha^{1+\varepsilon} \mathbb{E}|X_1 - \theta|^{1+\varepsilon} \Big)^n.
	\end{align*}
	We will use a convexity upper bound as follows. For any $a,b \geq 0$ and $0 < h < 1$,
	\begin{align} \label{eq:Con.Bd}
	(a+b)^{1+\varepsilon} &= \Big(h \frac{a}{h} + (1-h) \frac{b}{1-h} \Big)^{1+\varepsilon}, \nonumber \\
	&\leq h\Big(\frac{a}{h}\Big)^{1+\varepsilon} + (1-h)\Big(\frac{b}{1-h}\Big)^{1+\varepsilon} = \frac{a^{1+\varepsilon}}{h^{\varepsilon}} + \frac{b^{1+\varepsilon}}{(1-h)^{\varepsilon}}. \tag{CB}
	\end{align}
	Therefore, for any~$0 < h < 1$,
	\begin{equation} \label{eq:Prc}
	\mathbb{E}|X_1 - \theta|^{1+\varepsilon} \leq h^{-\varepsilon} \mathbb{E}|X_1 - \mu|^{1+\varepsilon} + (1-h)^{-\varepsilon} |\mu - \theta|^{1+\varepsilon}.
	\end{equation}
	This leads to worse constants  than in~\citet{Cat12}, and is the price to pay for the generalization. Using the above bound, we obtain
	\begin{align*}
	\mathbb{E}\Big[ \exp(r_n(\theta)) \Big] &\leq \Big(1 + \alpha (\mu -\theta) + h^{-\varepsilon} C_{p} \alpha^{1+\varepsilon} v + C_{p} \alpha^{1+\varepsilon} (1-h)^{-\varepsilon}|\mu - \theta|^{1+\varepsilon} \Big)^n \\
	&\leq \exp \Big( \alpha n(\mu - \theta) + n h^{-\varepsilon} C_{p} \alpha^{1+\varepsilon} v + n C_{p} \alpha^{1+\varepsilon} (1-h)^{-\varepsilon}|\mu - \theta|^{1+\varepsilon} \Big).
	\end{align*}
	Similarly, using the lower bound on the influence function in~\eqref{eq:C_HT}, we obtain by symmetric arguments
	\[
	\mathbb{E}\Big[ \exp(-r_n(\theta)) \Big] \leq \exp \Big( - \alpha n(\mu - \theta)  + n h^{-\varepsilon} C_{p} \alpha^{1+\varepsilon} v + n C_{p} \alpha^{1+\varepsilon} (1-h)^{-\varepsilon}|\mu - \theta|^{1+\varepsilon} \Big).
	\]
	Let~$\delta \in (0,1)$. As in~\citet{Cat12}, we define
	\begin{align*}
	B_{+}(\theta) &= (\mu - \theta) +  h^{-\varepsilon} C_{p} \alpha^{\varepsilon} v +  C_{p} \alpha^{\varepsilon} (1-h)^{-\varepsilon}|\mu - \theta|^{1+\varepsilon} + \frac{\log(2/\delta)}{\alpha n}, \\
	B_{-}(\theta) &= (\mu - \theta) - h^{-\varepsilon} C_{p} \alpha^{\varepsilon} v -  C_{p} \alpha^{\varepsilon} (1-h)^{-\varepsilon}|\mu - \theta|^{1+\varepsilon} - \frac{\log(2/\delta)}{\alpha n}.
	\end{align*}
	By the exponential Markov inequality, we have
	\begin{align} \label{eq:Probs}
	\mathbb{P} \Big\{ r_n(\theta) \geq n\alpha B_{+}(\theta)\Big\} &\leq \frac{\mathbb{E}\Big[ \exp(r_n(\theta)) \Big]}{\exp(\alpha n B_{+}(\theta))} \leq \delta/2, \nonumber \\
	\mathbb{P} \Big\{ r_n(\theta) \leq n\alpha B_{-}(\theta)\Big\} &\leq \frac{\mathbb{E}\Big[ \exp(-r_n(\theta)) \Big]}{\exp(-\alpha n B_{-}(\theta))} \leq \delta/2.
	\end{align}
	Note that the function~$B_{+}$ is a strictly convex function of~$\theta$ and $B_{+}(\theta) \rightarrow \infty$ as $|\theta| \rightarrow \infty$. Therefore,~$B_{+}$ has a unique minimum on~$\mathbb{R}$, which is achieved at
	\[
	\theta_{*} = \mu + \frac{1-h}{\alpha} \Big( \frac{1}{(1+\varepsilon) C_{p}} \Big)^{\frac{1}{\varepsilon}},
	\]
	so that
	\[
	\min_{\theta \in \mathbb{R}} B_{+}(\theta) = B_{+}(\theta_{*}) = h^{-\varepsilon} \alpha^{\varepsilon} C_{p} v - \frac{\varepsilon}{1+\varepsilon} \frac{1-h}{\alpha} \Big( \frac{1}{(1+\varepsilon) C_{p}} \Big)^{\frac{1}{\varepsilon}} + \frac{\log(2/\delta)}{\alpha n}.
	\]
	Suppose that this minimum is non-positive, i.e.
	\begin{equation}  \label{eq:Np_eq2}
	h^{-\varepsilon} \alpha^{1+\varepsilon} C_{p} v + \frac{\log(2/\delta)}{ n} \leq \frac{\varepsilon}{1+\varepsilon} (1-h) \Big( \frac{1}{(1+\varepsilon) C_{p}} \Big)^{\frac{1}{\varepsilon}}.
	\end{equation}
	Then the equation
	\begin{equation*}
	B_{+}(\theta) = 0
	\end{equation*}
	has a real root, and, if the inequality is strict, it has two real roots. Since~$B_{+}(\mu)>0$ and $\theta_{*} > \mu$, the roots are larger than~$\mu$.
	Letting~$z = \mu - \theta$, the equation
	\[
	\widehat{B}_{+}(z) = z +  h^{-\varepsilon} C_{p} \alpha^{\varepsilon} v +  C_{p} \alpha^{\varepsilon} (1-h)^{-\varepsilon}|z|^{1+\varepsilon} + \frac{\log(2/\delta)}{\alpha n},
	\]
	has a root~$z_{+}= \mu - \theta_{+}(\alpha)$ in~$[-1,0)$ using~\eqref{require:n:alpha3}.
	This is because $\widehat{B}_{+}(0) > 0$ and
	$\widehat{B}_{+}(-1) = h^{-\varepsilon} C_p \alpha^{\varepsilon} v + C_p \alpha^{\varepsilon}(1 - h)^{-\varepsilon} + \frac{\log(2/\delta)}{\alpha n} - 1 < 0$.
	Additionally, since~$|z|^{1+\varepsilon} > |z|$, we have that
	\[
	\widehat{B}_{+}(z) \geq z +  h^{-\varepsilon} C_{p} \alpha^{\varepsilon} v +  C_{p} \alpha^{\varepsilon} (1-h)^{-\varepsilon}|z| + \frac{\log(2/\delta)}{\alpha n}.
	\]
	Further on~$[-1,0)$, we have that
	\begin{equation}
	\widehat{B}_{+}(z) \geq z +  h^{-\varepsilon} C_{p} \alpha^{\varepsilon} v -  C_{p} \alpha^{\varepsilon} (1-h)^{-\varepsilon}z + \frac{\log(2/\delta)}{\alpha n}.
	\end{equation}
	We have~$\widehat{B}_{+}(z) \geq 0$ for
	\[
	\mu - \theta_{+}(\alpha) := z_{+} \geq -\frac{h^{-\varepsilon} C_{p} \alpha^{\varepsilon} v + \frac{\log(2/\delta)}{\alpha n}}{1 - C_{p} \alpha^{\varepsilon} (1-h)^{-\varepsilon}}
	\]
	Further using~\eqref{require:n:alpha1}, we have
	\[
	\mu - \theta_{+}(\alpha) := z_{+} \geq - 2 (h^{-\varepsilon} C_{p} \alpha^{\varepsilon} v + \frac{\log(2/\delta)}{\alpha n}).
	\]
	By the monotonicity of the root, we know $\widetilde{\mu}_{c} \leq \theta_{+}(\alpha)$.
	Thus, it holds
	\[
	\mu - \widetilde{\mu}_{c} := z_{+} \geq - 2 (h^{-\varepsilon} C_{p} \alpha^{\varepsilon} v + \frac{\log(2/\delta)}{\alpha n}).
	\]
   Symmetric arguments establish the bounds in other direction.
   Finally, by the special choice of $\alpha$ as given in \eqref{eq:main_alpha}, it is straightforward to compute that
   $|\widetilde{\mu}_{c} - \mu_c| \leq 4 (C_{p} v)^{\frac{1}{1+\varepsilon}} h^{\frac{-\varepsilon}{1+\varepsilon}} \Big( \frac{\log(2/\delta)}{n} \Big)^{\frac{\varepsilon}{1+\varepsilon}}$.
\end{proof}

\section{Proof of Results in Section \ref{sec:mf}}

\begin{proof}[Proof of Theorem~\ref{thm:discrete}]
	According to Theorem~\ref{thm:Nprob_ineq}, we know
	\begin{eqnarray}
	\mathbb P \bigg(|\hat \mu_f - m_f| \geq 2 (h^{-\varepsilon} C_p \alpha^{\varepsilon} v +
	\frac{\log(2 |\mathcal F| / \delta)}{\alpha n}) \bigg) \leq \frac{\delta}{|\mathcal F|}
	\end{eqnarray}
	for any fixed $f$.
	Therefore, we have
	\begin{eqnarray}
	& & \mathbb P \bigg(\sup_{f \in \mathcal F} |\hat \mu_f - m_f| \geq 2 (h^{-\varepsilon} C_p \alpha^{\varepsilon} v +
	\frac{\log(2 |\mathcal F| / \delta)}{\alpha n}) \bigg) \nonumber \\
	&\leq& \sum_{f \in \mathcal F} \mathbb P \bigg( |\hat \mu_f - m_f| \geq 2 (h^{-\varepsilon} C_p \alpha^{\varepsilon} v +
	\frac{\log(2 |\mathcal F| / \delta)}{\alpha n}) \bigg) \nonumber \\
	&\leq& |\mathcal F| \frac{\delta}{|\mathcal F|} = \delta.
	\end{eqnarray}
	Finally, by \eqref{ass:fast}, we arrive at that
	\[m_{\hat f} - m^{\ast} \leq 4 (h^{-\varepsilon} C_p \alpha^{\varepsilon} v +
	\frac{\log(2 |\mathcal F| / \delta)}{\alpha n})\]
	holds with probability at least $1 - \delta$.
\end{proof}

\begin{proof}[Proof of Lemma \ref{lem:key:obs1}]
To prove Lemma \ref{lem:key:obs1}, we need to first introduce the following Lemma \ref{lem:B} - Lemma \ref{lem:r}.

\begin{lemma}\label{lem:B}
	For any fixed $f \in \mathcal F$ and $\mu \in \mathbb R$,
	it holds
	\begin{eqnarray}
	B_f^{-}(\mu,0) \leq \bar r_f(\mu) \leq B_f^{+}(\mu,0),
	\end{eqnarray}
	and, therefore, $m_f - 2 h^{-\epsilon} C_p \alpha^{\epsilon} v \leq \bar \mu_f \leq m_f + 2 h^{-\epsilon} C_p \alpha^{\epsilon} v$. In particular,
	\[B_{\hat f}^{-}(\mu,0) \leq \bar r_{\hat f}(\mu) \leq B_{\hat f}^{+}(\mu,0).\]
	For any $\mu$ and $\eta$ such that $\bar r_{\hat f}(\mu) < \eta$, if $\eta$-condition also holds with
	$\eta$ being given in \eqref{eq:eta}, then
	\begin{eqnarray}
	m_{\hat f} \leq \mu + 2 h^{-\epsilon} C_p \alpha^{\epsilon} v + 2 \eta.
	\end{eqnarray}
\end{lemma}

\begin{lemma}\label{lem:lip}
	Consider influence function $\phi(x) = - \textrm{sign}(x) \log(1 + |x| + C_p |x|^p)$.
	Then it is a Lipschitz function with a Lipschitz constant $L_p$ not exceeding $ \max\{(1+pC_p),p\}$. Furthermore, for any $0<\gamma<1$ the function  $\phi$ is globally H\"older with exponent $\gamma$.
\end{lemma}

\begin{lemma}\label{lem:r}
	Let $\mu_0 = m_{f^{\ast}} + A_{\alpha}(\delta)$. Then on the event,
	\[\Omega_{f^{\ast}}(\delta) := \{ \omega: |\hat \mu_{f^{\ast}} - m_{f^{\ast}}| \leq A_{\alpha}(\delta) \},\]
	the following inequalities hold:
	\begin{itemize}
		\item[i.]  $\hat r_{\hat f}(\mu_0) \leq 0$; ~
		ii. $\bar r_{f^{\ast}}(\mu_0) \leq 0$; ~
		iii. $-\hat r_{f^{\ast}}(\mu_0) \leq 2 L_p A_{\alpha}(\delta)$.
	\end{itemize}
\end{lemma}

Combining above lemmas, we can see that, with probability at least  $1 - 2\delta$, it holds

\begin{eqnarray}
\bar r_{\hat f}(\mu_0) &\leq& \hat r_{\hat f}(\mu_0) + \bar r_{f^{\ast}}(\mu_0) - \hat r_{f^{\ast}}(\mu_0) +  |\bar r_{\hat f}(\mu_0) - \hat r_{\hat f}(\mu_0) - \bar r_{f^{\ast}}(\mu_0) + \hat r_{f^{\ast}}(\mu_0)| \nonumber \\
&\leq& \hat r_{\hat f}(\mu_0) + \bar r_{f^{\ast}}(\mu_0) - \hat r_{f^{\ast}}(\mu_0) \nonumber + \sup_{f \in \mathcal F}|\bar r_{\hat f}(\mu_0) - \hat r_{\hat f}(\mu_0) - \bar r_{f^{\ast}}(\mu_0) + \hat r_{f^{\ast}}(\mu_0)| \nonumber \\
&\leq& \hat r_{\hat f}(\mu_0) + \bar r_{f^{\ast}}(\mu_0) - \hat r_{f^{\ast}}(\mu_0) + Q(\mu_0, \delta) \nonumber \\
&\leq& 0 + 0 + 2 L_p A_{\alpha}(\delta) + Q(\mu_0, \delta) \nonumber \\
&=& 2 L_p A_{\alpha}(\delta) + Q(\mu_0, \delta), \label{eq:combine}
\end{eqnarray}
where the first inequality in \eqref{eq:combine} follows from the triangle inequality, the third inequality in \eqref{eq:combine} follows from the definition of quantile function $Q$ where we define the $1 - \delta$ quantile of $\sup_{f \in \mathcal F}|X_f(\mu) - X_{f^{\ast}}(\mu)|$ by $Q(\mu, \delta)$, i.e., the minimum possible $q$ satisfying that
\[\mathbb P(\sup_{f \in \mathcal F} |X_f(\mu) - X_{f^{\ast}}(\mu)| \leq q ) \geq 1 - \delta.\]
The fourth inequality in \eqref{eq:combine} is according to Lemma~\ref{lem:r} that
$\hat r_{\hat f}(\mu_0) \leq 0$, $\bar r_{f^{\ast}}(\mu_0) \leq 0$ and
$- \hat r_{f^{\ast}}(\mu_0) \leq 2 L_p A_{\alpha}(\delta)$.
This completes the proof of lemma.
\end{proof}

\begin{proof}[Proof of Lemma~\ref{lem:B}]
	We write $Y = \alpha(f(X) - \mu)$ and use the fact that
	$\phi(x) \leq \log(1 + x + C_p |x|^p)$. Then
	\begin{eqnarray}
	\exp\{\alpha \bar r_f(\mu)\} &\leq&
	\exp\{ \mathbb E[\log(1 + Y + C_p |Y|^p)]\} \nonumber \\
	&\leq& \mathbb E[1 + Y + C_p |Y|^p], \nonumber \\
	&=& 1 + \alpha (m_f - \mu) + C_p \mathbb E[|\alpha(f(X) - m_f +m_f - \mu)|^p] \nonumber \\
	&\leq& 1 + \alpha (m_f - \mu) + h^{-\epsilon} C_p \alpha^{1 + \epsilon} v
	+ C_p \alpha^{1+\epsilon} (1 - h)^{-\epsilon} |m_f - \mu|^p \nonumber \\
	&\leq& \exp\{\alpha B_f^{+}(\mu, 0)\},
	\end{eqnarray}
	where we use the convexity upper bound as follows,
	\[(a + b)^{1 + \epsilon} \leq \frac{a^{1 + \epsilon}}{h^{\epsilon}} + \frac{b^{1 + \epsilon}}{(1-h)^{\epsilon}}.\]
	Therefore, we have $\bar r_f(\mu) \leq B_{f}^{+}(\mu, 0)$ held for any $f \in \mathcal F$.
	Recall that $\bar \mu_f$ satisfies $\bar r_f(\mu) = 0$, therefore
	$\bar \mu_f \leq \mu_f^{+}(0) \leq m_f + 2 h^{-\epsilon} C_p \alpha^{\epsilon} v$. The other side of inequality is similar.

	If $\bar r_{\hat f}(\mu) \leq \eta$, then $B_{\hat f}^{-}(\mu,0) \leq \eta$ which is equivalent to
	$B_{\hat f}^{-}(\mu, \eta) \leq 0$. Note that
	$\bar r_{\hat f}(\mu)$ is a non-increasing function, $\mu$ is then above the largest solution to $B_{\hat f}^{-}(\mu, \eta) = 0$.
	Thus, $\mu_{\hat f}^{-}(\eta) \leq \mu$ which implies
	$m_{\hat f} \leq \mu + 2 h^{-\eta} C_p \alpha^{\epsilon} v + 2 \eta$. This concludes the proof.
\end{proof}

\begin{proof}[Proof of Lemma~\ref{lem:lip}]
	We can easily compute the derivative of $\phi(x)$ for $x \neq 0$.
	That is,
	\begin{eqnarray}
	\phi'(x) = \frac{1 + p C_p |x|^{p-1}}{1 + |x| + C_p |x|^p},
	\end{eqnarray}
 so $|\phi^\prime(x)|\leq p$
 if $|x|\geq 1$ and $|\phi'(x)| \leq 1+pC_p$ if $0<|x|<1$, showing the claimed Lipschitz property.
	Additionally, $|\phi'(x)| \leq p/|x|$ for all $x$, so for a given $0<\gamma<1$ we have
 $|\phi'(x)| \leq C_{\alpha,p}/|x^{1-\gamma}$. This gives the required global H\"older continuity.
 \end{proof}

\begin{proof}[Proof of Lemma~\ref{lem:r}]
	For (i.), on $\Omega_{f^{\ast}}(\delta)$ and by the definition of $\hat f$, we have
	\[\hat \mu_{\hat f} \leq \hat \mu_{f^{\ast}} \leq m_{f^{\ast}} + A_{\alpha}(\delta)  = \mu_0.\]
	Since $\hat r_{\hat f}$ is a non-increasing function of $\mu$,
	$\hat r_{\hat f}(\mu_0) \leq \hat r_{\hat f}(\mu_{\hat f}) = 0$.

	For (ii.), by Lemma~\ref{lem:B},
	$\bar \mu_{f^{\ast}} \leq m_{f^{\ast}} + 2h^{-\epsilon} C_p \alpha^{\epsilon} v \leq m_{f^{\ast}} + A_{\alpha}(\delta) = \mu_0$. Again by the fact that
	$\bar r_{f^{\ast}}$ is a non-increasing function, we have
	$\bar r_{f^{\ast}}(\mu_0) \leq \bar r_{f^{\ast}}(\bar \mu_{f^{\ast}}) = 0$.
	
	For (iii.), by Lemma~\ref{lem:lip}, we can get
	\begin{eqnarray}
	|\hat r_{f^{\ast}}(\mu_0)| &=& |\hat r_{f^{\ast}}(\hat \mu_{f^{\ast}}) - \hat r_{f^{\ast}}(\mu_0)| \leq L_p |\hat \mu_{f^{\ast}} - \mu_0| \nonumber \\
	&\leq& L_p(|\hat \mu_{f^{\ast}} - m_{f^{\ast}}| + |m_{f^{\ast}} - \mu_0|) \nonumber \\
	&\leq& 2 L_p A_{\alpha}(\delta).
	\end{eqnarray}
	This implies $- \hat r_{f^{\ast}}(\mu_0) \leq 2 L_p A_{\alpha}(\delta)$.
\end{proof}

\section{Proof of Results in Section \ref{sec:Q}}

\begin{proof}[Proof of Lemma~\ref{lem:chain}]
	We let $e_n(T) := \inf \{\epsilon: N(T, d, \epsilon/2) \leq N_{n}\}$
	with $N_n = 2^{2^n}$.
	We can construct a partition $\mathcal A_n^{\ast}$ such that
	$|\mathcal A_n^{\ast}| \leq 2^{2^n}$ and $\Delta(A) \leq e_n(T)$ for any $A \in \mathcal A_n^{\ast}$.

	By the definition of $e_n(T)$, we know that $e_{n+1}(T) \leq e_n(T)$ and for any $\epsilon < e_n(T)$, it holds
	$N(T,d, \epsilon/2) > N_n$, i.e., $N(T,d,\epsilon/2) \geq 1 + N_{n}$.
	So we  have
	\begin{eqnarray}
	& & (\log(1 + N_{n}))^{1/\beta}((e_n(T))^{p/2} - (e_{n+1}(T))^{p/2}) \nonumber \\
	&=& (\log(1 + N_{n-1}))^{1/\beta} \int_{e_{n+1}(T)}^{e_n(T)} \frac{p}{2} \epsilon^{p/2 - 1} d\epsilon \nonumber \\
	&\leq& \frac{p}{2} \int_{e_{n+1}(T)}^{e_n(T)} \epsilon^{p/2 - 1} (\log(N(T,d,\epsilon/2)))^{1/\beta} d\epsilon.
	\end{eqnarray}

	Note that $\log(1 + N_n) \geq 2^n \log 2$ for any $n \geq 0$, we sum over $n$ and get
	\[(\log 2)^{1/\beta} \sum_{n} 2^{(n)/\beta} ((e_{n}(T))^{p/2} - (e_{n+1}(T))^{p/2}) \leq \frac{p}{2} \int_0^{e_0(T)} \epsilon^{p/2 - 1} (\log(N(T,d,\epsilon/2)))^{1/\beta} d\epsilon. \]
	
	Furthermore,
	\begin{eqnarray}
	& & \sum_{n} 2^{(n)/\beta} ((e_{n}(T))^{p/2} - (e_{n+1}(T))^{p/2}) \nonumber \\
	& = & \sum_{n \geq 0} 2^{(n)/\beta} (e_{n}(T))^{p/2}  - \sum_{n \geq 1} 2^{(n-1)/\beta} (e_{n}(T))^{p/2} \nonumber \\
	&\geq& (1 - 2^{- 1/\beta}) \sum_{n \geq 0} 2^{n/\beta} (e_n(T))^{p/2}.
	\end{eqnarray}
	
	Therefore, we have that
	\begin{eqnarray}
	\sum_{n \geq 0} 2^{n/\beta} (e_n(T))^{p/2} &\leq& \frac{1}{(\log 2)^{1/\beta} (1 - 2^{-1/\beta}) }\frac{p}{2} \int_0^{e_0(T)} \epsilon^{p/2 - 1} (\log(N(T,d,\epsilon/2)))^{1/\beta} d\epsilon \nonumber \\
	&\leq& C_{\beta, p} \int_0^{\infty} \epsilon^{p/2 - 1} (\log(N(T,d,\epsilon/2)))^{1/\beta} d\epsilon. \nonumber
	\end{eqnarray}
	
	Finally, by the definition of $\gamma_{\beta,p}(T,d)$, we have
	\begin{eqnarray}
	& & \gamma_{\beta,p}(T,d) \nonumber \\
	&\leq& \sup_{t \in T} \sum_{n \geq 0} 2^{n/\beta} (\Delta(A_n^{\ast}(t)))^{p/2} \nonumber \\
	&\leq&  \sum_{n \geq 0} 2^{n/\beta} \sup_{t \in T} (\Delta(A_n^{\ast}(t)))^{p/2} \nonumber \\
	&\leq&  \sum_{n \geq 0} 2^{n/\beta} (e_n(T))^{p/2} \nonumber \\
	&\leq& C_{\beta, p} \int_0^{\infty} \epsilon^{p/2 - 1} (\log(N(T,d,\epsilon/2)))^{1/\beta} d\epsilon.
	\end{eqnarray}
	This completes the proof.
\end{proof}

\begin{proof}[Proof of Theorem \ref{thm:bound1}]

By recalling
\begin{eqnarray}\label{def:X}
X_f(\mu) = \frac{1}{n} \sum_{i=1}^n [\frac{1}{\alpha}\phi(\alpha(f(X_i) - \mu)) - \frac{1}{\alpha} \mathbb E[\phi(\alpha(f(X_i) - \mu))]],
\end{eqnarray}
we then know
\begin{eqnarray}
& & n(X_f(\mu) - X_{f'}(\mu)) \nonumber \\
& = & \sum_{i=1}^n \big[ \frac{1}{\alpha} \phi(\alpha(f(X_i) - \mu)) - \frac{1}{\alpha}\mathbb E[\phi(\alpha(f(X) - \mu))] - (\frac{1}{\alpha}\phi(\alpha(f'(X_i) - \mu)) - \frac{1}{\alpha}\mathbb E[\phi(\alpha(f'(X) - \mu))]) \big] \nonumber.
\end{eqnarray}

    Using Holder property of $\phi$, we have
\begin{eqnarray}
& & \mathrm{Var}[\frac{1}{\alpha} \phi(\alpha(f(X_i) - \mu)) - \frac{1}{\alpha}\mathbb E[\phi(\alpha(f(X) - \mu))] - (\frac{1}{\alpha}\phi(\alpha(f'(X_i) - \mu)) - \frac{1}{\alpha}\mathbb E[\phi(\alpha(f'(X) - \mu))])] \nonumber \\
&=& \frac{1}{\alpha^2} \mathrm{Var}[ \phi(\alpha(f(X_i) - \mu)) - \mathbb E[\phi(\alpha(f(X) - \mu))] - (\phi(\alpha(f'(X_i) - \mu)) - \mathbb E[\phi(\alpha(f'(X) - \mu))])] \nonumber \\
& \leq & \frac{1}{\alpha^2} \mathbb E[ (\phi(\alpha(f(X_i) - \mu)) - \phi(\alpha(f'(X_i) - \mu)))^2]  \nonumber \\
&\leq& \frac{C_{3p}^2}{\alpha^2} \mathbb E[ |\alpha(f(X_i) - f'(X_i))|^p] \nonumber \\
& = & \frac{C_{3p}^2 \alpha^p}{\alpha^2} (d_p(f,f'))^p,
\end{eqnarray}
where $d_p(f,f') := (\mathbb E[|f(X) - f'(X)|^p])^{1/p}$.
Additionally, we have
\begin{eqnarray}
& & \bigg| \mathbb E[\phi(\alpha(f(X) - \mu))] - \mathbb E[\phi(\alpha(f'(X) - \mu))]\bigg| \nonumber \\
&\leq& \mathbb E[\bigg|\phi(\alpha(f(X) - \mu)) -  \phi(\alpha(f'(X) - \mu))\bigg|]  \nonumber \\
&\leq& C_{3p}\mathbb E[|\alpha(f(X_i) - f'(X_i))|^{p/2}] \nonumber \\
&\leq& C_{3p} \alpha^{p/2} (D(f,f'))^{p/2}. \label{e:var.p}
\end{eqnarray}
Thus we obtain
\begin{eqnarray}
& & \bigg|\frac{1}{\alpha} \phi(\alpha(f(X_i) - \mu)) - \frac{1}{\alpha}\mathbb E[\phi(\alpha(f(X) - \mu))] - (\frac{1}{\alpha}\phi(\alpha(f'(X_i) - \mu)) - \frac{1}{\alpha}\mathbb E[\phi(\alpha(f'(X) - \mu))])\bigg| \nonumber \\
&\leq& \frac{2}{\alpha} C_{3p} \alpha^{p/2} (D(f,f'))^{p/2}.
\end{eqnarray}

Then we can apply Bernstein inequality to get
\begin{eqnarray}
& & \mathbb P( n |X_f(\mu) - X_{f'}(\mu)| > nt ) \nonumber \\
&\leq& 2 \exp\{- \frac{n^2 t^2}{2(n C_{3p}^2 \alpha^p d_p^p(f,f')/\alpha^2 + 2 C_{3p} \alpha^{p/2 - 1} (D(f,f'))^{p/2} n t/3) }\} \nonumber \\
& = & 2 \exp\{- \frac{n t^2}{2( C_{3p}^2 \alpha^{p-2} d_p^p(f,f') + 2 C_{3p} \alpha^{p/2 - 1} (D(f,f'))^{p/2} t/3) }\}.
\end{eqnarray}
We then recall the following lemma, which is Lemma 2.2.10 from~\citet{van1996weak}.
\begin{lemma}\label{lem:van1}
	Let $a,b > $ 0, assume that the random variables satisfy,
	\[\mathbb P(|X_i| > x) \leq 2 \exp\{- \frac{1}{2} \frac{x^2}{b + ax} \} \]
	for any $x > 0$. Then
	\[ \|\max_{1 \leq i \leq m} X_i \|_{\psi_1} \leq
	48(a \log(1 + m) + \sqrt{b} \sqrt{\log(1+m)}).\]
\end{lemma}

We write $\tilde X_i = |X_{f_i}(\mu) - X_{f_i^{'}}(\mu)|$. Then
Lemma~\ref{lem:van1} gives us that
\begin{equation} \label{e:max.n}
\|\max_{1 \leq i \leq m} \tilde X_i\|_{\psi_1} \leq 48 C_{3p}\bigg(\frac{2 \alpha^{p/2 - 1}}{3 n} (D_m)^{p/2} \log(1 + m) + \sqrt{\frac{\alpha^{p-2}}{n}} d_{p,m}^{p/2} \sqrt{\log(1 + m)} \bigg),
\end{equation}
where $D_m := \max_i D(f_i, f_i')$ and $d_{p,m} := \max_i d_{p}(f_i, f_i')$.


  We now derive a bound on $Q(\mu, \delta)$.
Consider an admissible sequence ($\mathcal B_n$) such that for all
$f \in \mathcal F$,
\[\sum_{n \geq 0} 2^n (\Delta_D(B_n(f)))^{p/2} \leq 2 \gamma_{1,p}(\mathcal F, D) \]
and an admissible sequence $(\mathcal C_n)$ such that for all $f \in \mathcal F$,
\[\sum_{n \geq 0} 2^{n/2} (\Delta_{d_p}(C_n(f)))^{p/2} \leq 2 \gamma_{2,p}(\mathcal F, d_p). \]
Now we define an admissible sequence by intersecting the elements of ($\mathcal B_{n-1}$) and ($\mathcal C_{n-1}$): set $\mathcal A_0 = \{\mathcal F\}$ and set
\[\mathcal A_n = \{B \cap C: B \in \mathcal B_{n-1} ~\textrm{and}~ C \in \mathcal C_{n-1}\}.\]

Define a sequence of finite sets $\mathcal F_0 = \{f\} \subset \mathcal F_1 \subset \cdots \subset \mathcal F$ such that
$\mathcal F_n$ contains a single point in each set of $\mathcal A_n$.
For any $f' \in \mathcal F$, denote by $\pi_n(f')$ the unique elements of $\mathcal F_n$ in $A_n(f')$. Then  by continuity of $\phi$,
\begin{eqnarray}
  X_{f'}(\mu)-X_f(\mu) = \sum_{k=0}^{\infty} (X_{\pi_{k+1}(f')}(\mu) - X_{\pi_k(f')}(\mu))
\end{eqnarray}
a.s. Using the fact that $\|\cdot\|_{\psi_1}$ is a norm and \eqref{e:max.n}
we have
\begin{eqnarray}
& & \|\sup_{f' \in \mathcal F} |X_{f}(\mu) - X_{f'}(\mu)| \|_{\psi_1} \nonumber \\
&\leq& \sum_{k=0}^{\infty} \|\max_{f' \in \mathcal F_{k+1}} |X_{\pi_{k+1}(f')(\mu)} - X_{\pi_k(f')}(\mu)| \|_{\psi_1} \nonumber \\
&\leq& 48 C_{3p} \sum_k \bigg(\frac{2 \alpha^{p/2 - 1}}{3 n} (\Delta_D(B_k(f')))^{p/2} \log(1 + 2^{2^{k+1}}) + \sqrt{\frac{\alpha^{p-2}}{n}} (\Delta_{d_p}(C_k(f')))^{p/2} \sqrt{\log(1 + 2^{2^{k+1}})} \bigg) \nonumber \\
&\leq& 192 \log(2) C_{3p} \sum_k \bigg(\frac{2 \alpha^{p/2 - 1}}{3 n} (\Delta_D(B_k(f')))^{p/2} 2^k + \sqrt{\frac{\alpha^{p-2}}{n}} (\Delta_{d_p}(C_k(f')))^{p/2} 2^{k/2} \bigg) \nonumber \\
&\leq& 384 \log(2) C_{3p} \left(\frac{2 \alpha^{p/2-1}}{3 n} \gamma_{1,p}(\mathcal F, D) + \sqrt{\frac{\alpha^{p-2}}{n}} \gamma_{2,p}(\mathcal F, d_p)\right),
\end{eqnarray}
where we have use the fact that $\log(1 + 2^{2^{k+1}}) \leq 4 \log(2) 2^k$.

Since
\[ X \leq \|X\|_{\psi_1} \log (2 /\delta)\]
with probability at least $1-\delta$ for any sub-exponential random variable,
we conclude that
\[\mathbb P\bigg (\sup_{f \in \mathcal F} |X_f(\mu) - X_{f^{\ast}}(\mu)| \leq 384 \log(2) C_{3p} \left(\frac{2 \alpha^{p/2-1}}{3 n} \gamma_{1,p}(\mathcal F, D) + \sqrt{\frac{\alpha^{p-2}}{n}} \gamma_{2,p}(\mathcal F, d_p))\log(2/\delta) \right)\bigg) \geq 1 - \delta.\]
In particular,
\begin{eqnarray}
Q(\mu, \delta) \leq 384 \log(2) C_{3p} \log (2 /\delta) \left(\frac{2 \alpha^{p/2-1}}{3 n} \gamma_{1,p}(\mathcal F, D) + \sqrt{\frac{\alpha^{p-2}}{n}} \gamma_{2,p}(\mathcal F, d_p)\right)  \label{Q:bound1}
\end{eqnarray}
for every $\mu$.

We put together
\eqref{e:gen.cat},
\eqref{Q:bound1} and the obvious observation
$$
\mathbb E|W- \mathbb E W|^p \leq 2^p \mathbb E|W|^p
$$
valid for any random variable $W$ with a finite $p$th moment. This
gives us the desired result.
\end{proof}

\newpage

\section{Proof of Results in Section \ref{sec:example}}

\begin{proof}[Proof of Proposition \ref{prop:l2}]

It is straightforward to see that
\begin{eqnarray}
|(g(Z_i) - Y_i)^2 - (g'(Z_i) - Y_i)^2
&\leq&  d_{\infty} (g, g')\bigl(|Y_i-g(Z_i)| + |Y_i-g'(Z_i)|\bigr).
\end{eqnarray}
 Thus
\[d_{X,p}(f_g,f_{g'}) \leq  d_{\infty}(g,g')  \left[\left(\frac1n \sum_{i=1}^n |Y_i-g(Z_i)|^p\right)^{1/p}
+ \left(\frac1n \sum_{i=1}^n |Y_i-g'(Z_i)|^p\right)^{1/p}
\right].
\]
By Chebyshev's inequality,
it holds that
\[\frac{1}{n} \sum_{i=1}^n |Y_i|^p \leq \mathbb E[|Y|^p] + \sqrt{8 v / n \delta}\]
with probability at most $\delta / 8$.
Choosing $\Delta$ to be upper bound of $d_{\infty}(g,g')$ for any $g,g' \in \mathcal G$, we then have
\[d_X(f,f') \leq 2^{1 + (2-p)/p} d_{\infty}(g,g') \Big(\Delta^p + \mathbb E[|Y|^p] + \sqrt{8 v / n \delta}\Big)^{1/p}\]
holds with probability at least $1 - \delta/8$.
By definition, it is easy to see that $\gamma_{2,p}(\mathcal G, d_1) \leq c \gamma_{2,p}(\mathcal G, d_2)$ for any distances $d_1, d_2$ satisfying $d_1 \leq c d_2$.
Then, we know that
$$\Gamma_{\delta} \leq \Gamma_{\delta}(\Delta) := 2^{1 + (2-p)/p} (\Delta^p + \mathbb E[|Y|^p] + \sqrt{8 \sigma^2 / n \delta})^{1/p} \cdot \gamma_{2,p}(\mathcal G, d_{\infty}).$$
By choosing $\Delta$ large enough, it holds $\Gamma_{\delta}(\Delta) \geq \Delta \geq \text{diam}_{d_p}(\mathcal F)^{p/2}$.

\end{proof}

\begin{proof}[Proof of Proposition \ref{prop:kernel}]

By representer theorem, the optimizer of \eqref{obj:kernel} has the form, $\hat h(x) = \sum_{i=1}^n c_i K(x_i,x)$.
Therefore, solving \eqref{obj:kernel} requires handling with an $n$ by $n$ matrix, which is computationally expensive in most cases. In the following, we prove a stronger version. We consider a smaller Hilbert space $\mathcal H_s$ instead of $\mathcal H$.

Note that the kernel $K$ is a Mercer kernel which admits the approximation
$$K(x,y) = \sum_{j=1}^{\infty} \lambda_j \varphi(v_j,x) \varphi(v_j,y).$$
We define a smaller RKHS space
$$\mathcal {H}_s = \{h(x): h(x) = \sum_{i=s}^S c_i \varphi(v_i,x), c_i \in \mathbb R\},$$
where $v_i, i = 1,\ldots, S$ are $S$ features with $S << n$.
Then we practically solve the following estimator,
\begin{eqnarray}
\hat f_{H_s} = \arg\min_{f = L \circ h \in L \circ \mathcal H_s} \{\hat \mu_f + \lambda_n \|h\|_{\mathcal H_s}^2 \}.
\end{eqnarray}

Let $\tilde m^{\ast} := \min_{f \in L \circ \mathcal H_s} m_{f}$.
Given $v_1, \ldots, v_S$ and recalling the fact that
$\gamma_{\beta,p}(L \circ \mathcal H_s, d) \leq \gamma_{\beta,p}(\mathcal F, d)$
for any $\beta$, distance $d$ and sub-space $\mathcal H_s$, we apply Theorem~\ref{thm:bound1} or Theorem~\ref{thm:bound2} (simply modifying the proof by setting $A_{\alpha}(\delta) = h^{-\epsilon} C_p \alpha^{\epsilon} v + \frac{\log(2/\delta)}{\alpha n} + \lambda_n + \text{err}$) and obtain \textbf{stronger result},
\begin{eqnarray}
m_{\hat f_{H_s}} - \tilde m^{\ast} \leq
6 L_p (h^{-\epsilon} C_p \alpha^{\epsilon} v + \frac{\log(2/\delta)}{\alpha n} + \lambda_n + \text{err}) + 2 Q_{1,\mathcal H_s}(\delta) \label{eq:kernel}
\end{eqnarray}
holds with probability $1 - \delta$.
Here $\text{err}$ is an approximation error (which will be explain later in this section) 
and
\begin{eqnarray}
Q_{1,\mathcal H_s}(\delta) &=& L_1 C_{3p} \log (2 /\delta) (\frac{2 \alpha^{p/2-1}}{3 n} \gamma_{1,p}(L \circ \mathcal H_s, D) + \sqrt{\frac{\alpha^{p-2}}{n}} \gamma_{2,p}(L \circ \mathcal H_s, d_p)) \nonumber.
\end{eqnarray}
\begin{remark}\label{rmk:cover}
We can obtain the upper bounds of
$\gamma_{1,p}(L \circ \mathcal H_s, D), \gamma_{2,p}(L \circ \mathcal H_s, d_p))$ by computing the covering number
$N(L \circ \mathcal H_s, D, \epsilon / 2)$
and
$N(L \circ \mathcal H_s, d_p, \epsilon / 2)$ in specific cases.
For example,
suppose loss function $L$ is $c_1$-Lipschitz continuous with respect to argument $h(x)$.
Then
$N(L \circ \mathcal H_s, D, \epsilon / 2) \leq
N(\mathcal H_s, D, \epsilon / 2 c_1)$
and
$N(L \circ \mathcal H_s, d_p, \epsilon / 2) \leq N(\mathcal H_s, d_p, \epsilon / 2 c_1)$.
Write $\mathbb C = \{(c_1, \ldots, c_s): c_i \in [-b,b]\}$ and assume
eigen-functions satisfy $\max_i \sup_{x} \varphi(v_i,x) \leq B$ and $\max_i \mathbb E[|\varphi(v_i, X)|^p]^{1/p} \leq B$.
We know
$N(\mathcal H_s, D, \epsilon / 2 c_1) \leq N(\mathbb C, \ell_1, \epsilon / 2 c_1 B)$
and
$N(\mathcal H_s, d_p, \epsilon / 2 c_1) \leq N(\mathbb C, \ell_p, \epsilon / 2 c_1 B S^{(p-1)/p})$.
Finally, it is know that
$N(\mathbb C, \ell_1, \epsilon / 2 c_1 B) = O\big((\frac{S}{\epsilon})^{S}\big) $
and
$N(\mathbb C, \ell_p, \epsilon / 2 c_1 B S^{(p-1)/p}) = O\big((\frac{\sqrt{S}}{\epsilon})^{S}\big)$.
Upper bounds of
$\gamma_{1,p}(L \circ \mathcal H_s, D), \gamma_{2,p}(L \circ \mathcal H_s, d_p)$ is then obtained from Lemma~\ref{lem:chain}.
\end{remark}

Then problem is reduced to understanding the difference $\tilde m^{\ast} - m^{\ast}$.
By the definition, we know
\begin{eqnarray}
\tilde m^{\ast} - m^{\ast} = \tilde m^{\ast} - m_{f_0} + m_{f_0} - m^{\ast} \leq m_{f_0} - m^{\ast} =: \text{err} \label{eq:err}
\end{eqnarray}
for any $f_0 \in L \circ \mathcal H_s$.
We need to find a suitable $f_0 = L \circ h_0$ such that $m_{f_0} - m^{\ast}$ is as small as possible
(i.e., approximating $L \circ h^{\ast}$ as close as possible).
By definition of $m_f$, we have
\begin{eqnarray}
m_{f_0} - m^{\ast} &=& \mathbb E [L(Y - h_0(X))] - \mathbb E[L(Y - h^{\ast}(X))] \nonumber \\
&\leq& \mathbb E | C(Y) (h_0(X) - h^{\ast}(X))| \nonumber \\
&\leq& \sqrt{E [C^2(Y)]} \sqrt{\mathbb E [(h_0(X) - h^{\ast}(X))^2]} \nonumber \\
&\leq& C \sqrt{\mathbb E [(h_0(X) - h^{\ast}(X))^2]}.   \label{eq:m-diff-middle}
\end{eqnarray}
by adjusting constant $C$. The last inequality uses the assumption that $C(Y)$ is square integrable.

Since $K$ is a Mercer kernel which satisfies
$K(x,y) = \sum_{i=1}^{\infty} \lambda_i \varphi_i(x) \varphi_i(y)$ with $\{\varphi_i(\cdot)\}$ are orthonormal bases in $L_2(X)$, $\lambda_i$ are non-increasing.
Then we know
$\hat K(x,y) = \sum_{j=1}^S \lambda_{i_j} \varphi_{i_j}(x)\varphi_{i_j}(y)$.
In addition, $h^{\ast}$ can be decomposed as $h^{\ast} = \sum_{i=1}^{\infty} a_i^{\ast} \varphi_i(x)$
satisfying that $\sum_{i=1}^{\infty} (a_i^{\ast})^2 / \lambda_i \leq 1$.
To this end, we deliberately choose $h_0(x) = \sum_{j=1}^S a_{i_j}^{\ast} \varphi_{i_j}(x)$.
\begin{eqnarray}
\|h_0(x) - h^{\ast}(x)\|_{L_2} &=& \sqrt{\int |h_0(x) - h^{\ast}(x)|^2 dx}
\nonumber \\
&\leq& \sqrt{\sum_{i=1}^{\infty} (a_{i}^{\ast})^2 - \sum_{j=1}^{S} (a_{i_j}^{\ast})^2 } \nonumber \\
&\leq& \sqrt{\lambda_{i_0}}, \label{eq:eigen}
\end{eqnarray}
where $i_0 = \arg\min\{ i: ~ i ~ \text{is not in}~ i_1, \ldots, i_S\}$.
The last inequality \eqref{eq:eigen} uses the fact that
\[\sum_{i=1}^{\infty} (a_{i}^{\ast})^2 - \sum_{j=1}^{S} (a_{i_j}^{\ast})^2
= \sum_{i = i_0}^{\infty} (a_i^{\ast})^2
\leq \lambda_{i_0} (\sum_{i = i_0}^{\infty} (a_i^{\ast})^2 / \lambda_i ) \leq \lambda_{i_0} \Big(\sum_{j=1}^{\infty} (a_{j}^{\ast})^2 / \lambda_{j} \Big) \leq \lambda_{i_0}. \]
Therefore, with \eqref{eq:m-diff-middle}, we have
\[\text{err} = m_{f_0} - m^{\ast} \leq C \sqrt{\lambda_{i_0}}\]
and plug this back into \eqref{eq:kernel} to conclude the analysis of excess risk, $m_{\tilde f_{H_s}} - m^{\ast}$.
Finally, note that $\text{err} = 0$ when $\mathcal H_s = \mathcal H$.
This completes the proof.
\end{proof}

\newpage

\section{Proof of Results in Section \ref{sec:algorithm}} \label{sec:proofH}

\begin{proof}[Proof of Theorem~\ref{thm:convergence:alg}]
The proof consists of two main steps.

\noindent \textit{Step 1}. The key step is to prove the uniform concentration bounds of differences between gradients $g^{(t)}$ and their expectations.
Let $r_{n,w}(\theta) = \frac{1}{n\alpha} \sum_{i=1}^n \psi(\alpha(\nabla f_w(X_i) - \theta))$ and we let $\hat \theta_w$ be the solution to $r_{n,w}(\theta) = 0$.
Our first goal is to compare the difference between $\hat \theta_{w_2} - \hat \theta_{w_1}$ for $w_1 \neq w_2$.

We consider the following two cases one by one. Case 1 is the special case of Case 2.
\begin{itemize}
	\item[] \textit{Case 1}. It holds $|\nabla f_{w_1}(X) - \nabla f_{w_2}(X)| \leq R |w_1 - w_2|$ for any $X$.  (That is, Assumption \textbf{A1} is replaced by bounded Lipschitz condition.)
	\item[] \textit{Case 2}. It holds $|\nabla f_{w_1}(X) - \nabla f_{w_2}(X)| \leq R_{\eta} |w_1 - w_2|$ with $1 - \eta$ probability.
\end{itemize}

In the first case, we know that
\begin{eqnarray}
\frac{1}{n\alpha} \sum_{i=1}^n \psi(\alpha(\nabla f_{w_1}(X_i) - R|w_1 - w_2| - \theta))  \leq r_{n,w_2}(\theta) \leq \frac{1}{n\alpha} \sum_{i=1}^n \psi(\alpha(\nabla f_{w_1}(X_i) + R|w_1 - w_2| - \theta))
\end{eqnarray}
Since $\psi(\cdot - \theta)$ is non-increasing, then $\hat \theta_{w_2}$ is no greater than the solution to
$\sum_{i=1}^n \psi(\alpha(\nabla f_{w_1}(X_i) + R|w_1 - w_2| - \theta)) = 0$ and
is no smaller than $\sum_{i=1}^n \psi(\alpha(\nabla f_{w_1}(X_i) - R|w_1 - w_2| - \theta)) = 0$.
It is also easy to see that $\hat \theta_{w_1} \pm R|w_1 - w_2|$s are the solutions to $\sum_{i=1}^n \psi(\alpha(\nabla f_{w_1}(X_i) \pm R|w_1 - w_2| - \theta)) = 0$, respectively.
Therefore, we have
\begin{eqnarray}\label{w:diff}
|\hat \theta_{w_1} - \hat \theta_{w_2}| \leq R|w_1 - w_2|.
\end{eqnarray}

In the second case, without loss of generality,
we consider the following influence function
\begin{eqnarray}
\psi(x) =
\begin{cases}
\log \big(1 - \frac{p-1}{p}(p C_p)^{-\frac{1}{p-1}} \big)& \text{if}~ x \leq - (p C_p)^{-\frac{1}{p-1}}, \\
\log(1 + x + C_p|x|^p) &  \text{if} ~ - (p C_p)^{-\frac{1}{p-1}} \leq x \leq 0, \\
- \log(1 - x + C_p|x|^p) &  \text{if} ~ 0 < x \leq (p C_p)^{-\frac{1}{p-1}}, \\
- \log \big(1 - \frac{p-1}{p}(p C_p)^{-\frac{1}{p-1}} \big) & \text{if}~ x \geq (p C_p)^{-\frac{1}{p-1}}.
\end{cases}
\end{eqnarray}

We define $\tilde X_{i,\eta}$ be the truncation version of $X_i$ at level $\eta$, that is,
\begin{eqnarray}\label{eq:truncX}
\tilde X_{i,\eta} = X_i \mathbf 1_{|X_i| \leq B_{\eta}},
\end{eqnarray}
where $B_{\eta}$ is the largest positive constant satisfying that
$|\nabla f_{w_1}(X) - \nabla f_{w_2}(X)| \leq R_{\eta}|w_1 - w_2|$ for any $|X| \leq B_{\eta}$.

Therefore, we are able to define $\tilde r_{n,w,\eta}(\theta) = \frac{1}{n \alpha} \sum_{i=1}^n \psi(\alpha(\nabla f_{w}(\tilde X_{i,\eta}) - \theta))$ and let $\hat \theta_{w,\eta, \zeta}$ be the solution to $\tilde r_{n,w,\eta}(\theta) = \zeta$.
Next, we study the difference between $\hat \theta_{w,\eta, \zeta}$ and $\hat \theta_{w,\eta, 0}$. Without loss of generality, we assume $\zeta > 0$ and it is easy to see that
$\hat \theta_{w,\eta, \zeta} \leq \hat \theta_{w,\eta, 0}$ by the fact that $\psi$ is non-increasing.

By assumption, we can obtain that
\[|\nabla f_{w}(X)| \leq R_{\frac{1}{2}} |w| \leq R_{\frac{1}{2}} D_w\]
with $1/2$ probability, where $D_w$ is the diameter of parameter space.
Let event $E_{\text{small},x} = \{ \#\{i: |\nabla f_{w}(X_i)| \geq R_{\frac{1}{2}} D_w \} \geq 1/4 n \}$. By Hoeffding inequality, we know event $E_{\text{small},x}$ happens with probability at least
$ 1 - \exp\{- n / 4\} \geq 1 - \delta$.

On event $E_{\text{small},x}$,
it is not hard to get that
\begin{eqnarray}
\zeta &=& - \int_{\hat \theta_{w,\eta, \zeta}}^{\hat \theta_{w,\eta, 0}}
\tilde r_{n,w,\eta}'(\theta) d\theta \nonumber \\
&=& - \frac{1}{n \alpha}  \sum_{i=1}^n \int_{\hat \theta_{w,\eta, \zeta}}^{\hat \theta_{w,\eta, 0}} \frac{\partial \psi(\alpha(\nabla f_{w}(\tilde X_{i,\eta}) - \theta))}{\partial \theta} d\theta \nonumber \\
&\geq& \frac{1}{4} \psi'_{0.5} \int_{\hat \theta_{w,\eta, \zeta}}^{\hat \theta_{w,\eta, 0}} 1 d\theta  \label{diff:integral:lower}\\
&=& \frac{\psi'_{0.5}}{4} |\hat \theta_{w,\eta, 0} - \hat \theta_{w,\eta, \zeta}|. \label{diff:integral}
\end{eqnarray}
Here \eqref{diff:integral:lower} uses the fact that
$\frac{\partial \psi(\alpha(\nabla f_{w}(\tilde X_{i,\eta}) - \theta))}{\partial \theta} \geq \alpha \psi'_{0.5}$ provided that
$\alpha (R_{1/2} + 1)D_w \leq (2 p C_p)^{-\frac{1}{p-1}}$.
$\psi'_{0.5} = \inf_{0 \leq x \leq (2 p C_p)^{-\frac{1}{p-1}}} \psi'(x)$.
Remark: it is not hard to compute that
\[\psi'_{0.5} \geq \inf_{0 \leq x \leq (2 p C_p)^{-\frac{1}{p-1}}} \frac{1 - C_p p |x|^{p-1}}{1 - x + C_p |x|^p}
\geq \inf_{0 \leq x \leq (2 p C_p)^{-\frac{1}{p-1}}} 1 - C_p p |x|^{p-1} = 1/2.\]

Therefore, we have
\begin{eqnarray}\label{ineq:theta1}
\theta_{w,\eta, \zeta} \geq  \hat \theta_{w,\eta, 0} - \frac{4 \zeta}{\psi'_{0.5}} ~~ \text{and} ~~ \theta_{w,\eta, - \zeta} \geq  \hat \theta_{w,\eta, 0} + \frac{4 \zeta}{\psi'_{0.5}}
\end{eqnarray}
for any $w$ and $\zeta$ on event $E_{\text{small},x}$.

By the choice of $\psi$, it is easy to derive that
\begin{eqnarray}
\tilde r_{n, w, \eta}(\theta) - \frac{2A \eta}{\alpha} \leq  r_{n, w}(\theta) \leq \tilde r_{n, w, \eta}(\theta) + \frac{2A \eta}{\alpha}, \nonumber
\end{eqnarray}
where $A := - \log \big(1 - \frac{p-1}{p}(p C_p)^{-\frac{1}{p-1}} \big)$.
This implies
\begin{eqnarray}\label{ineq:t1}
\hat \theta_{w, \eta, \frac{2A \eta}{\alpha}} \leq \hat \theta_{w} \leq
\hat \theta_{w, \eta, -\frac{2A \eta}{\alpha}}
\end{eqnarray}
and thus
\begin{eqnarray}\label{ineq:t2}
\hat \theta_{w, \eta, 0} - \frac{8 A \eta}{\alpha \psi'_{0.5}}\leq \hat \theta_{w} \leq
\hat \theta_{w, \eta, 0} + \frac{8 A \eta}{\alpha \psi'_{0.5}}
\end{eqnarray}
by taking $\zeta = 2 A \eta / \alpha$.

Following the proof in Case 1, we can get
\[|\hat \theta_{w_1, \eta, 0} - \hat \theta_{w_2, \eta, 0}| \leq R_{\eta} |w_1 - w_2|.\]
Together with \eqref{ineq:t2}, we arrive at
\begin{eqnarray}
|\hat \theta_{w_1} - \hat \theta_{w_2}| \leq R_{\eta} |w_1 - w_2| + \frac{16 A \eta}{\alpha \psi'_{0.5}}
\end{eqnarray}

Additionally, by Lipschitz assumption, we know that $|\mathbb E[\nabla f_{w_1}(X)] - \mathbb E[\nabla f_{w_2}(X)]| \leq L_f|w_1 - w_2|$.
This implies that
\[|\hat \theta_{w} - \mathbb E[\nabla f_{w}(X)]| - |\hat \theta_{w'} - \mathbb E[\nabla f_{w'}(X)]| \leq (R_{\eta} + L_f)|w - w'| + \frac{32 A \eta}{\alpha \psi'_{0.5}}\]
for any $w, w'$.

\noindent \textit{Case 1}. ~ Taking $\delta' = \frac{\delta}{(2 D_w (R + L_f) /(A_{\alpha}(\delta)))^d}$ and using union bound over grids,
\[w \in \{-D_w, -D_w + \frac{A_{\alpha}(\delta)}{(R +L_f)}, \ldots,  D_w\}^d,\]
we have the uniform concentration inequality,
\begin{eqnarray}
\mathbb P(\sup_w |\hat \theta_{w} - \mathbb E[\nabla f_{w}(X)]| \geq 2 A_{\alpha}(\delta')) \leq \delta \label{ineq:uniform}
\end{eqnarray}
on event $E_{\text{small},x}$.

\noindent \textit{Case 2}. ~ Taking $\delta' = \frac{\delta}{(2 D_w (R_{\eta} +L_f) \alpha \psi'_{0.5}/(16 A \eta))^d}$ and using union bound over grids,
\[w \in \{-D_w, -D_w + \frac{16 A \eta}{(R_{\eta} +L_f) \alpha \psi'_{0.5}}, \ldots,  D_w\}^d,\]
we have the uniform concentration inequality,
\begin{eqnarray}
\mathbb P(\sup_w |\hat \theta_{w} - \mathbb E[\nabla f_{w}(X)]| \geq A_{\alpha}(\delta') + \frac{32 A \eta}{\alpha \psi'_{0.5}}) \leq \delta
\end{eqnarray}
on event $E_{\text{small},x}$.
In particular, we take $64 \eta/\alpha < A_{\alpha}(\delta)$, then term $\frac{32 A \eta}{\alpha \psi'_{0.5}}$ can be absorbed into $A_{\alpha}(\delta')$ by multiplying a constant 2.


\medskip

\noindent \textit{Step 2}.
We prove the convergence of Algorithm~\ref{alg:GD} under two scenarios,
(1) $F$ is differentiable and $L$-Lipschitz continuous, (2) $F$ is $\kappa$ strongly-convex and $L$-Lipschitz continuous.

For the first scenario, by the Lipschitz continuity, we have
\begin{eqnarray}
& & F(w^{(t+1)}) - F(w^{(t)}) \leq \nabla F(w^{(t)})(w^{(t+1)} - w^{(t)}) + \frac{L}{2} \|w^{(t+1)} - w^{(t)}\|^2 \nonumber \\
&=& - \gamma_t \nabla F(w^{(t)}) g^{(t)} + \frac{L \gamma_t^2}{2} \|g^{(t)}\|^2 \nonumber \\
& = & - \gamma_t \nabla F(w^{(t)}) (\nabla F(w^{(t)}) + \zeta^{(t)})
 + \frac{L \gamma_t^2}{2} \|\nabla F(w^{(t)}) + \zeta^{(t)}\|^2 \nonumber \\
 &\leq& - \gamma_t \|\nabla F(w^{(t)})\|^2 + \gamma_t \sqrt{d} \bar A_{\alpha}(\delta) \|\nabla F(w^{(t)})\| + L \gamma_t^2 (\|\nabla F(w^{(t)})\|^2 + d (\bar A_{\alpha}(\delta))^2), \label{eq:gd1}
\end{eqnarray}
where $\zeta^{(t)} :=  g^{(t)} - \nabla F(w^{(t)})$ and it is easy to see that $\|\zeta^{(t)}\| \leq \sqrt{d} \bar A_{\alpha}(\delta)$,
where
\[\bar A_{\alpha}(\delta) := A_{\alpha}(\frac{\delta}{2 D_w (R_{\eta} +L_f) \alpha \psi'_{0.5}/(16 A \eta)}).\]

By direct calculation, we can find that
\[\gamma_t \sqrt{d} \bar A_{\alpha}(\delta) \|\nabla F(w^{(t)})\| +
L \gamma_t^2 d (\bar A_{\alpha}(\delta))^2
\leq \frac{\gamma_t - \gamma_t^2 L}{2} \|\nabla F(w^{(t)})\|^2 \]
when $\gamma_t \leq \frac{4}{9 L}$ and $\|\nabla F(w^{(t)})\| \geq \sqrt{d} \bar A_{\alpha}(\delta)$.
Together with \eqref{eq:gd1}, we know
\begin{eqnarray}
F(w^{(t+1)}) - F(w^{(t)}) \leq - \frac{\gamma_t - \gamma_t^2 L}{2} \|\nabla F(w^{(t)})\|^2 \leq - \frac{5}{18} \gamma_t \|\nabla F(w^{(t)})\|^2.
\end{eqnarray}
Define the $T_{stop}$ be the smallest $t$ such that $\|\nabla F(w^{(t)})\| \leq \sqrt{d} \bar A_{\alpha}(\delta)$. Then we know
\begin{eqnarray}
\sum_{t=1}^{T_{stop}} \gamma_t \leq \frac{18(m_{f^{(0)}} - m_{f^{\ast}})}{5 d  (\bar A_{\alpha}(\delta))^2}
\end{eqnarray}
holds with probability at least $1 - \delta$. This complete the proof by using the assumption that $\gamma_t \equiv \gamma$.

In the second scenario, we can compute that
\begin{eqnarray}
& & \|w^{(t+1)} - w^{\ast}\| = \|w^{(t)} - \gamma_t g^{(t)} - w^{\ast}\| \nonumber \\
&\leq&  \|\|w^{(t)} - \gamma_t \nabla F(w^{(t)}) - w^{\ast}\|\| + \gamma_t \|\nabla F(w^{(t)}) - g^{(t)}\|. \label{eq:convex1}
\end{eqnarray}
The first term of \eqref{eq:convex1} can be handled via standard method in~\citet{nesterov2003introductory}. It then follows that
\[ \|w^{(t)} - \gamma_t \nabla F(w^{(t)}) - w^{\ast} \|^2 \leq (1 - \frac{2 \gamma_t \kappa L}{\kappa + L}) \|w^{(t)} - w^{\ast}\|^2. \]
For the second term, it is bounded via $\gamma_t \zeta$ with statistical error $\zeta := \sqrt{d} \bar A_{\alpha}(\delta)$.

Therefore, we have
\begin{eqnarray}
\|w^{(t+1)} - w^{\ast}\| \leq (\prod_{s=0}^{t} a_t) \|w^{0} - w^{\ast}\| +
\zeta(\sum_{s = 0}^t \gamma_s \prod_{s' = s+1}^t a_s),
\end{eqnarray}
where $a_t = \sqrt{1 - \frac{2 \gamma_t \kappa L}{\kappa + L}}$ and $\prod_{s' = t+1}^t \equiv 1$.
Especially, if $\gamma_t \equiv \gamma$ and we let $a \equiv \sqrt{1 - \frac{2 \gamma \kappa L}{\kappa + L}}$, we have
 \begin{eqnarray}
 \|w^{(t+1)} - w^{\ast}\| \leq a^{t+1} \|w^{0} - w^{\ast}\| +
 \zeta \gamma \frac{1 - a^{t+1}}{1 - a}.
 \end{eqnarray}
Finally, it is bounded that $\frac{1 - a^{t+1}}{1 - a} \leq \gamma / (1 - \sqrt{1 - \frac{2 \gamma \kappa L}{\kappa + L}})$.
\end{proof}

\begin{proof}[Proof of Theorem \ref{thm:alg3}]

We do Taylor expansion for
\begin{eqnarray}
\sum_{i=1}^n \phi(\alpha(f_{w^{(t+1)}}(X_i) - \mu)) \label{eq:solve}
\end{eqnarray}
and get
\begin{eqnarray}
& & \sum_{i=1}^n \phi(\alpha(f_{w^{(t)}}(X_i) - \hat \mu^{(t)}) + \alpha(f_{w^{(t+1)}}(X_i) - f_{w^{(t)}}(X_i) + (\hat \mu^{(t)} - \mu))  ) \nonumber \\
& = & \sum_{i=1}^n \phi(\alpha(f_{w^{(t)}}(X_i) - \hat \mu^{(t)}))
+ \sum_{i=1}^n \phi'(\alpha(f_{w^{(t)}}(X_i) - \hat \mu^{(t)}))
\alpha(f_{w^{(t+1)}}(X_i) - f_{w^{(t)}}(X_i) + (\hat \mu^{(t)} - \mu))
+ \nonumber \\
& & + O(\sum_{i=1}^n \alpha^2(f_{w^{(t+1)}}(X_i) - f_{w^{(t)}}(X_i) + (\hat \mu^{(t)} - \mu))^2).
\end{eqnarray}
Therefore, let $\mu^{(t+1)}$ be the solution to \eqref{eq:solve} = 0, it satisfies
\begin{eqnarray}
\mu^{(t+1)} &=& \frac{\sum_i \phi(\alpha(f_{w^{(t)}}(X_i) - \hat \mu^{(t)}))}{\alpha \sum_i \phi'(\alpha(f_{w^{(t)}}(X_i) - \hat \mu^{(t)})) } + \hat \mu^{(t)} + \sum_i \nu_i^{(t)} (f_{w^{(t+1)}}(X_i) - f_{w^{(t)}}(X_i)) \nonumber \\
& &
+ O(\frac{\sum_i \alpha^2 (f_{w^{(t+1)}}(X_i) - f_{w^{(t)}}(X_i) + (\hat \mu^{(t)} - \mu))^2}{\alpha \sum_i \phi'(\alpha(f_{w^{(t)}}(X_i) - \hat \mu^{(t)}))}).
\end{eqnarray}
By recalling that
\[\hat \mu^{(t+1)} = \hat \mu^{(t)} + \sum_i \nu_i^{(t)} (f_{w^{(t+1)}}(X_i) - f_{w^{(t)}}(X_i)),\]
we could compute the difference, $\text{err}_{\mu}^{(t + 1)} := |\mu^{(t+1)} - \hat \mu^{(t+1)}|$,
\begin{eqnarray}
\mu^{(t+1)} - \hat \mu^{(t+1)} =
\underbrace{\frac{\sum_i \phi(\alpha(f_{w^{(t)}}(X_i) - \hat \mu^{(t)}))}{\alpha \sum_i \phi'(\alpha(f_{w^{(t)}}(X_i) - \hat \mu^{(t)})) }}_{\text{term}_1}
+
\underbrace{O(\frac{\sum_i \alpha^2 (f_{w^{(t+1)}}(X_i) - f_{w^{(t)}}(X_i) + (\hat \mu^{(t)} - \mu))^2}{\alpha \sum_i \phi'(\alpha(f_{w^{(t)}}(X_i) - \hat \mu^{(t)}))})}_{\text{term}_2}.
\end{eqnarray}
We next prove that
(i) $\text{term}_1$ is $O_p(\text{err}_{\mu}^{(t)} + \text{err}_{\mu}^{(t) 2})$;
(ii) $\text{term}_2$ is $O_p(\alpha + \alpha^{p-1})$.

We first control term
$\sum_i \phi'(\alpha(f_{w^{(t)}}(X_i) - \hat \mu^{(t)})) $.
Notice that
$\phi'(x) \geq 1/2$ if $|x| \leq c$. We define the index set
$\mathcal I_{large}^{(t)} := \{i: |\alpha(f_{w^{(t)}}(X_i) - \hat \mu^{(t)})| \geq c\}$.
It can be computed that
\begin{eqnarray}
\mathbb P(\{\alpha(f_{w^{(t)}}(X_i) - \hat \mu^{(t)})| \geq c\}) \leq C \alpha^p
\end{eqnarray}
for a universal constant $C$.
Then by Hoeffding inequalities for binary variables, we have
\begin{eqnarray}
|\mathcal I_{large}^{(t)}| = O_p(\alpha^p n \log n).
\end{eqnarray}
Therefore, we have
\begin{eqnarray}
n \geq \sum_i \phi'(\alpha(f_{w^{(t)}}(X_i) - \hat \mu^{(t)})) \geq \frac{1}{2}(n - n \alpha^p \log n). \label{eq:lower:denominator}
\end{eqnarray}
For $\sum_i \phi(\alpha(f_{w^{(t)}}(X_i) - \hat \mu^{(t)}))$, we know
\begin{eqnarray}
\sum_i \phi(\alpha(f_{w^{(t)}}(X_i) - \hat \mu^{(t)})) &=&
\sum_i \phi(\alpha(f_{w^{(t)}}(X_i) - \hat \mu^{(t)})) - 0 \nonumber \\
&=&
\sum_i \phi(\alpha(f_{w^{(t)}}(X_i) - \hat \mu^{(t)}))
- \sum_i \phi(\alpha(f_{w^{(t)}}(X_i) - \mu^{(t)})) \nonumber \\
&\leq& \sum_i \alpha \phi'(\alpha(f_{w^{(t)}}(X_i) - \hat \mu^{(t)})) |\hat \mu^{(t)} - \mu^{(t)}| + O(n \alpha^2 |\hat \mu^{(t)} - \mu^{(t)}|^2) \nonumber \\
&=& \alpha \sum_i \phi'(\alpha(f_{w^{(t)}}(X_i) - \hat \mu^{(t)})) \text{err}_{\mu}^{(t)} + O(n \alpha^2 \text{err}_{\mu}^{(t) 2}).
\end{eqnarray}
Therefore, the first term can be bounded by,
\begin{eqnarray}
\text{term}_1 &\leq&
\frac{\alpha \sum_i \phi'(\alpha(f_{w^{(t)}}(X_i) - \hat \mu^{(t)})) \text{err}_{\mu}^{(t)} + O(n \alpha^2 \text{err}_{\mu}^{(t) 2})}{\alpha \sum_i \phi'(\alpha(f_{w^{(t)}}(X_i) - \hat \mu^{(t)})) } \nonumber \\
&\leq& \text{err}_{\mu}^{(t)} + O(\alpha \text{err}_{\mu}^{(t)2}).
\end{eqnarray}

For $\text{term}_2$,
we define index set
$\mathcal I_{x}^{(t)} := \{i: |f_{w^{(t+1)}}(X_i) - f_{w^{(t)}}(X_i) + (\hat \mu^{(t)} - \mu^{(t)})| \geq x\}$ given level $x$.
Following the idea of deriving \eqref{ineq:uniform}, we know
$\sup_w |\hat \mu_f - m_f| = o_p(1)$.
Thus both $\hat \mu^{(t)}$ and $\mu^{(t)}$ are bounded by some constant $R$. Furthermore, $\mathbb E[|f_w^{(t+1)}(X_i) - f_w^{(t)}(X_i) + (\hat \mu^{(t)} - \mu^{(t)})|^p] \leq p (\mathbb E[|f_w^{(t+1)}(X_i)|^p] + \mathbb E[|f_w^{(t)}(X_i)|^p] + R) \leq \tilde C$ for some universal constant $\tilde C$.

Again, by calculations, we can get that
\begin{eqnarray}
\mathbb P(|f_w(X_i) - f_w(X_i) + (\hat \mu^{(t)} - \mu^{(t)})| \geq x) \leq \tilde C /x^p
\end{eqnarray}
and
$|\mathcal I_{x}^{(t)}| = O_p(n \log n / x^p) $ for all $x$ and $t$.

Then the numerator in $\text{term}_2$ can be bounded by
\begin{eqnarray}
& & \sum_{i \notin \mathcal I_{1}^{(t)}} \alpha^2 (f_{w^{(t+1)}}(X_i) - f_{w^{(t)}}(X_i) + (\hat \mu^{(t)} - \mu))^2
+
\sum_{i \in \mathcal I_{1}^{(t)}} \alpha^2 (f_{w^{(t+1)}}(X_i) - f_{w^{(t)}}(X_i) + (\hat \mu^{(t)} - \mu))^2 \nonumber \\
&\leq& \alpha^2  |\mathcal I_{1}^{(t)c}| +
\sum_{i \in \mathcal I_{1}^{(t)}} \alpha^2 (f_{w^{(t+1)}}(X_i) - f_{w^{(t)}}(X_i) + (\hat \mu^{(t)} - \mu))^2 \\
&\leq& \alpha^2 n + \sum_{k=2}^{x_{cut}/\alpha} \sum_{i \in \mathcal I_{k-1}^{(t)} - \mathcal I_k^{(t)}} \alpha^2 (f_{w^{(t+1)}}(X_i) - f_{w^{(t)}}(X_i) + (\hat \mu^{(t)} - \mu))^2 \nonumber \\
&\leq& \alpha^2 n  + \sum_{k=2}^{x_{cut}/\alpha} \sum_{i \in \mathcal I_{k-1}^{(t)} - \mathcal I_k^{(t)}} \alpha^2 k^2 \nonumber \\
&\leq& \alpha^2 n + 2 \alpha^2 \int_{1}^{x_{cut}/\alpha} x n \log n / x^p dx \nonumber \\
&\leq& \alpha^2 n + 2 \alpha^2 n \log n  (x_{cut}/\alpha)^{2-p},
\end{eqnarray}
where $x_{cut}$ being the threshold that $\phi'(x) \equiv 0$ if $|x| \geq x_{cut}$.

Therefore,
\begin{eqnarray}
\text{term}_2 &\leq& (\alpha^2 n + 2 \alpha^2 n \log n  (x_{cut}/\alpha)^{2-p}) / \big(\frac{\alpha}{2}(n - n\alpha^p\log n)\big)\nonumber \\
&\leq& 5 (\log n) (\alpha + \alpha^{p-1}).
\end{eqnarray}
Putting everything together, we have
\begin{eqnarray}
\text{err}_{\mu}^{(t+1)} \leq \text{err}_{\mu}^{(t)} + O(\alpha \text{err}_{\mu}^{(t)2}) + 5(\log n)(\alpha + \alpha^{p-1}).
\end{eqnarray}
When $n$ is large enough, $O(\alpha \text{err}_{\mu}^{(t)2})$ can be absorbed into $6(\log n)(\alpha + \alpha^{p-1})$.
Thus we have
\begin{eqnarray}
\text{err}_{\mu}^{(t+1)} \leq \text{err}_{\mu}^{(t)} + 6(\log n)(\alpha + \alpha^{p-1})
\end{eqnarray}
and it arrives at
\begin{eqnarray}
\text{err}_{\mu}^{(t)} \leq 6 t (\log n) (\alpha + \alpha^{(p-1)}). \label{eq:129}
\end{eqnarray}
This also implies that $\hat \mu^{(t+1)} = \mu^{(t+1)} + 	O(\text{err}_{\mu}^{(t)}) = m_{f_{w^{(t+1)}}} + O(\text{err}_{\mu}^{(t)}) + o_p(1) \leq R$.

Moreover, we compare the difference between weights
$\tilde \nu_i := \frac{\phi'(\alpha(f_{w^{(t)}}(X_i) - \mu^{(t)}))}{\sum_i \phi'(\alpha(f_{w^{(t)}}(X_i) - \mu^{(t)}))}$ and $\nu_i^{(t)}$.
That is,
\begin{eqnarray}
& & |\tilde \nu_i  - \nu_i^{(t)}| \nonumber \\
& = & |\frac{\phi'(\alpha(f_{w^{(t)}}(X_i) - \mu^{(t)}))}{\sum_i \phi'(\alpha(f_{w^{(t)}}(X_i) - \mu^{(t)}))}
- \frac{\phi'(\alpha(f_{w^{(t)}}(X_i) - \hat \mu^{(t)}))}{\sum_i \phi'(\alpha(f_{w^{(t)}}(X_i) - \hat \mu^{(t)}))}| \nonumber \\
&\leq& |\frac{|(\sum_i \phi'(\alpha(f_{w^{(t)}}(X_i) - \hat \mu^{(t)})) - \sum_i \phi'(\alpha(f_{w^{(t)}}(X_i) - \mu^{(t)}))) \phi'(\alpha(f_{w^{(t)}}(X_i) - \mu^{(t)})) |}{\sum_i \phi'(\alpha(f_{w^{(t)}}(X_i) - \mu^{(t)})) \cdot \sum_i \phi'(\alpha(f_{w^{(t)}}(X_i) - \hat \mu^{(t)}))}| \nonumber \\
& & +
|\frac{|(\phi'(\alpha(f_{w^{(t)}}(X_i) - \hat \mu^{(t)})) -  \phi'(\alpha(f_{w^{(t)}}(X_i) - \mu^{(t)}))) \cdot \sum_i \phi'(\alpha(f_{w^{(t)}}(X_i) - \mu^{(t)})) |}{\sum_i \phi'(\alpha(f_{w^{(t)}}(X_i) - \mu^{(t)})) \cdot \sum_i \phi'(\alpha(f_{w^{(t)}}(X_i) - \hat \mu^{(t)}))}|  \nonumber \\
&\leq& C_{\phi''} n \alpha \text{err}_{\mu}^{(t)} /  (\sum_i \phi'(\alpha(f_{w^{(t)}}(X_i) - \mu^{(t)})) \cdot \sum_i \phi'(\alpha(f_{w^{(t)}}(X_i) - \hat \mu^{(t)}))) \nonumber \\
& & + C_{\phi''} \alpha \text{err}_{\mu}^{(t)} / \sum_i \phi'(\alpha(f_{w^{(t)}}(X_i) - \hat \mu^{(t)})) \nonumber \\
&\leq& 3 C_{\phi''} \alpha \text{err}_{\mu}^{(t)} / n,
\end{eqnarray}
where $C_{\phi''}:= \max_{x: |x| \leq x_{cut}} \phi''(x)$ and the last inequality uses the fact that
$\sum_i \phi'(\alpha(f_{w^{(t)}}(X_i) - \hat \mu^{(t)})) \geq \frac{1}{2}(n - n\alpha^p \log n)$ by \eqref{eq:lower:denominator}.

As a results, we could get that
\begin{eqnarray}
|\nabla_w \hat \mu_{f_{w^{(t)}}} - g^{(t)}| &\leq& \sum_i \frac{3 C_{\phi''} \alpha \text{err}_{\mu}^{(t)}}{n} \nabla_w f_{(w^{(t)})}(X_i) \nonumber \\
&\leq& 3 C_{\phi''} \alpha \text{err}_{\mu}^{(t)}  \cdot \int_1^{x_{cut}/\alpha} x^{-p} dx \\
&\leq& 3 C_{\phi''} \alpha \text{err}_{\mu}^{(t)}  / (p-1) =: \zeta^{(t)}, \label{eq:133}
\end{eqnarray}
elementwisely.

Finally, we analyze the difference $\hat \mu_f$ of at each step,
\begin{eqnarray}
& & \hat \mu_{f_{w^{(t+1)}}} - \hat \mu_{f_{w^{(t)}}} \leq \nabla_w \hat \mu_{f_{w^{(t)}}}(w^{(t+1)} - w^{(t)}) + \frac{L}{2} \|w^{(t+1)} - w^{(t)}\|^2 \nonumber \\
&=& - \gamma_t \nabla_w \hat \mu_{f_{w^{(t)}}} g^{(t)} + \frac{L \gamma_t^2}{2} \|g^{(t)}\|^2 \nonumber \\
& = & - \gamma_t \nabla_w \hat \mu_{f_{w^{(t)}}} (\nabla_w \hat \mu_{f_{w^{(t)}}} + \zeta^{(t)})
+ \frac{L \gamma_t^2}{2} \|\nabla_w \hat \mu_{f_{w^{(t)}}} + \zeta^{(t)}\|^2 \nonumber \\
&\leq& - \gamma_t \|\nabla_w \hat \mu_{f_{w^{(t)}}}\|^2 + \gamma_t \|\nabla_w \hat \mu_{f_{w^{(t)}}}\| \|\zeta^{(t)}\| + L \gamma_t^2 (\|\nabla_w \hat \mu_{f_{w^{(t)}}}\|^2 + \|\zeta^{(t)}\|^2), \label{eq:gd2}
\end{eqnarray}
When $\|\nabla_w \hat \mu_{f_{w^{(t)}}}\| \geq 2 \|\zeta^{(t)}\|$ and $\gamma_t \leq 1 / 5L$, we have
\begin{eqnarray}
\hat \mu_{f_{w^{(t+1)}}} - \hat \mu_{f_{w^{(t)}}}
\leq - \frac{1}{4} \gamma_t \|\nabla_w \hat \mu_{f_{w^{(t)}}}\|^2
\label{eq:descent}
\end{eqnarray}

Moreover, we assume tuning parameter $\alpha \in (0,1)$, then we know
$\|\zeta^{(t)}\| \leq \sqrt{d} C_{\phi''} \alpha^{p} 36 t (\log n)/(p-1)$ from \eqref{eq:129} and \eqref{eq:133}.
Since $T_{end}(\varrho) := \min \{t: \|\nabla_w \hat \mu_{f_{w^{(t)}}}\| \leq \varrho \}$,
therefore we have
\begin{eqnarray}
T_{end}(\varrho) \leq \frac{\hat \mu_{f_{w^{(0)}}}}{\gamma \varrho^2}
\end{eqnarray}
as long as
\begin{eqnarray}
    \varrho \geq \sqrt{d} C_{\phi''} \alpha^{p} 36 t (\log n)/(p-1)
    \label{eq:varrho}
\end{eqnarray}
holds for any $t \leq T_{end}(\varrho)$.
In other words, it suffices to have $\varrho \geq \Big(36 C_{\phi''} \alpha^{p} \sqrt{d}(\log n) \frac{\hat \mu_{f_{w^{(0)}}}}{\gamma (p-1)}\Big)^{1/3}$ to make \eqref{eq:varrho} held. This concludes the proof.
\end{proof}

\end{document}